\documentclass[10pt]{article} 
\usepackage[preprint]{tmlr}

\usepackage{graphicx}
\usepackage{amsmath, amssymb}
\usepackage{amsmath, amssymb, amsthm}
\usepackage{bm, bbm}
\usepackage{hyperref}
\usepackage{url}
\usepackage{xcolor}
\usepackage{subcaption}
\usepackage[group-minimum-digits = 4, group-separator = {,}]{siunitx}
\usepackage{def}
\usepackage{comment}

\newtheorem{theorem}{Theorem}[section]
\newtheorem{proposition}[theorem]{Proposition}
\newtheorem{lemma}[theorem]{Lemma}
\newtheorem{corollary}[theorem]{Corollary}

\title{An $(\epsilon,\delta)$-accurate level set estimation with a stopping criterion}

\author{\name Hideaki Ishibashi \email ishibashi@brain.kyutech.ac.jp \\
      \addr Kyushu Institute of Technology \\
      \AND
      \name Kota Matsui \email matsui.kota.x3@f.mail.nagoya-u.ac.jp \\
      \addr Nagoya University / RIKEN AIP \\
      \AND
      \name Kentaro Kutsukake \email kutsukake.kentaro.c3@f.mail.nagoya-u.ac.jp \\
      \addr Nagoya University \\
      \name Hideitsu Hino \email hino@ism.ac.jp \\
      \addr The Institute of Statistical Mathematics/ RIKEN AIP
      }

\def\month{MM}  
\def\year{YYYY} 
\def\openreview{\url{https://openreview.net/forum?id=XXXX}} 

\let\AND\relax
\usepackage{algorithm}
\usepackage{algorithmic}

\begin{document}

\maketitle

\begin{abstract}
The level set estimation problem seeks to identify regions within a set of candidate points where an unknown and costly to evaluate function’s value exceeds a specified threshold, providing an efficient alternative to exhaustive evaluations of function values. Traditional methods often use sequential optimization strategies to find $\epsilon$-accurate solutions, which permit a margin around the threshold contour but frequently lack effective stopping criteria, leading to excessive exploration and inefficiencies. This paper introduces an acquisition strategy for level set estimation that incorporates a stopping criterion, ensuring the algorithm halts when further exploration is unlikely to yield improvements, thereby reducing unnecessary function evaluations. We theoretically prove that our method satisfies $\epsilon$-accuracy with a confidence level of $1 - \delta$, addressing a key gap in existing approaches. Furthermore, we show that this also leads to guarantees on the lower bounds of performance metrics such as F-score. Numerical experiments demonstrate that the proposed acquisition function achieves comparable precision to existing methods while confirming that the stopping criterion effectively terminates the algorithm once adequate exploration is completed.
\end{abstract}

\section{Introduction}
Adaptive experimental design is a data-driven approach to planning experiments that determines the next experimental conditions based on data obtained so far. It is applied in various fields of experimental sciences such as drug discovery~\citep{griffiths2020constrained} and the development of new materials~\citep{ueno2016combo}.
For example, in manufacturing industries, identifying defective areas where the physical properties of materials do not meet the desired quality is a crucial issue.
Such defective areas are often determined by measuring the physical properties, using techniques like X-ray diffraction, in various parts of refined materials and determining whether they exceed acceptable lower limits.
This problem can be formulated as a {\it Level Set Estimation} (LSE) by considering a black-box function that takes the coordinates of measurement locations as input and outputs the physical properties at each measurement location.
LSE is a problem aimed to identify regions on the input space where the output values of a black-box function are greater (or smaller) than a certain threshold, and active learning~(AL;~\citet{Settles2010}) approach has been proposed specifically to perform LSE with as few experimental iterations as possible~\citep{Bichon2008-dm,gotovos2013active,NIPS2016_ce78d1da,zanette2019robust}.

In practical scenarios of adaptive experimental design, determining ``when to stop the experiment'' is crucially important.
If stopping is not done appropriately, it can lead to wasteful experiments and the squandering of various costs.
In experimental sciences, there are situations where an upper limit on the number of experiments that can be conducted. A na\"ive approach often involves conducting experiments up to such a ``budget limit'' and then stopping.
For LSE, few theoretical guarantees exist on the consistency~\citep{10.3150/18-BEJ1074} or sample complexity~\citep{DBLP:conf/aistats/BachocCG21}, and there also exist some theoretical results on the finite-time guarantee of LSE~\citep{gotovos2013active,DBLP:conf/aistats/MasonJMCJN22}, but to the best of the authors' knowledge, research on the stopping criterion for LSE is limited, with one example being the F-score sampling criterion~\citep{10.1007/s00366-021-01341-7}. This method stops when the 5th percentile of the sampled F-scores exceeds the desired F-score, and it allows for intuitive parameter setting and can be applied to a wide range of acquisition functions. However, in many applications, it is often unclear what is the maximum possible F-score for the problem, and the actual F-score at the stopping point may not exceed the desired F-score, making it difficult to stop the LSE procedure by specifying the F-score.

\paragraph{Contributions} In this paper, we propose an acquisition function for LSE based on the distribution of a random variable that represents the difficulty of classification. The proposed acquisition function entails a natural stopping criterion, probabilistically ensuring that the algorithm can be appropriately stopped when the LSE is accurately performed when used with the proposed acquisition function. Furthermore, our method probabilistically guarantees the lower bounds of performance metrics such as the F-score, accuracy, recall, precision, and specificity. Experiments using test functions and real-world data on the quality of silicon ingots demonstrate that the proposed method performs at least as well as existing methods in terms of the F-score, and can stop the algorithm when sufficient estimation accuracy is achieved.

\paragraph{Related Works} 
For active learning, stopping criteria based on various perspectives have been proposed. 
For example, it is investigated in~\citet{Vlachos2008} that the use of classifier confidence to determine that there are no informative instances remaining in the candidate point set and to stop AL.
In~\citet{Olsson2009}, an intrinsic stopping criterion based on the exhaustiveness of the candidate point set is proposed, that does not depend on a predefined threshold parameter.
A stopping criteria based on {\it Stabilizing Predictions} is proposed in~\citet{Bloodgood2009}, that checks the stability of the current model's predictions on the validation set and decides whether to stop the AL. 
In \citet{Laws2008,Bloodgood2013,Altschuler2019}, stopping criterion based on the change in the F-score is considered.
Criteria called {\it TotalConf} and {\it LeastConf} are proposed in \citet{10.1145/3397271.3401267}, which stop the AL based on the amount of change in the classification confidence (i.e., prediction uncertainty) for unlabeled data.
A method to stop AL based on upper bound of the generalization error is proposed in~\citet{DBLP:conf/aistats/IshibashiH20}.
For Bayesian optimization (BO), stopping criteria based on regret have been proposed~\citep{makarova2022automatic,ishibashi2023stopping,NEURIPS2024_b204de70}. 
Note that each of these studies concerns stopping criteria for active learning and adaptive experimental design for classification, regression and optimization tasks, and are not directly applicable to the LSE problem.

Similar to the LSE problem, the estimation of the excursion set (which is also known as the probability of failure of a system in the industrial world) has also been considered, and different approaches such as sequential experimental design and kriging have been employed to tackle it with criteria targeted to reduce the uncertainty about the level set~\citep{bect2012sequential,Azzimonti2016-sf,Chevalier2014-uu}.
Contour finding, which identifies the contour where a black-box function equals a given threshold, has been developed independently of LSE but is closely related to it~\citep{Cole01012023,NEURIPS2018_01a06836,LI20118683}.
Several extensions of LSE to various situations are also considered, such as LSE under input uncertainty~\citep{Chevalier2013-on,inatsu2020active,iwazaki2020bayesian}, heavy-tailed output noise~\citep{Lyu2021-ze} or heteroscedasticity of outputs~\citep{Zhang2024-yu}, settings that aim at distributionally robust LSE~\citep{inatsu2021active}, dealing with Bernoulli observations~\citep{Letham2022-mt}, considering control over type-I and type-II errors~\citep{Azzimonti2021-wf}, and the setting where the input is composed of both deterministic and uncertain parts~\citep{doi:10.1080/00401706.2024.2394475}.

\section{Level Set Estimation}
Consider an unknown function $f: \cX \rightarrow \mathbb{R}$ where $\cX$ is a finite set of input $\vx$. This is a so-called pool-based problem. The objective of LSE is to classify, given a threshold $\theta \in \mathbb{R}$, whether the outputs $\{f(\vx)\mid \vx \in \cX\}$ corresponding to a given candidate point set $\cX$ exceed $\theta$, using as few datasets as possible. The upper/lower level sets are defined as $H_{\theta} = \{ \vx \in \cX \mid f(\vx) > \theta \}$ and $L_{\theta} = \{ \vx \in \cX \mid f(\vx) \leq \theta\}$, respectively. In LSE, the following procedure is iteratively performed to achieve this objective:
i) Estimate the surrogate function $\hat{f}$ from the obtained dataset.
ii) Utilize the surrogate function to classify each candidate point into any one of the upper-level set, the lower-level set, or the undetermined set.
iii) Select the next search point based on the surrogate function. iv) Query the oracle for the corresponding output of the selected point. v) Add the obtained point to the dataset.

The surrogate function is often modeled by the Gaussian process regression (GPR;~\citet{NIPS1995_7cce53cf}). Consider a set of input-output pairs $S_N=\{(\vx_n, y_n)\}^N_{n=1}$. In GPR, we assume that the function $\hat{f}$ is generated by a Gaussian process (GP) with a mean function $m(\vx)$ and a covariance function $k(\vx,\vx')$. Additionally, the observed output $y$ is assumed to have Gaussian noise with precision parameter $\lambda$ added to the generated function $\hat{f}$. Therefore, in GPR, we consider the following generative model:
$\hat{f}(\vx) \sim \mathcal{N}(m(\vx), k(\vx, \vx')), \quad \mbox{and} \quad y \mid \vx \sim \mathcal{N}(\hat{f}(\vx), \lambda^{-1})$. 
Denoting $\vy := (y_1, y_2, \ldots, y_N)$, the joint distribution of $\vy$ and the output $\hat{f}(\vx^\ast)$ for a new input $\vx^\ast$ can be expressed by the following equation.
\begin{equation}
  \left[    \begin{array}{c}
      \vy  \\
      \hat{f}(\vx^\ast)
    \end{array}
  \right]
  \sim \mathcal{N}\left(
  \left[    \begin{array}{c}
      \vm  \\
      m(\vx^\ast)
\end{array}
  \right]
  ,
  \left[    \begin{array}{cc}
      \tilde{\vK} & \vk(\vx^\ast) \\
      \vk\T(\vx^\ast) & k(\vx^\ast,\vx^\ast)
    \end{array}
  \right]\right),
\end{equation}
where $\tilde{\vK} = \vK + \lambda^{-1}\vI$, $[\vK]_{i,j} = k(\vx_i, \vx_j)$, $\vk(\vx^\ast) = (k(\vx_n, \vx^\ast))^N_{n=1} \in \mathbb{R}^N$, and $\vm = (m(\vx_n))^N_{n=1} \in \mathbb{R}^{N}$.
Therefore, the posterior distribution when observing the dataset $S_N$ is given by $p(\hat{f}(\vx^\ast)\mid \vy) = \mathcal{N}(\hat{f}(\vx^\ast) \mid \mu_N(\vx^\ast), \sigma_N^2(\vx^\ast))$. Here,
\begin{align}
\label{eq:posterior_mean}
\mu_N(\vx^\ast)&=m(\vx^\ast)+\vk\T(\vx^\ast)\tilde{\vK}\inv(\vy-\vm),\; \\  
\label{eq:posterior_variance}
\sigma_N^2(\vx^\ast)&=k(\vx^\ast,\vx^\ast)-\vk\T(\vx^\ast)\tilde{\vK}\inv\vk(\vx^\ast).\;
\end{align}

In LSE using GPR, the next exploration point is determined based on the posterior distribution. Specifically, if we define the acquisition function $\alpha: \cX \rightarrow \mathbb{R}$ parameterized by $p(\hat{f}|\vy)$, the next exploration point is determined by $\vx^{\rm new} = \argmax_{\vx\in \cX} \alpha(\vx ; p(\hat{f}\mid\vy), \theta)$. 
Although there are various types of acquisition functions, such as those based on confidence bounds~\citep{gotovos2013active} and expected improvement for level set estimation~\citep{zanette2019robust}, we focus on a typical approach based on misclassification probability~\citep{bryan2005active}.
Assuming that the true function $f$ is generated from the posterior distribution $p(\hat{f}\, |\,\vy)$, 
the probability $\Pr (\vx \in L_{\theta} )$ can be expressed as follows, where $\Phi(\cdot)$ denotes the cumulative distribution function of the standard Gaussian:
\begin{align}
\Pr (\vx \in L_{\theta} )= \int^{\theta}_{-\infty}p(\hat{f}(\vx)\mid\vy)d\hat{f}(\vx) 
= \Phi\left(\frac{\theta-\mu_N(\vx)}{\sigma_N(\vx)}\right).
\label{eq:ProbZ}
\end{align}
Similarly, $\Pr (\vx \in H_{\theta} ) = 1- \Pr (\vx \in L_{\theta} )$. 
Then, $p^{\rm min}(\vx)=\min\{\Pr (\vx \in H_{\theta} ),\Pr (\vx \in L_{\theta} )\}$ represents the difficulty of classifying the candidate point $\vx$; hence we call this ``misclassification probability''~\citep{bryan2005active}. Similarly, we call $p^{\rm max}(\vx)=\max\{\Pr (\vx \in H_{\theta} ),\Pr (\vx \in L_{\theta} )\}$ ``classification probability''. Therefore, the following acquisition function selects the candidate points that are difficult to classify as the next points of evaluation:
\begin{equation}
    \vx^{\rm new} = \argmax_{\vx \in \cX}\; p^{\rm min}(\vx). 
\label{eq:base_acq_func}
\end{equation}

When classifying candidate points, the standard method is the classification rule based on confidence intervals proposed by \citet{gotovos2013active}. Let $\tilde{H}_\theta$ and $\tilde{L}_\theta$ be estimated upper-level set and lower-level set, respectively. We further introduce an undetermined set $\tilde{U}_{\theta}$.
Then, a candidate point $\vx$ is classified according to the following classification rule:
\begin{align}
    \tilde{H}_\theta =& \{\vx \mid \vx \in \cX, \mu_N(\vx) - \beta \sigma_N(\vx) > \theta\}, \label{eq:standard_classification_rule_H} \\
    \tilde{L}_\theta =& \{\vx \mid \vx \in \cX, \mu_N(\vx) + \beta \sigma_N(\vx) < \theta\}, \label{eq:standard_classification_rule_L} \\
    \tilde{U}_\theta =& \cX \backslash \{\tilde{H}_{\theta} \cup \tilde{L}_{\theta}\}, \label{eq:standard_classification_rule_U} 
\end{align}
where $\beta$ is the parameter that controls the exploration-exploitation trade-off in the acquisition function.

\section{Proposed Acquisition Function and Stopping Criterion}
\begin{figure*}[th]
    \centering
\includegraphics[height=4.5cm,keepaspectratio]{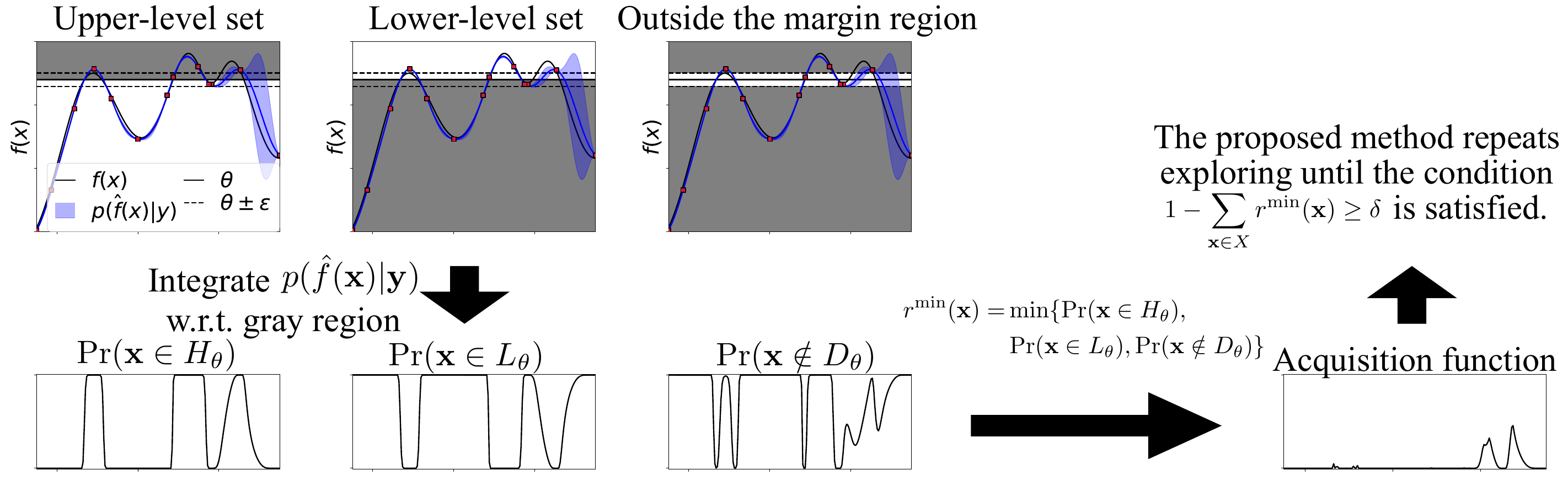}
    \caption{The proposed method selects a candidate point that is difficult to classify and has a low probability of containing the true function value in the margin region and stops LSE when the condition is satisfied.
    }
    \label{fig_concept_diagram}
\end{figure*}

This section describes the acquisition function for LSE proposed in this paper and its stopping criterion. The pseudocode of the proposed LSE procedure is shown in Appendix~\ref{seq:pseudocode}.

\subsection{Proposed Acquisition Function}

The acquisition function based on misclassification probability \eqref{eq:base_acq_func} is an intuitive and natural choice, where points with Bernoulli distribution parameters close to $0.5$, and therefore difficult to classify, are selected as candidates for the next observation.
In this formulation, noting that the cumulative distribution function of the standard Gaussian is $\Phi(0) = 0.5$, Eq.~\eqref{eq:ProbZ} suggests two possible scenarios for the selected candidate points. The first case occurs when the true function value at the candidate point is far from the threshold, but due to insufficient data observed near the candidate point, the posterior distribution's variance is large, making classification difficult ($\Phi\left(\frac{\theta - \mu_N(\vx)}{\sigma_N(\vx)}\right) \to 0.5$ as $\sigma_N(\vx) \to \infty$). In this case, exploring the candidate point reduces the variance of the posterior distribution and increases the classification probability, making it less likely to be explored in subsequent searches. This is the case the exploration offers reasonable information. 

On the other hand, the second case is problematic. The acquisition function~\eqref{eq:base_acq_func} would select points at which the true function values are close to the threshold ($\Phi\left(\frac{\theta - \mu_N(\vx)}{\sigma_N(\vx)}\right) \to 0.5$ as $\mu_N(\vx) \to \theta$). 
In this scenario, the same candidate point is repeatedly explored while other candidate points are ignored. For this issue, the previous study has heuristically used the product of the misclassification probability and the posterior variance as the acquisition function~\citep{bryan2005active}.

In this study, we assume that a margin $\epsilon > 0$ is given\footnote{Here, the margin is assumed to be given, but a method for setting the margin based on the observed data are discussed in Appendix.~\ref{seq:determine_margin}.}, which indicates a tolerance of the accuracy of estimation. If the gap between the true function value $f(\vx)$ and threshold $\theta$ for a candidate point $\vx$ lies within the range $\cE := (- \epsilon/2, \epsilon/2]$, it is considered as a difficult to classify, and that exploring this candidate point will not increase certainty about the classification, and thus the point is removed from the candidate point set. To put it another way, for candidate points that are difficult to classify, we decide to make a concession and perform an $\epsilon$-accurate classification. The notion of the margin is essentially equivalent to $\epsilon$-accuracy introduced in~\citet{NIPS2016_ce78d1da}, but we explicitly utilize it as information for determining the next experimental condition. The difficult-to-classify set is defined as $U_{\theta} = \{ \vx \in \cX \mid (f(\vx)-\theta) \in \cE\}$, and the solution triplet $(\tilde{H}_\theta, \tilde{L}_\theta, \tilde{U}_\theta)$ is $\epsilon$-accurate if $\forall \vx \in \tilde{H}_\theta$ is in $H_\theta$, $\forall \vx \in \tilde{L}_\theta$ is in $L_\theta$, and $\forall \vx \in \tilde{U}_\theta$ is in $U_\theta$.

The probability of $\vx \in U_{\theta}$ is given by
\begin{equation*}
\Pr(\vx \in U_{\theta})
= \int^{\theta+\epsilon/2}_{\theta-\epsilon/2}p(\hat{f}(\vx)\mid\vy)d\hat{f}(\vx) = \Phi\left(\frac{\theta+\epsilon/2-\mu_N(\vx)}{\sigma_N(\vx)}\right)
 - \Phi\left(\frac{\theta-\epsilon/2-\mu_N(\vx)}{\sigma_N(\vx)}\right).
\end{equation*}
Similarly, $\Pr(\vx \notin U_{\theta}) = 1- \Pr(\vx \in U_{\theta})$. 
Then, with $r^{\rm min}(\vx):=\min\{\Pr(\vx \in H_{\theta}), \Pr(\vx \in L_{\theta}), \Pr(\vx \notin U_{\theta})\}$, we redefine the acquisition function as
\begin{equation}
    \vx^{\rm new} = \argmax_{\vx \in \cX}\; r^{\rm min}(\vx).
    \label{eq:proposed_acq_func}
\end{equation}
As shown in Fig.~\ref{fig_concept_diagram}, this acquisition function evaluates not only the probability that a candidate point belongs to the upper/lower level sets but also the probability $\Pr(\vx \notin U_{\theta})$ that the gap does not fall within the range $\cE$. For a candidate point $\vx$ where the gap $f(\vx)-\theta$ is within $\cE$, if the area around the candidate point has not been well explored, $\Pr(\vx \notin U_{\theta})$ increases, and if $p^{\rm min}(\vx)$ is also large, then $r^{\rm min}(\vx)$ increases, leading to the selection of $\vx$. Conversely, if the area around the candidate point has been thoroughly explored, the posterior variance decreases, thus increasing the probability that the gap $f(\vx)-\theta$ falls within $\cE$ and decreasing $\Pr(\vx \notin U_{\theta})$. Therefore, even if $p^{\rm min}(\vx)$ is large and classification is difficult, $r^{\rm min}(\vx)$ becomes small, making it less likely to be chosen as the next point of evaluation.

As similar approaches to the misclassification-based approach, there are entropy-based and variance-based approaches~\citep{bryan2005active,Cole01012023}. 
These acquisition functions share the fundamental idea with the one in \eqref{eq:base_acq_func} and therefore inherit similar issues to those mentioned at the beginning of this section regarding \eqref{eq:base_acq_func}.
Several studies have discussed approaches to address the issues of these acquisition functions~\citep{NEURIPS2018_01a06836,Picheny2010-nz}, but none provide theoretical guarantees on stopping performance, leaving the evaluation of this aspect to empirical analysis.
In contrast, our acquisition function addresses the aforementioned issues while also providing theoretical guarantees on stopping performance, as discussed in the next section.

\subsection{Classification rule and stopping criterion}
We describe the classification rule and stopping method for LSE using the acquisition function Eq.~\eqref{eq:proposed_acq_func}. 
Letting $r^{\rm max}(\vx):=\max\{\Pr(\vx \in H_{\theta}), \Pr(\vx \in L_{\theta}), \Pr(\vx \in U_{\theta})\}$, the proposed classification rule is as follows:
\begin{align}
    \tilde{H}_{\theta} :=& \{\vx \mid \vx \in \cX, r^{\rm max}(\vx)=\Pr(\vx \in H_{\theta})\}, \label{eq:classification_rule_H} \\
    \tilde{L}_{\theta} :=& \{\vx \mid \vx \in \cX, r^{\rm max}(\vx)=\Pr(\vx \in L_{\theta})\}, \label{eq:classification_rule_L} \\
    \tilde{U}_{\theta} :=& \{\vx \mid \vx \in \cX, r^{\rm max}(\vx)=\Pr(\vx \in U_{\theta})\}. \label{eq:classification_rule_D}
\end{align}
As will be discussed later, this classification rule can be considered equivalent to the classification rule of Eqs.~\eqref{eq:standard_classification_rule_H}, \eqref{eq:standard_classification_rule_L}, and \eqref{eq:standard_classification_rule_U} under certain conditions.

The proposed stopping criterion uses a confidence parameter $\delta\; (0 < \delta < 1)$ as a threshold, and LSE is stopped when the following inequality is satisfied:
\begin{equation}
    1-\sum_{\vx \in \cX}r^{\rm min}(\vx)\geq \delta.
    \label{eq:stopping_criterion}
\end{equation}
That is, LSE is stopped when the sum of the acquisition function values for all candidate points becomes small enough. 
At the point of stopping LSE, the following probability inequality holds:
\begin{theorem}
\label{thm:stopping_criterion}
If we assume that $\tilde{H}_{\theta}$,$\tilde{L}_{\theta}$ and $\tilde{U}_{\theta}$ are determined by using the classification rule of Eqs.~\eqref{eq:classification_rule_H},\eqref{eq:classification_rule_L} and \eqref{eq:classification_rule_D}, then the following inequality holds:
\begin{equation} \Pr((\tilde{H}_{\theta},\tilde{L}_{\theta},\tilde{U}_{\theta}) \text{ is $\epsilon$-accurate}) \geq  1-\sum_{\vx \in \cX}r^{\rm min}(\vx).
    \label{eq:ineq_of_sc}
\end{equation}
\end{theorem}

The proof is shown in Appendix~\ref{app:proofs}. From this theorem, when Eq.~\eqref{eq:stopping_criterion} is satisfied, stopping LSE guarantees that $(\tilde{H}_{\theta},\tilde{L}_{\theta},\tilde{U}_{\theta})$ is $\epsilon$-accurate with a probability of at least $\delta$.
Therefore, we refer to the LSE that uses the combination of the proposed acquisition function and stopping criterion as the $(\epsilon, \delta)$-{\it{accurate}} LSE. 
Since the left-hand side of Eq.~\eqref{eq:ineq_of_sc} can be evaluated by sampling functions according to the GP posterior distribution, we provide the tightness of the proposed lower bound in the Appendix~\ref{app:thm_and_lb}.

By using theorem~\ref{thm:stopping_criterion}, we can also guarantee the lower bound of performance measures. Here, we present only the lower bound of the F-score as follows.
\begin{proposition}
\label{prop:lower_bound_performance_measure}
If we assume that $\tilde{H}_{\theta}$,$\tilde{L}_{\theta}$ and $\tilde{U}_{\theta}$ are determined by using the classification rule of Eqs.~\eqref{eq:classification_rule_H},\eqref{eq:classification_rule_L} and \eqref{eq:classification_rule_D}, then the  inequality
\begin{align*}
\text{F-score} \geq& \frac{2 |\tilde{H}_\theta|}{2|\tilde{H}_\theta| + |\tilde{U}_\theta|}
\end{align*}
holds with probability $1-\sum_{\vx \in \cX}r^{\rm min}(\vx)$.
\end{proposition}
The lower bound of other performance measures such as accuracy, recall, precision, and specificity, and their proofs are shown in Appendix~\ref{app:proofs}. In \citet{Qing2022}, the lower bound of the F-score could not be analytically computed, so it is estimated using sampling. In contrast, our method provides lower bounds for the various measures such as F-score. 

The standard classification rule and the proposed classification rule can be considered the same under certain conditions. The standard classification rule can be interpreted as 
follows: $\vx$ is classified into upper-level set when $\Pr(\vx \in H_\theta)> \Phi(\beta)$ is satisfied, $\vx$ is classified into lower-level set when $\Pr(\vx \in L_\theta)> \Phi(\beta)$ is satisfied, and $\vx$ is classified into undetermined set when the both of conditions are not satisfied. Regarding $\Pr(\vx \in U_\theta)$ as $\Phi(\beta)$, the standard classification is equivalent to the proposed classification rule under the assumption that $\Phi(\beta)>0.5$\footnote{The condition $\Phi(\beta)>0.5$ is added because in the standard classification rule, when $\Phi(\beta) \leq 0.5$, there is a possibility that $\vx$ belongs to both the upper-level and lower-level sets. The standard classification rule often uses values such as $\beta=1.96$, which corresponds to $\Phi(\beta)=0.975$, implicitly assuming $\Phi(\beta) > 0.5$}. 

By using the above relationship, we can show that the triplet $(\tilde{H}_{\theta},\tilde{L}_{\theta},\tilde{U}_{\theta})$ is $\epsilon$-accurate in the case of the standard classification rule. Here, $\beta$ in the standard classification rule and $\epsilon$ in the proposed classification rule can be mutually converted. Note that $\beta$ also changes for each $\vx$ in general even if $\epsilon$ is common to all $\vx$ since $\Pr(\vx \in U_\theta)$ varies depending on $\vx$. We denote $\Pr(\vx \in U_\theta)$ as $g(\epsilon \mid \vx)$, then there is an inverse mapping of $g(\epsilon \mid \vx)$ because $g(\cdot \mid \vx): \bR^+ \rightarrow (0, 1)$ is a strictly increasing function with respect to $\epsilon\in\bR^+$. Therefore, the following mutual conversions between $\epsilon$ and $\beta$ hold:
\begin{equation*}
    \beta=\Phi^{-1}\left(g(\epsilon \mid \vx)\right), \quad \epsilon=g\inv(\Phi(\beta) \mid \vx).
\end{equation*}
With these conversions, we can show that the triplet $(\tilde{H}_\theta, \tilde{L}_\theta, \tilde{U}_\theta)$ is $\epsilon$-accurate when we use the standard classification rule as follows:
\begin{corollary}
\label{col:prob_ineq}
We assume that $\tilde{H}_{\theta}$,$\tilde{L}_{\theta}$ and $\tilde{U}_{\theta}$ are determined by using the classification rule of Eqs.~\eqref{eq:standard_classification_rule_H}, \eqref{eq:standard_classification_rule_L}, and \eqref{eq:standard_classification_rule_U} with $\beta$ and $\Phi(\beta) > 0.5$. Let $\tilde{r}^{\rm min}(\vx)=\min\{\Pr(\vx \in H_{\theta}), \Pr(\vx \in L_{\theta}), 1-\Phi(\beta)\}$. Then, the following inequality holds:
\begin{equation}
    \Pr((\tilde{H}_{\theta},\tilde{L}_{\theta},\tilde{U}_{\theta}) \text{ is $g\inv(\Phi(\beta) \mid \vx)$-accurate}) \geq  1-\sum_{\vx \in \cX}\tilde{r}^{\rm min}(\vx).
\end{equation}
\end{corollary}
The proof is shown in Appendix~\ref{app:proofs}. 

\subsection{Choice of $\epsilon$ and $\delta$}

We explain how to determine the parameters $\epsilon$ and $\delta$, and its sensitivity. Regarding $\delta$, the proposed lower bound tends to increase monotonically, and the bound becomes tighter as the true probability of $\epsilon$-accuracy approaches 1 as shown in Appendix.~\ref{app:thm_and_lb}. Therefore, we just have to set $\delta$ close to 1, such as $\delta = 0.99$. Since the stopping time does not change significantly when $\delta$ is close to 1, we can say that the stopping timing tends to be insensitive to the choice of $\delta$. 

On the other hand, it is difficult to set $\epsilon$ appropriately, since it depends on the range of the objective function and the noise variance. In this study, instead of determining $\epsilon$ directly, we determine $\epsilon$ as follows:\footnote{When using a stationary kernel, $\epsilon$ is independent of $\vx$. Please refer to the Appendix.~\ref{seq:determine_margin} for the detailed derivation of this equation.}
\begin{equation}
    \epsilon(\vx) = 2\sqrt{\frac{\lambda^{-1} k(\vx, \vx)}{\lambda^{-1} +L k(\vx, \vx)}}\Phi^{-1}\left(1-\frac{1 - \delta}{2|\cX|}\right),
\end{equation}
where $L$ is a parameter set by the user instead of $\epsilon$, and it can be interpreted as the minimum number of observations required per candidate point, even in cases where classification is not possible. As shown in the Appendix.~\ref{app:L}, $L$ is less sensitive to the range of the function and the noise variance than directly specifying $\epsilon$, making it a more robust parameter.

\subsection{Computational cost}
The proposed stopping criterion only requires the cumulative distribution function (CDF) of the standard normal distribution, and it does not require any sampling. Since the CDF of the standard normal distribution can be efficiently computed using libraries, the computational cost increases only linearly with the number of candidate points. In contrast, F-scores sampling (FS)~\citep{Qing2022} requires sampling functions from the posterior distribution, resulting in a quadratic increase in computational cost with respect to the number of candidate points. Therefore, compared to FS, the proposed stopping criterion remains computationally feasible even as the number of candidate points increases.

\section{Experimental results}
\label{sec:experimental_results}
In this section, we demonstrate the effectiveness of the proposed acquisition function using both synthetic data and a practical application for estimating the red zone in silicon ingots.\footnote{In these experiments, we use a Macbook Pro with Apple M1 Max (10-core CPU, 32-core GPU and 32GB memory), and implemented with Python and library GPy~\citep{gpy2014}. The code is available at \url{https://github.com/hideaki-ishibashi/stopping_LSE}} 
In all experiments, the threshold $\theta$ that defines the level set is a pre-fixed value, but the results of setting $\theta$ to several different values are also shown in the Appendix~\ref{app:theta}. Note that consistent results are obtained even when different thresholds are used.

\subsection{LSE for test functions}
Aiming at demonstrating the applicability of the proposed method across functions with various shapes, we evaluate the proposed method using test functions commonly used as benchmarks in the study of optimization algorithms. The test functions employed in this experiment are the {\tt{Rosenbrock}}, {\tt{Branin}}, and {\tt{Cross in tray}} functions\footnote{\url{https://www.sfu.ca/~ssurjano/optimization.html}}, each representing a different landscape: one with a single local minimum with a spherical vicinity, one with a valley-like structure, and one with multiple local minima. Additional results are discussed in the Appendix~\ref{app:testfunc}. For each function, thresholds are set, and candidate points exceeding these thresholds are considered part of the true upper set, while those below are viewed as the true lower set. The thresholds are set as follows: 
$\theta=100$ for the {\tt{Rosenbrock}} function and the {\tt{Branin}} function, $\theta=-1.5$ for the {\tt{Cross in tray}} function. 
Although these test functions have continuous domains, they are discretized into a grid of $20 \times 20 = 400$ points, which serve as observation candidates.  In the Appendix~\ref{app:effect_n_candidates}, we show that the stopping timing of the proposed method tends not to change even if the number of candidate points increases. Any of these points may be selected by acquisition functions, and repeated selections of the same points is allowed. Gaussian noise is added to the observations, with standard deviations set according to the range of each test function: $\sigma_{\rm noise}=30$ for the {\tt{Rosenbrock}} function, $\sigma_{\rm noise}=20$ for the {\tt{Branin}} function, and $\sigma_{\rm noise}=0.01$ for the {\tt{Cross in tray}} function.

The proposed method is evaluated based on both the performance of the acquisition function and the efficiency achieved by early stopping.  
Generally, the performance of the LSE acquisition function is assessed using the F-score, which compares the predicted upper/lower level sets to the true upper/lower level sets over the candidate points. 
Not all candidate points may be classified in every search due to the classification rules.  
For the evaluation purpose, unclassified candidate points are assigned to the upper or lower level sets only for the F-score calculation if the posterior mean of GP exceeds or falls below the threshold, respectively. 
The performance of the acquisition function is evaluated based on the mean and variance of the convergence speed of the F-score when the LSE algorithm is executed using five randomly selected initial points.
On the other hand, the effectiveness of the stopping criterion is evaluated based on the stopping time and the F-score at that moment. In other words, a good stopping criterion allows the algorithm to stop with fewer observations while achieving a high F-score.

 \paragraph{Comparison methods} 
The level set estimation problem is also related to Bayesian optimization~\citep{Nguyen2020-nb} and bandit problems~\citep{10.5555/3104322.3104451}, and its applications range from brain science~\citep{MARCHINI20041203} to astronomy~\citep{Beaky1992-vj,Nikakhtar2018-rk} for example. 
Various algorithms (acquisition functions) have been proposed, but in this study, we compare those that are considered particularly major types and important in terms of practical performance:
in addition to uncertainty sampling (US), which selects points that maximize the predictive variance of the Gaussian process as a baseline, we consider Straddle~\citep{bryan2005active}, MILE~\citep{gotovos2013active}, RMILE~\citep{gotovos2013active}, and MELK~\citep{DBLP:conf/aistats/MasonJMCJN22}, which is a recently proposed sampling method based on experimental design. Although many other methods exist, they do not consistently outperform those mentioned here. It should also be noted that the main focus of this paper is the proposal of an acquisition function equipped with a stopping criterion. In these acquisition functions, candidate points are classified according to Eqs.~\eqref{eq:standard_classification_rule_H}, \eqref{eq:standard_classification_rule_L}, and \eqref{eq:standard_classification_rule_U}, where $\beta = 1.96$. On the other hand, in the proposed method, candidate points are classified according to Eqs.~\eqref{eq:classification_rule_H}, \eqref{eq:classification_rule_L}, and \eqref{eq:classification_rule_D}. The same value of $\beta = 1.96$ is used for Straddle, MILE, and RMILE. MELK assumes that candidate points are not reclassified and that multiple points are sampled simultaneously. Following the settings of the previous study~\citep{DBLP:conf/aistats/MasonJMCJN22}, MELK samples 10 points at a time without reclassifying candidate points. In contrast, other methods reclassify candidate points and sample one point at a time.
RMILE's robust adjustment parameter $\nu$ is set to $\nu = 0.1$ according to the previous studies~\citep{zanette2019robust,inatsu2021active}. The proposed acquisition function also requires setting a parameter $L$, which is conservatively set to $L=5$ to address the complex shape function based on the experimental results in Appendix~\ref{app:L}. The threshold for the proposed stopping criterion is set at $\delta=0.99$. To evaluate the stopping criterion of the proposed method, we consider two stopping criteria. One is a standard stopping criterion which stops LSE when all candidate points are classified' (we call this the {\it{fully classified (FC)}} criterion), and the other is a stopping criteria based on sampling F-scores~\citep{Qing2022} (the criterion referred to as F-score Sampling (FS)). The stopping times of these stopping criteria are compared with the stopping times when using the proposed acquisition function and stopping criterion.
In the FS criterion, as hyperparameters, we need to set the desired F-score and the probability of exceeding that F-score. In this experiment, we set the desired F-score and the probability to $0.95$ and $95$\% (that is, $5$th percentile). 

\paragraph{Hyper-parameter setting} 
In this experiment, we consider a GP with the mean function set to $\theta$ and the covariance function defined by a Gaussian kernel $k(\vx,\vx')=\rho \exp(-\frac{1}{2 l^2}\|\vx-\vx'\|^2)$. The mean function is set to $\theta$ to ensure that, in the absence of any observed data, the probability of unobserved candidate points being classified into either the upper or lower level set is $50\%$. This setting can be adjusted based on any prior knowledge available in real applications. The variance $\rho$ of the Gaussian kernel, the kernel width $l$, and the noise precision $\lambda$ are hyperparameters, which are estimated by maximizing the marginal likelihood of the observed data each time a search is conducted using LSE. To prevent large fluctuations in the hyperparameters with each search, gamma priors are placed on $\rho$ and $l$. Additionally, the noise precision $\lambda$ is constrained within the range $[10^{-6}, 10^6]$ to prevent it from becoming infinite.

\begin{figure*}[th]
    \centering
    \begin{tabular}{ccc}
    \includegraphics[height=4cm,width=4.2cm]{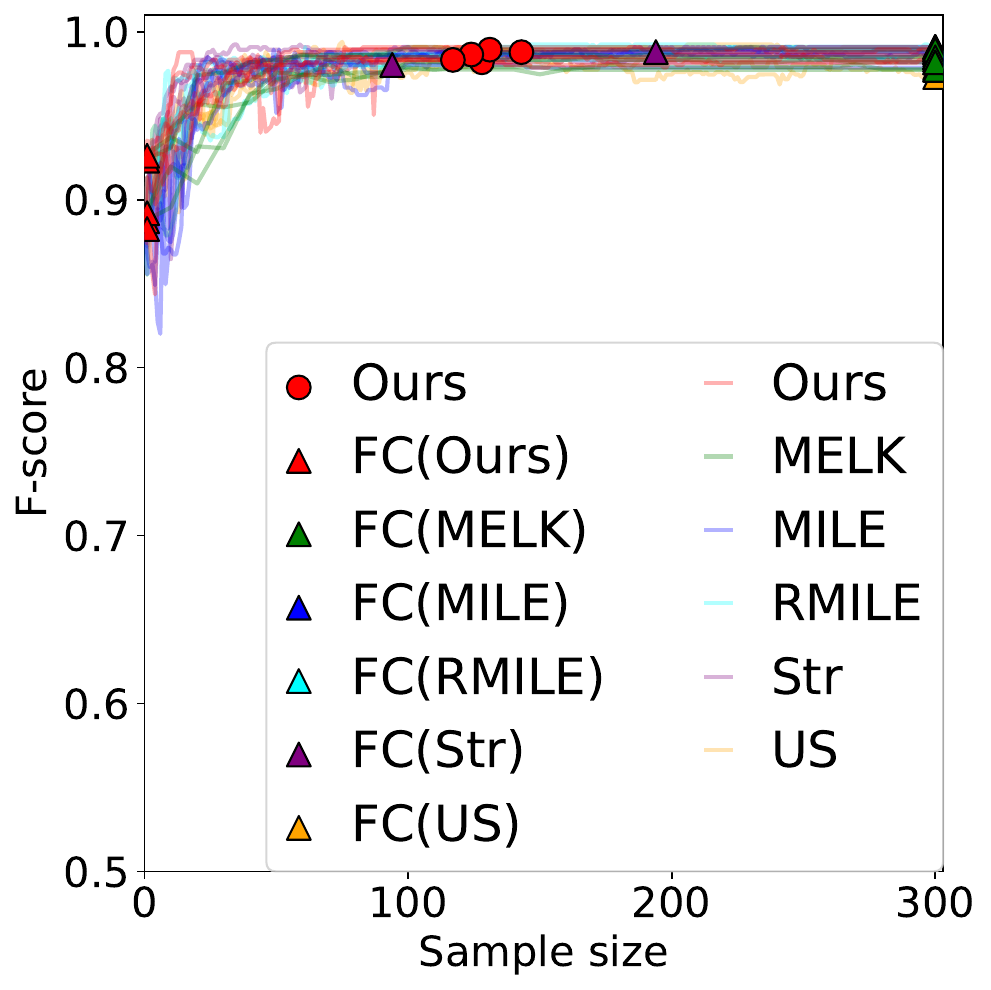}&    
    \includegraphics[height=4cm,width=4.2cm]{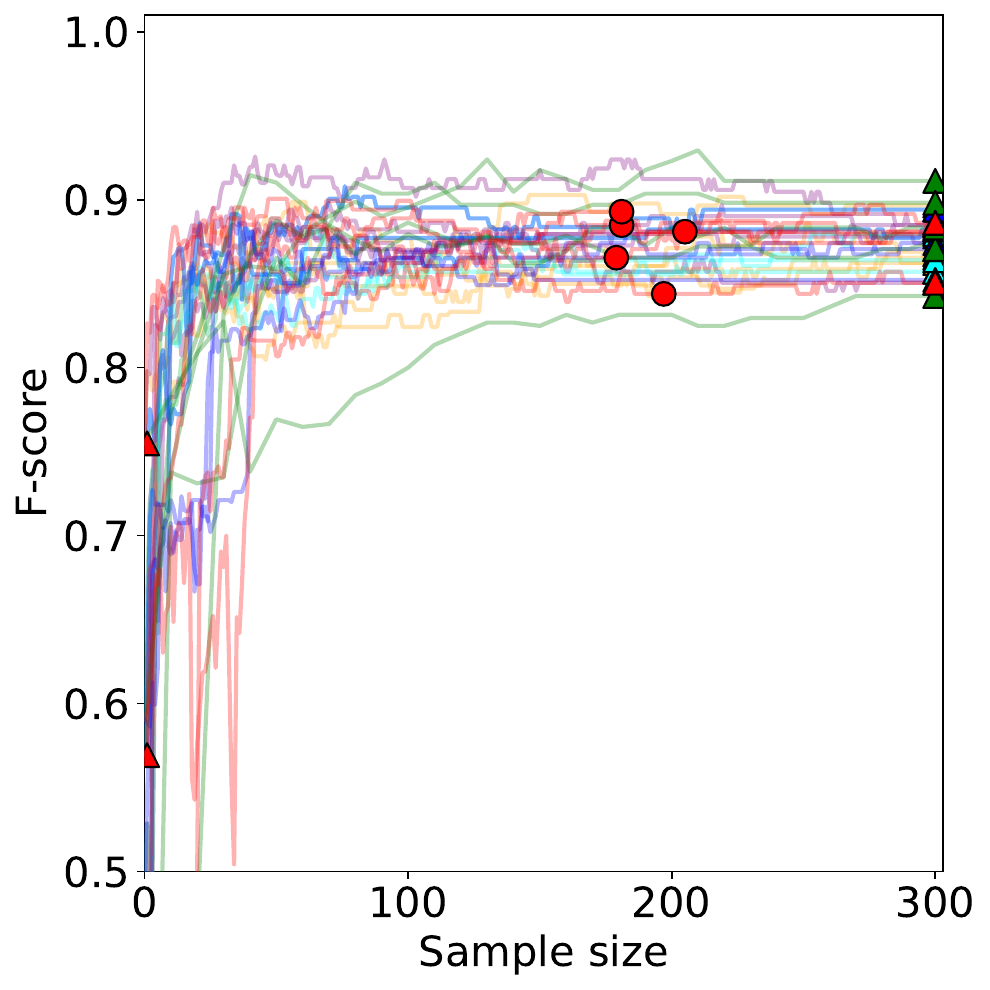}&
    \includegraphics[height=4cm,width=4.2cm]{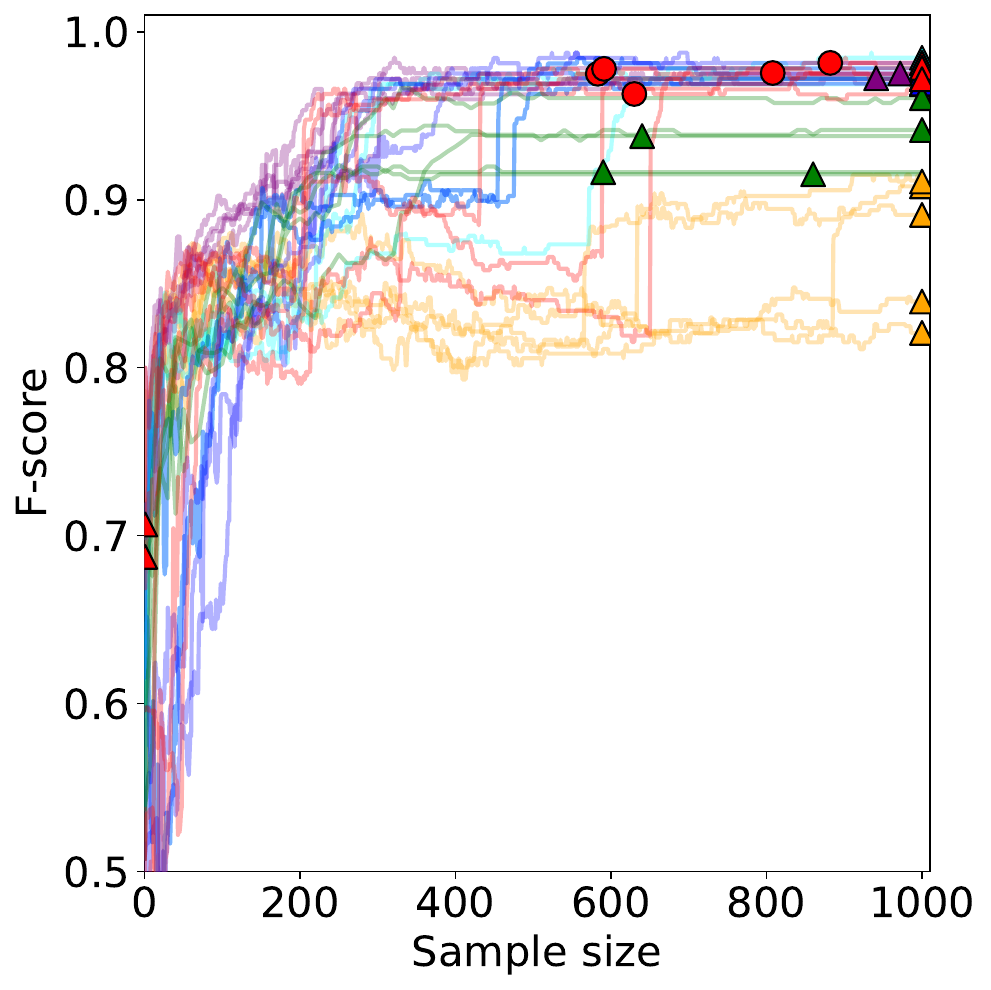}\\
    (a) {\tt{Rosenbrock}}& (b) {\tt{Branin}}& (c) {\tt{Cross in tray}}  \\
    \includegraphics[height=4cm,width=4.2cm]{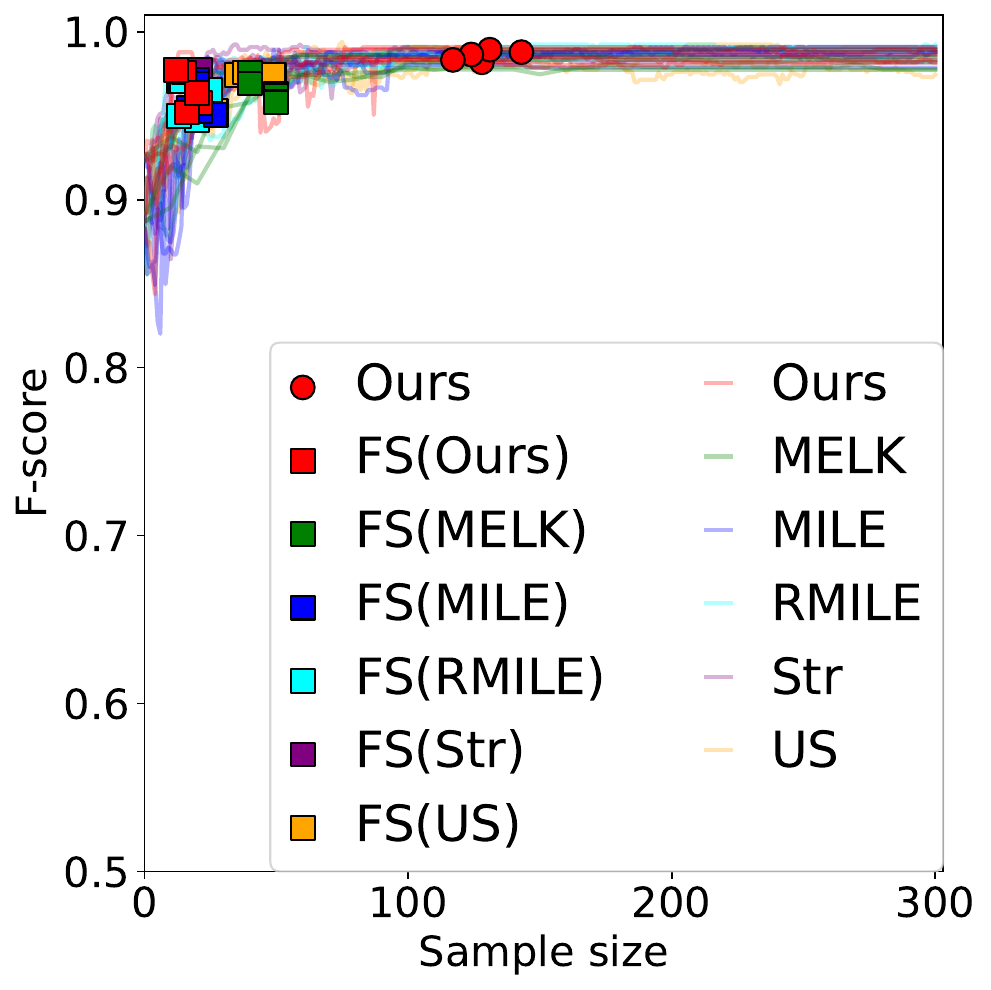}&    
    \includegraphics[height=4cm,width=4.2cm]{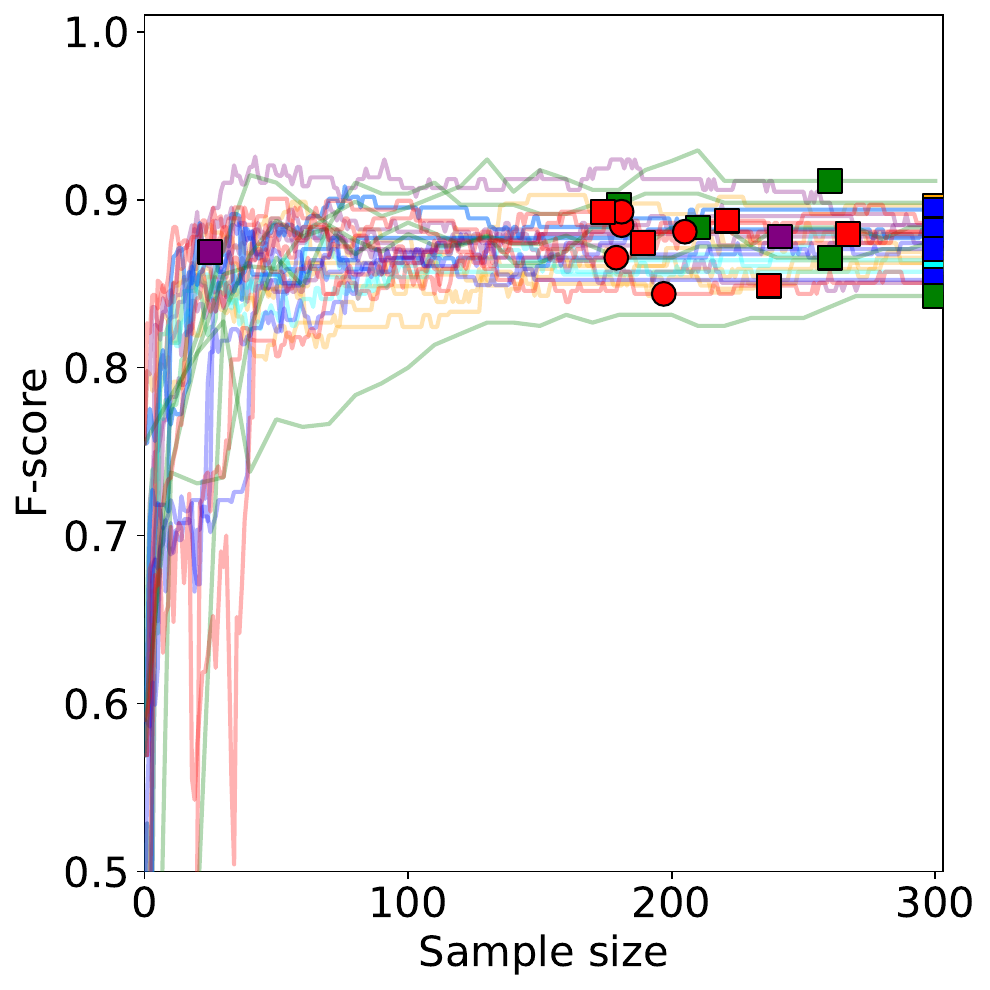}&
    \includegraphics[height=4cm,width=4.2cm]{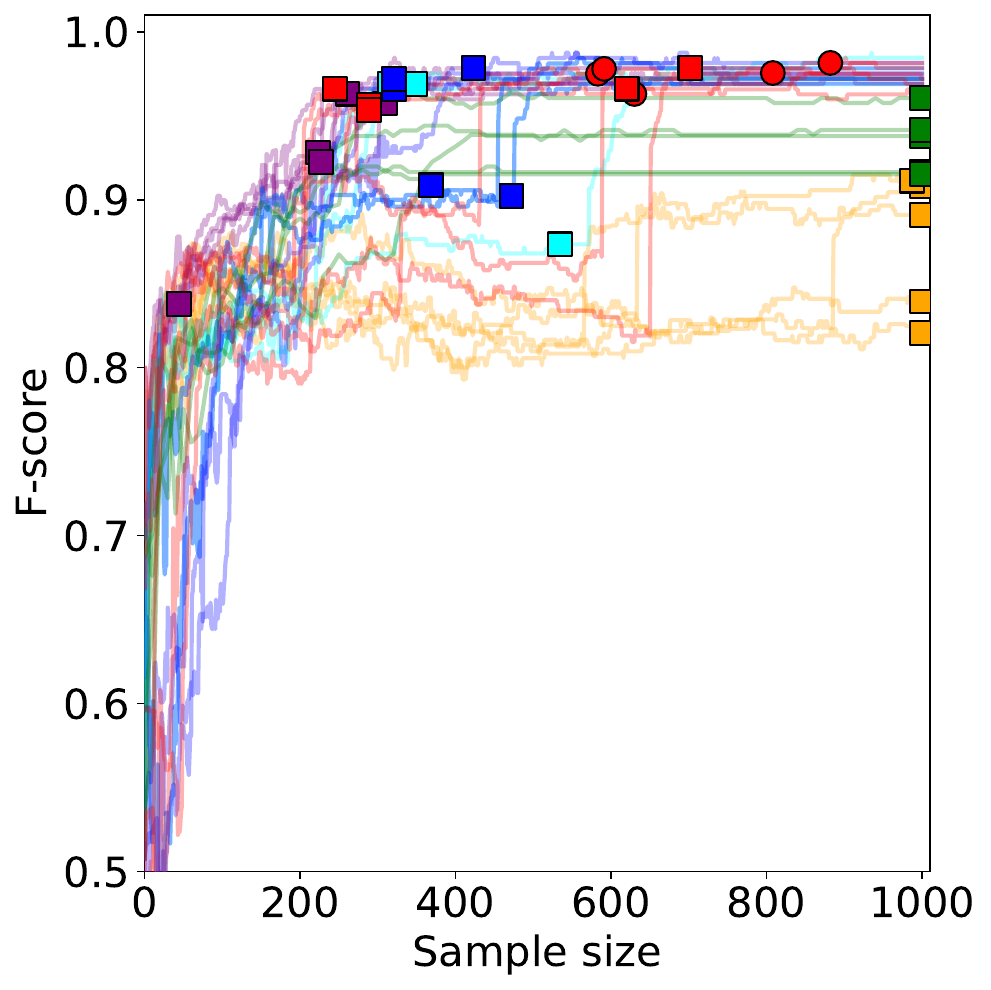}\\
    (d) {\tt{Rosenbrock}}& (e) {\tt{Branin}}& (f) {\tt{Cross in tray}}  \\
    \end{tabular}
    \caption{
    F-scores using each acquisition function and stoped timings with the proposed (Ours), F-score sampling (FS) and fully classified (FC) criteria for test functions. (a)--(c) Stopping time of FC and Ours. (d)--(f) Stopping times of FS and Ours.
    }
    \label{fig_result_test_func}
\end{figure*}

\paragraph{Results}
The F-scores for each acquisition function and the respective stopping timings are shown in Figs.~\ref{fig_result_test_func}. 
Although there are slight differences between individual test functions, no acquisition function, including the proposed method, significantly outperforms the baseline US or is markedly inefficient. Thus, it is crucial to stop LSE at the right moment when the F-scores have converged to enhance the efficiency of LSE. 
Regarding the stopping timings, the fully classified criterion often fails to stop the LSE even when the budget is fully utilized except for MELK in {\tt{Cross in tray}} function. The inability of the FC criterion to stop is due to the occurrence of difficult-to-classify candidate points when function values at candidate points equal the threshold, and classification becomes more challenging as noise is added to the data, distancing the function values from the threshold. 
In MELK, it is sometimes possible to stop LSE even when the fully classified stopping criteria are used, as shown in Fig.~\ref{fig_result_test_func}(c), since it does not reclassify candidate points. However, the F-score of MELK may be lower than that of other methods, as it cannot correct candidate points that have been misclassified. 
In the FS criterion, as shown in Fig.~\ref{fig_result_test_func}(d), when the F-score converges, LSE can be stopped regardless of the acquisition function used. However, as shown in Figs.~\ref{fig_result_test_func}(f), despite the F-score not having converged, FS stops LSE. This is because the desired F-score is set to 0.95 in this experiment. This value was suitable when the F-score converged to 1, as in case {\tt{Rosenbrock}}. However, it was not suitable when the noise was high, and the F-score did not converge to 1, as in cases of {\tt{Branin}} function. Moreover, when noise was low, FS stopped LSE before the F-score converged to 1. Therefore, it is necessary to set the appropriate desired F-score according to the situation in the FS. Furthermore, under the FS criterion, despite setting the desired threshold to $0.95$, the actual F-scores at the stopping time tend to be lower than $0.95$ as shown in Fig.~\ref{fig_result_test_func}(e) and (f). Therefore, in practical applications, the desired value should be set higher than expected.
In contrast, the proposed method, despite using the same parameters, can stop at the time when convergence occurs, regardless of where the F-score converges. This demonstrates that, compared to FS, the proposed method does not require the parameters of the stopping criterion to be finely tuned to the specific problem.

\subsection{Red-zone estimation of silicon ingots}
We demonstrate the effectiveness of the proposed method when using LSE to estimate the red zone in silicon ingots used in solar cells. The objective in this problem is to estimate regions contaminated with impurities (called the red zone) that are unsuitable for solar cell production. Typically, red zone estimation is performed through spatial mapping with measurement points placed in a regular grid, which is very time-consuming. Recently, the efficiency of red zone estimation using LSE has been proposed~\citep{Hozumi2023AdaptiveDA}. 
The data used in the experiments consist of lifetime measurements taken at grid points on two different types of silicon ingots, with each ingot measured at a grid of $161 \times 121$ points~\citep{Kutsukake_2015}. Hereinafter, the lifetime data from the first silicon ingot will be referred to as {\tt{Lifetime1}}, and from the second ingot as {\tt{Lifetime2}}. In both cases, the threshold is set to $\theta=230$.

The performance of LSE methods are evaluated, in the similar manner to the test functions, by observing the F-scores, and the F-score at the stopping timing, with initial values changed randomly five times. The same methodology as for the test functions was used for comparison, but once a candidate point is selected, it is not selected again to estimate the noise because we have only one observation for each point. 
The parameters of the acquisition functions, the GP prior, and the hyperparameter estimation method were employed in the same manner as for the test functions.

\begin{figure*}[th]
    \centering
    \begin{tabular}{cc}
    \includegraphics[height=4cm,width=4.2cm]{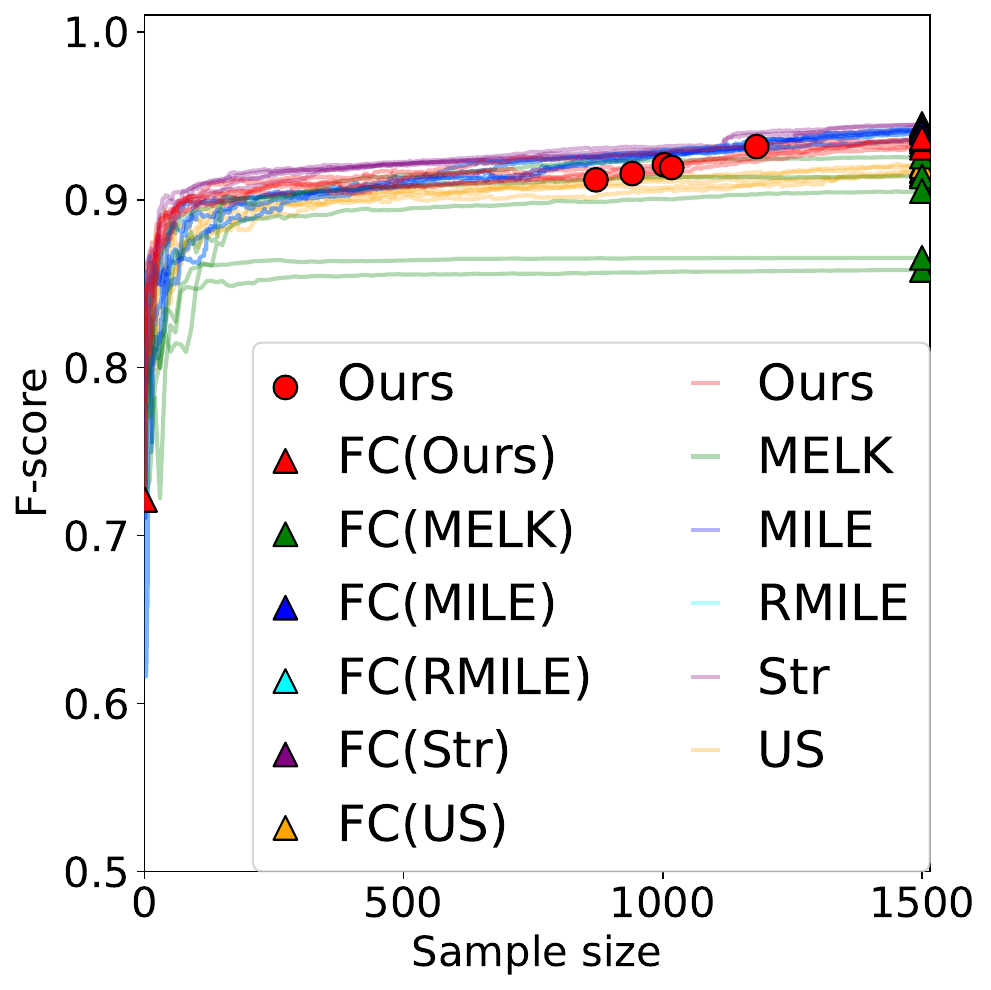}&
    \includegraphics[height=4cm,width=4.2cm]{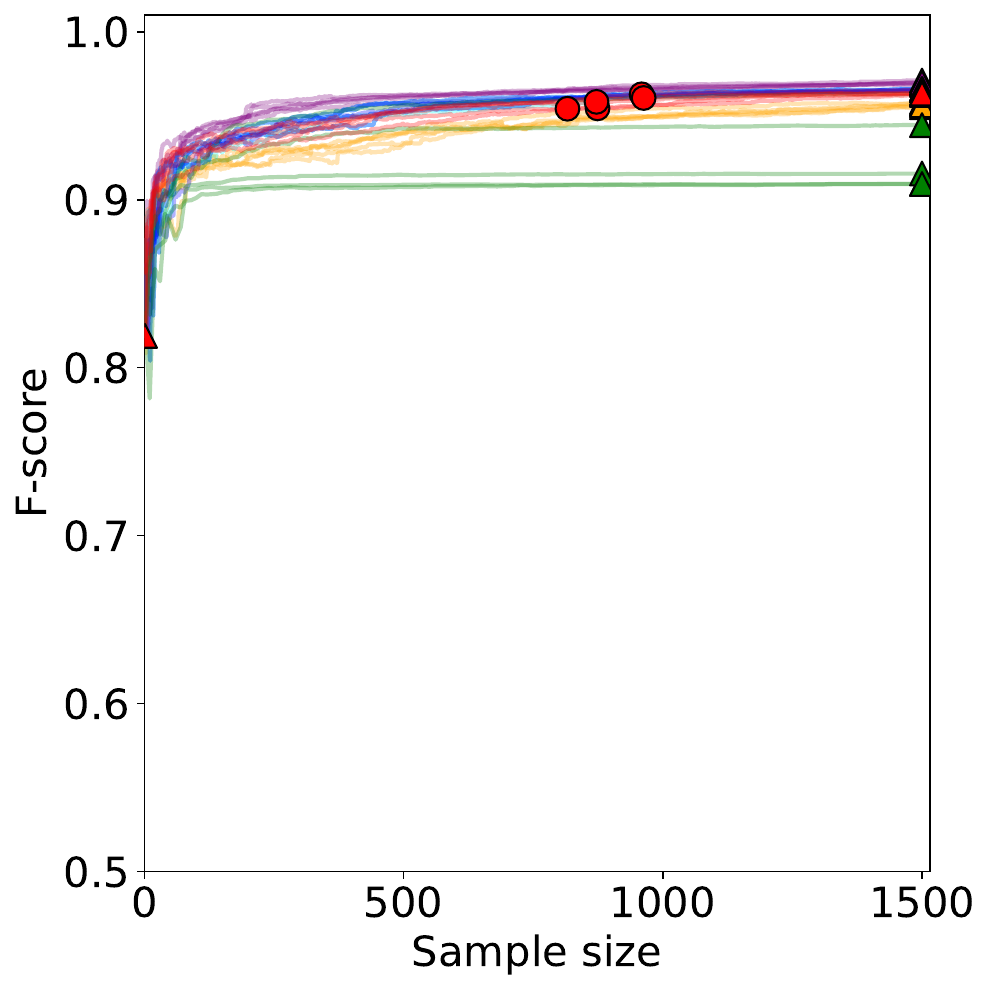} \\
    (a) {\tt{Lifetime1}}& (b) {\tt{Lifetime2}} \\
    \includegraphics[height=4cm,width=4.2cm]{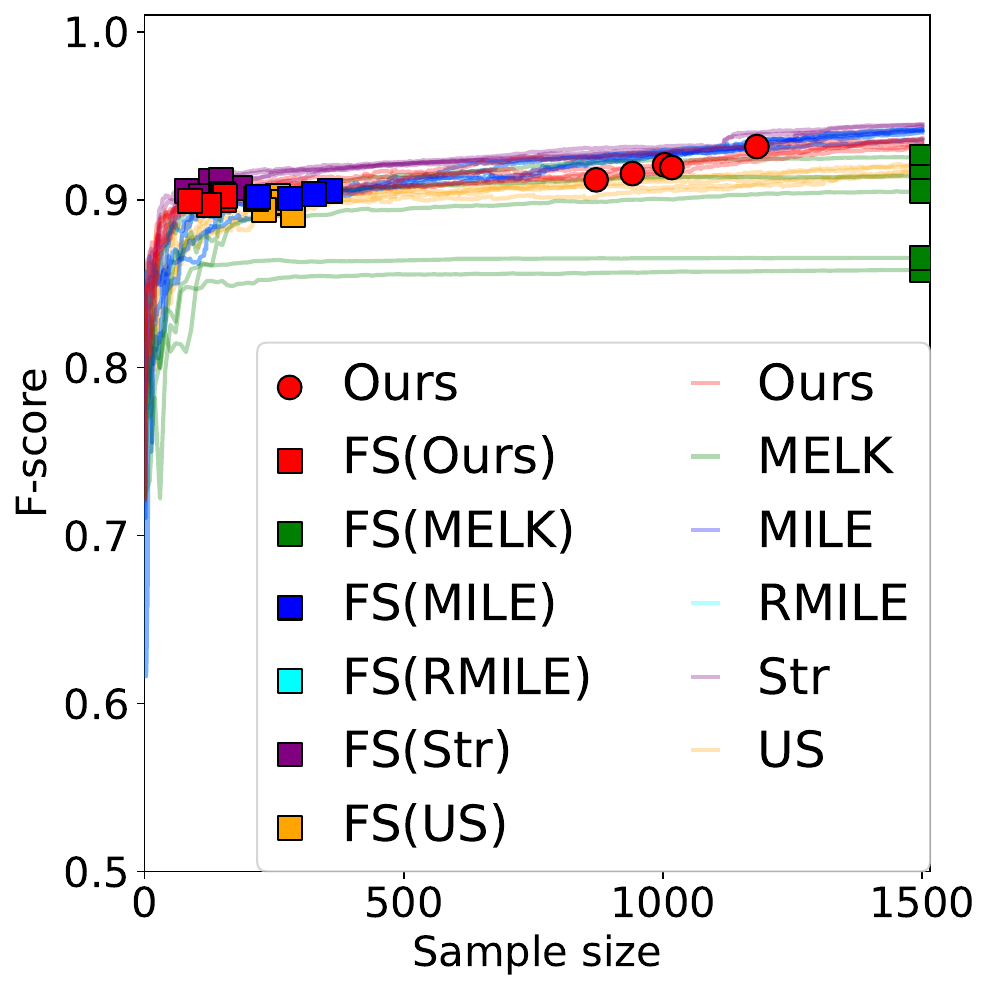}&
    \includegraphics[height=4cm,width=4.2cm]{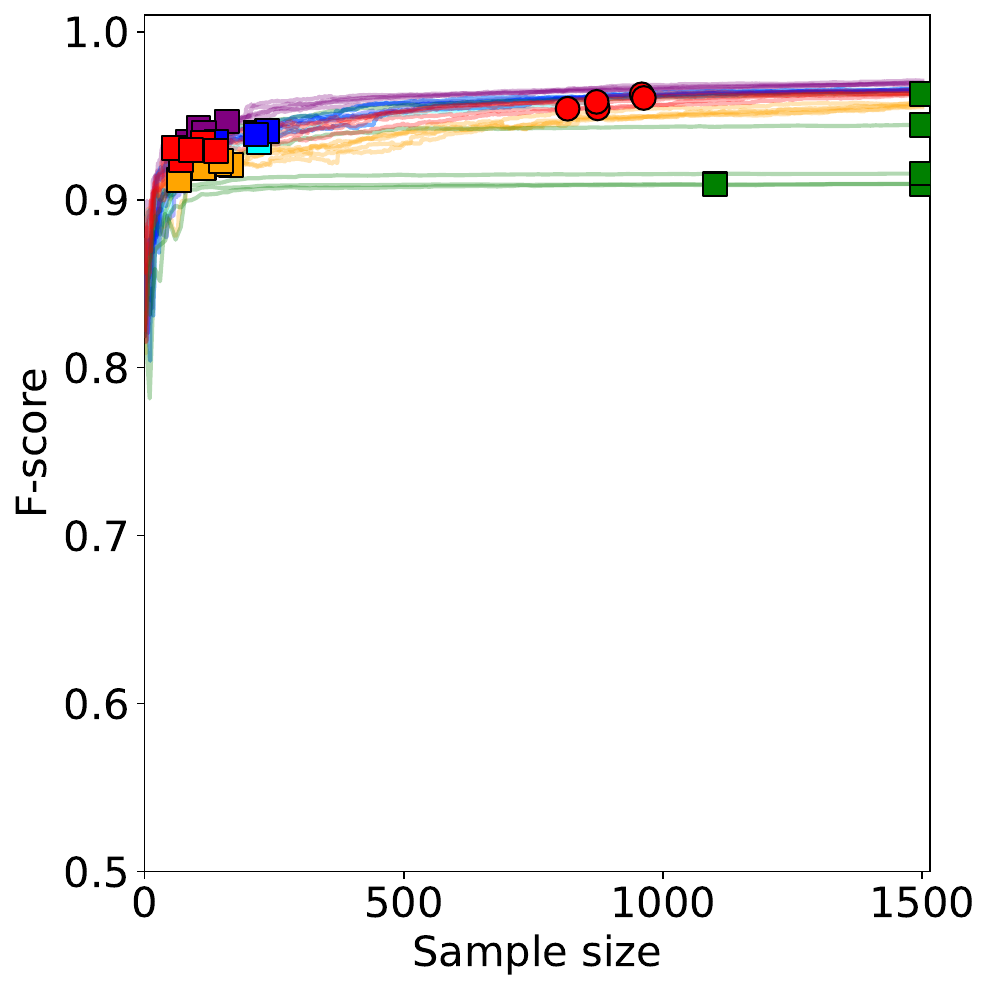} \\
    (c) {\tt{Lifetime1}}&    
    (d) {\tt{Lifetime2}} \\   
    \\
    \end{tabular}
    \caption{
    F-scores using each acquisition function and stoped timings with the proposed (Ours), F-score sampling (FS) and fully classified (FC) criteria for red zone estimation. (a), (b) Stopping time of FC and Ours. (c), (d) Stopping times of FS and Ours.
    }
    \label{fig_result_lifetime}
\end{figure*}

As shown in Figs.~\ref{fig_result_lifetime}, the transition of the F-score shows no significant differences regardless of the acquisition function used, except for MELK. In MELK, the F-score tends to converge to a low value. This is because MELK does not reclassify candidate points.
In the perspective of the stopping timings, the FC stopping criterion fails to stop even after the entire budget is used. On the other hand, the proposed criterion allows for early stopping once the F-scores converge. This is likely due to measurement noise, leading to difficult-to-classify candidate points. Thus, the proposed method effectively stops the LSE in red zone estimation. FS criterion stops LSE earlier than the proposed criterion. However, the F-score continues to gradually increase even after FS stops, and the appropriate stopping point varies depending on the situation.

\section{Discussion and Conclusion}
\label{sec:limitations}
In this paper, we proposed an acquisition function for level set estimation by directly modeling the difficulty of classifying into upper/lower sets. The proposed acquisition function is based on the notion of the $\epsilon$-accuracy~\citep{NIPS2016_ce78d1da}, and an adaptive determination method of the $\epsilon$ parameter is proposed. A stopping criterion for the algorithm was also proposed. When applied to both synthetic and real data, the proposed method performed comparably to existing methods in terms of acquisition function performance and was able to stop the algorithm early even in the presence of observation noise. Empirically, the proposed method tends to be conservative. This behavior can be beneficial in some cases but harmful in others. Balancing theoretical guarantees with more aggressive stopping remains an open problem for future work.
Another direction for future work future direction is extending the method to query-based problems where evaluation points are selected from a continuous domain~\citep{pmlr-v89-shekhar19a}. %
Additionally, extending the method to high-dimensional problems is also an important issue.

\section*{Acknowledgments}
This work was supported by Grants-in-Aid from the Japan Society for the Promotion of Science (JSPS) for Scientific Research (KAKENHI grant nos. JP23K28146 and JP24K20836 to K.M., JP24K15088 to H.I., and JP23K24909, 25H01494 and JPMJMI21G2 to H.H.).

\bibliography{main}
\bibliographystyle{tmlr}

\newpage

\appendix

\section{Proof of Theorem~\ref{thm:stopping_criterion}, Proposition~\ref{prop:lower_bound_performance_measure} and Corollary~\ref{col:prob_ineq}}
\label{app:proofs}

Let $\one_A(a)$ be the indicator function that returns $1$ if a certain input $a$ is included in the set $A$, and $0$ otherwise.
Let binary variables $z(\vx)$ and $w(\vx)$ as
$z(\vx):=\one_{\mathbb{R}^+}(\hat{f}(\vx)-\theta)$, and $ w(\vx):= \one_{\cE}(\hat{f}(\vx) - \theta),
$ respectively. 

To prove the theorem~\ref{thm:stopping_criterion}, proposition~\ref{prop:lower_bound_performance_measure} and corollary~\ref{col:prob_ineq}, we introduce the following two lemmas.
\begin{lemma}
\label{lem:general_prob_ineq}
Let the candidate point set be $\mathcal{X}$, then, the following holds for $\gamma(\vx)=\max\{p^{\rm min}(\vx), \Pr(\vx \in U_{\theta})\}$ and $\eta(\vx)=\one_{\mathbb{R}^+}(\Pr(\vx \in U_{\theta}) - p^{\rm max}(\vx))$:
\begin{equation*}
    \Pr(\forall \vx \in \mathcal{X}, \left|z(\vx) - \bE \left[z(\vx)\right]\right|\leq  \gamma(\vx), w(\vx)\geq \eta(\vx)) \geq  1-\sum_{\vx \in X}r^{\rm min}(\vx).
\end{equation*}
\end{lemma}
\begin{proof}
For notational simplicity, we introduce shorthand notations. $p^0(\vx) := \Pr (\vx \in L_{\theta} ) =\Pr(z(\vx)=0), \; p^1(\vx) := \Pr(z(\vx)=1)=1-p^0(\vx)$ and $q^1(\vx) := \Pr(\vx \in U_{\theta}) = \Pr(w(\vx)=1), q^{0}(\vx) := \Pr(w(\vx)=0)=1-q^1(\vx)$. 
Let's consider the joint distribution of $z(\vx)$ and $w(\vx)$. Specifically, let $p^{00}(\vx):=\Phi\left(\frac{\theta-\epsilon/2-\mu_N(\vx)}{\sigma_N(\vx)}\right)$, $p^{01}(\vx):=\Phi\left(\frac{\theta-\mu_N(\vx)}{\sigma_N(\vx)}\right) - \Phi\left(\frac{\theta-\epsilon/2-\mu_N(\vx)}{\sigma_N(\vx)}\right)$, $p^{11}(\vx):=\Phi\left(\frac{\theta+\epsilon/2-\mu_N(\vx)}{\sigma_N(\vx)}\right) - \Phi\left(\frac{\theta-\mu_N(\vx)}{\sigma_N(\vx)}\right)$, $p^{10}(\vx):=1 - \Phi\left(\frac{\theta+\epsilon/2-\mu_N(\vx)}{\sigma_N(\vx)}\right)$. Then, the joint distribution of $z(\vx)$ and $w(\vx)$ is given as
\begin{equation}
\Pr(z(\vx), w(\vx)) =\prod^{1,1}_{i=0,j=0}(p^{ij}(\vx))^{\delta(i-z(\vx))\delta(j-w(\vx))},
\end{equation}
where $\delta(a) = 1$ if and only if $a=0$. 
This distribution considers the probability of $p(f(\vx)\mid\vy)$ when partitioning the range of $f(\vx)$ into four regions: $f(\vx)\leq\theta-\epsilon/2$, $\theta-\epsilon/2 < f(\vx) \leq \theta$, $\theta < f(\vx) \leq \theta+ \epsilon/2$, and $f(\vx) > \theta+ \epsilon/2$. Integrating the value of $p(f(\vx)\mid\vy)$ over each region results in the probability in that region. Furthermore, considering the marginal probabilities for each region, we can show that $p^0(\vx) = p^{00}(\vx) + p^{01}(\vx)$ and $p^1(\vx) = p^{10}(\vx) + p^{11}(\vx)$. Similarly, $q^0(\vx) = p^{00}(\vx) + p^{10}(\vx)$ and $q^1(\vx) = p^{01}(\vx) + p^{11}(\vx)$ hold.

The possible values for $z(\vx)$ are only $0$ or $1$, and given that $\Pr(z(\vx)=0, w(\vx)=0) = p^{00}(\vx)$ and $\Pr(z(\vx)=1, w(\vx)=0) = p^{10}(\vx)$, the cumulative distribution function for $z(\vx)$ can be expressed as follows:
\begin{align}
    \Pr(z(\vx) \leq b, w(\vx)=0) =& \left\{
    \begin{array}{ll}
    0 & b < 0, \\
    p^{00}(\vx) & 0 \leq b < 1,\\
    q^0(\vx) & b \geq 1. \\
    \end{array}
    \right.
\end{align}
Similarly, we have
\begin{align}
    \Pr(z(\vx) \geq a, w(\vx)=0) =& \left\{
    \begin{array}{ll}
    q^0(\vx) & a \leq 0, \\
    p^{10}(\vx) & 0 < a \leq 1,\\
    0 & a > 1. \\
    \end{array}
    \right.
\end{align}
For $\gamma(\vx)>0$, using the notation $b=\mathbb{E}[z(\vx)]+\gamma(\vx)$, from $\mathbb{E}[z(\vx)]=p^1(\vx)$ we have
\begin{align}
    \Pr(z(\vx) - \mathbb{E}[z(\vx)]\leq \gamma(\vx), w(\vx)=0) =& \left\{
    \begin{array}{ll}
    0 & p^1(\vx)+\gamma(\vx) < 0, \\
    p^{00}(\vx) & 0 \leq p^1(\vx)+\gamma(\vx) < 1,\\
    q^0(\vx) & p^1(\vx)+\epsilon \geq 1, \\
    \end{array}
    \right. \\
    =& \left\{
    \begin{array}{ll}
    0 & \gamma(\vx) < - p^1(\vx), \\
    p^{00}(\vx) & -p^1(\vx) \leq \gamma(\vx) < 1-p^1(\vx),\\
    q^0(\vx) & \gamma(\vx) \geq 1-p^1(\vx). \\
    \end{array}
    \right.
\end{align}
From $\gamma(\vx)>0$ and $p^0(\vx)=1-p^1(\vx)$, we also have
\begin{align}
    \Pr(z(\vx) - \mathbb{E}[z(\vx)]\leq \gamma(\vx), w(\vx)=0)
    =& \left\{
    \begin{array}{ll}
    p^{00}(\vx) & 0 < \gamma(\vx) < p^0(\vx),\\
    q^0(\vx) & \gamma(\vx) \geq p^0(\vx). \\
    \end{array}
    \right.
\end{align}
Similarly, with $a=\mathbb{E}[z(\vx)]-\gamma(\vx)$, we have
\begin{align}
    \Pr(z(\vx) - \mathbb{E}[z(\vx)]\geq -\gamma(\vx), w(\vx)=0)
    =& \left\{
    \begin{array}{ll}
    q^0(\vx) & p^1(\vx)-\gamma(\vx) \leq 0, \\
    p^{10}(\vx) & 0 < p^1(\vx)-\gamma(\vx) \leq 1,\\
    0 & p^1(\vx)-\gamma(\vx) > 1 \\
    \end{array}
    \right. \\
    =& \left\{
    \begin{array}{ll}
    q^0(\vx) & \gamma(\vx) \geq p^1(\vx), \\
    p^{10}(\vx) & -p^0(\vx) \leq \gamma(\vx) < p^1(\vx),\\
    0 & \gamma(\vx) < -p^0(\vx). \\
    \end{array}
    \right.
\end{align}
Then, from $\gamma(\vx)>0$, 
\begin{align}
    \Pr(z(\vx)-\mathbb{E}[z(\vx)] \geq -\gamma(\vx), w(\vx)=0) =& \left\{
    \begin{array}{ll}
    q^0(\vx) & \gamma(\vx) \geq p^1(\vx), \\
    p^{10}(\vx) & 0 < \gamma(\vx) < p^1(\vx),
    \end{array}
    \right.
\end{align}
We also have the following relationship:
\begin{align}
    \Pr(a \leq z(\vx) \leq b, w(\vx)=0) + \Pr(z(\vx)<a, w(\vx)=0) + \Pr(z(\vx)>b, w(\vx)=0) =&\; \Pr(w(\vx)=0), \\
    \Pr(z(\vx)<a, w(\vx)=0) + \Pr(z(\vx) \geq a, w(\vx)=0) 
    =& \; \Pr(w(\vx)=0), \\
    \Pr(z(\vx)>b, w(\vx)=0) + \Pr(z(\vx) \leq b, w(\vx)=0) =& \; \Pr(w(\vx)=0).
\end{align}
From these equations, we have
\begin{equation}
    \Pr(a \leq z(\vx) \leq b, w(\vx)=0) = \Pr(z(\vx) \geq a, w(\vx)=0) + \Pr(z(\vx) \leq b, w(\vx)=0) - \Pr(w(\vx)=0),
\end{equation}
and $k=\argmax_{i\in \{0,1\}}\{p^i(\vx)\}$ leads us to
\begin{align}
    \Pr(|z(\vx) -\mathbb{E}[z(\vx)]|\leq \gamma(\vx), w(\vx)=0) =& \left\{
    \begin{array}{ll}
    p^{10}(\vx) + p^{00}(\vx) - q^0(\vx) & 0 < \gamma(\vx) < p^{\rm min}(\vx),\\
    p^{k0}(\vx)+q^0(\vx)-q^0(\vx) & p^{\rm min}(\vx) \leq \gamma(\vx) < p^{\rm max}(\vx),\\
    q^0(\vx)+q^0(\vx)-q^0(\vx) & \gamma(\vx) \geq p^{\rm max}(\vx) \\
    \end{array}
    \right. \notag \\
    =& \left\{
    \begin{array}{ll}
    0 & 0 < \gamma(\vx) < p^{\rm min}(\vx),\\
    p^{k0}(\vx) & p^{\rm min}(\vx) \leq \gamma(\vx) < p^{\rm max}(\vx),\\
    q^0(\vx) & \gamma(\vx) \geq p^{\rm max}(\vx). \\
    \end{array}
    \right. \label{eq:concent_ineq_0}
\end{align}
With the similar argument for $\Pr(|z(\vx) -\mathbb{E}[z(\vx)]|\leq \gamma(\vx), w(\vx)=1)$, we have
\begin{equation}
    \Pr(|z(\vx) -\mathbb{E}[z(\vx)]|\leq \gamma(\vx), w(\vx)=1) = \left\{
    \begin{array}{ll}
    0 & 0 < \gamma < p^{\rm min}(\vx),\\
    p^{k1}(\vx) & p^{\rm min}(\vx) \leq \gamma < p^{\rm max}(\vx),\\
    q^1(\vx) & \gamma(\vx) \geq p^{\rm max}(\vx). \\
    \end{array}
    \right. \label{eq:concent_ineq_1}
\end{equation}
Then, from Eqs.~\eqref{eq:concent_ineq_0}and~\eqref{eq:concent_ineq_1}, defining
\begin{equation}
    \delta(\vx, \gamma(\vx), \eta(\vx)) = \left\{
    \begin{array}{ll}
    1 & (0 < \gamma(\vx) < p^{\rm min}(\vx)) \land (\eta(\vx) \leq 0),\\
    1 & (0 < \gamma(\vx) < p^{\rm min}(\vx)) \land (0 < \eta(\vx) \leq 1),\\
    1 & (0 < \gamma(\vx) < p^{\rm min}(\vx)) \land (\eta_m > 1),\\
    p^{\rm min}(\vx) & (p^{\rm min}(\vx) \leq \gamma(\vx) < p^{\rm max}(\vx)) \land (\eta(\vx) \leq 0),\\
    1-p^{k1}(\vx) & (p^{\rm min}(\vx) \leq \gamma(\vx) < p^{\rm max}(\vx)) \land (0 < \eta(\vx) \leq 1),\\
    1 & (p^{\rm min}(\vx) \leq \epsilon < p^{\rm max}(\vx)) \land (\eta(\vx) >1),\\
    0 & (\gamma(\vx) \geq p^{\rm max}(\vx)) \land (\eta(\vx) \leq 0), \\
    q^0(\vx) & (\gamma(\vx) \geq p^{\rm max}(\vx)) \land (0 < \eta(\vx) \leq 1), \\
    1 & (\gamma(\vx) \geq p^{\rm max}(\vx)) \land (\eta(\vx) > 1),   
    \end{array}
    \right.
\end{equation}
we have
\begin{equation*}
    \Pr(|z(\vx) -\mathbb{E}[z(\vx)]|\leq \gamma(\vx), w(\vx)\geq\eta(\vx))=1-\delta(\vx, \gamma(\vx), \eta(\vx)).
\end{equation*}
Let $E(\vx)$ be the event $\{(|z(\vx) -\mathbb{E}[z(\vx)]|\leq \gamma(\vx)) \cap (w(\vx)\geq\eta(\vx))\}$, and its complement is denoted by $\overline{E(\vx)}$. Since $\Pr(E(\vx))=1-\Pr(\overline{E(\vx)})$, we have
\begin{equation}
    \Pr(\overline{E(\vx)}) = \delta\left(\vx, \gamma(\vx),\eta(\vx)\right).
\end{equation}

By Boole's inequality, the following inequalities hold:
\begin{equation}
\Pr\left(\cup_{\vx \in \cX}\overline{E(\vx)}\right)\leq \sum_{\vx \in \cX}\delta(\vx,\gamma(\vx), \eta(\vx)).
\end{equation}
By De Morgan's laws, we have
\begin{equation}
    \Pr\left(\cap_{\vx\in \cX}E(\vx)\right) \geq 1 - \sum_{\vx \in \cX}\delta(\vx,\gamma(\vx), \eta(\vx)).
\end{equation}
Therefore, the following inequality is established:
\begin{equation*}
    \Pr\left(\forall \vx \in \cX, \left|z(\vx) - \bE\left[z(\vx)\right]\right|\leq \gamma(\vx), w(\vx)\geq\eta(\vx)\right) \geq 1-\sum^M_{m=1}\delta(\vx,\gamma(\vx), \eta(\vx)).
\end{equation*}
Here, let $\gamma(\vx) = \max\{p^{\rm min}(\vx), q^1(\vx)\}$ and $\eta(\vx) = \one_{\mathbb{R}^+}(q^1(\vx) - p^{\rm max}(\vx))$, then $\delta(\vx, \gamma(\vx), \eta(\vx)) = r^{\rm min}(\vx)$. Thus, the lemma is proved.

\end{proof}

\begin{lemma}
\label{lem:specific_events_in_each_case}
Let $r^{\rm max}(\vx) = \max\{\Pr(\vx \in H_\theta), \Pr(\vx \in L_\theta), \Pr(\vx \in U_\theta)\}$. For any event $E(\vx):=\{(\left|z(\vx) - \bE\left[z(\vx)\right]\right|\leq \gamma(\vx))\cap (w(\vx)\geq\eta(\vx))\}$, the following relationships hold:
\begin{enumerate}
    \item If $r^{\rm max}(\vx)=\Pr(\vx \in H_\theta)$, then $E(\vx)=(f(\vx)-\theta > 0)$.
    \item If $r^{\rm max}(\vx)=\Pr(\vx \in L_\theta)$, then $E(\vx)=(f(\vx)-\theta\leq 0)$.
    \item If $r^{\rm max}(\vx)=\Pr(\vx \in U_\theta)$, then $E(\vx)=\left(-\frac{\epsilon}{2}<f(\vx) - \theta \leq \frac{\epsilon}{2}\right)$.
\end{enumerate}
\end{lemma}
\begin{proof}
We first prove that $E(\vx)=(f(\vx)-\theta > 0)$ holds when $r^{\rm max}(\vx)=\Pr(\vx \in H_\theta)$. Since $r^{\rm max}(\vx)=\Pr(\vx \in H_\theta)$, we can derive  following inequalities:
\begin{align*}
p^{\rm max}(\vx)=& \Pr(\vx \in H_\theta), \\
p^{\rm min}(\vx)=& \Pr(\vx \in L_\theta),\\
\gamma(\vx)=&\max\{p^{\rm min}(\vx), \Pr(\vx\in U_\theta)\} \\
=&\max\{\Pr(\vx\in L_\theta), \Pr(\vx\in U_\theta)\} (< p^{\rm max}(\vx)),\\
\eta(\vx)=&\mathbbm{1}_{\mathbb{R}^+}(\Pr(\vx\in U_\theta) - \Pr(\vx\in H_\theta)) \\
=&0.
\end{align*}
From these, the event $E(\vx)$ can be transformed as follows:
\begin{equation*}
    E(\vx)=\{(\left|z(\vx) - \mathbb{E} \left[z(\vx)\right]\right|\leq  \max\{\Pr(\vx\in L_\theta), \Pr(\vx\in U_\theta)\}) \cap (w(\vx)\geq 0)\}.
\end{equation*}
Since $w(\vx) \in \{0,1\}$, $w(\vx)\geq 0$ is always true. Hence,
\begin{align*}
E(\vx) =&\{\left|z(\vx) - \mathbb{E} \left[z(\vx)\right]\right|\leq  \max\{\Pr(\vx\in L_\theta), \Pr(\vx\in U_\theta)\}\} \\
=&\{\mathbb{E}\left[z(\vx)\right]-\max\{\Pr(\vx\in L_\theta), \Pr(\vx\in U_\theta)\} \leq z(\vx) \leq   \mathbb{E}\left[z(\vx)\right]+\max\{\Pr(\vx\in L_\theta), \Pr(\vx\in U_\theta)\}\}.
\end{align*}
Applying $\mathbb{E} \left[z(\vx)\right]=\Pr(\vx\in H_\theta)$, we have
\begin{align*}
E(\vx) =&\{\Pr(\vx\in H_\theta)-\max\{\Pr(\vx\in L_\theta), \Pr(\vx\in U_\theta)\} \leq z(\vx) \leq   \max\{1, \Pr(\vx\in H_\theta)+\Pr(\vx\in U_\theta)\}\}
\end{align*}
Since $\max\{\Pr(\vx\in L_\theta), \Pr(\vx\in U_\theta)\} < \Pr(\vx\in H_\theta)$, the following inequalities hold.
\begin{align*}
\Pr(\vx\in H_\theta)-\max\{\Pr(\vx\in L_\theta), \Pr(\vx\in U_\theta)\}  > 0,\\
\max\{1, \Pr(\vx\in H_\theta)+\Pr(\vx\in U_\theta)\} \geq 1.
\end{align*}
From the above, the event $E(\vx)$ is always true only when $z(\vx)=1$. Since $z(\vx):=\one_{\mathbb{R}^+}(\hat{f}(\vx)-\theta)$, $z(\vx)=1$ holds when $f(\vx)-\theta>0$. Therefore, when $r^{\rm max}(\vx)=\Pr(\vx \in H_\theta)$, the following relationship holds.
\begin{equation*}
E(\vx)=\{f(\vx)-\theta>0\}
\end{equation*}

Similarly, when $r^{\rm max}(\vx)=\Pr(\vx \in L_\theta)$, we have the following equations.
\begin{align*}
p^{\rm max}(\vx)=& \Pr(\vx \in L_\theta), \\
p^{\rm min}(\vx)=& \Pr(\vx \in H_\theta),\\
\gamma(\vx)=&\max\{p^{\rm min}(\vx),  = \Pr(\vx\in U_\theta)\}, \\
=&\max\{\Pr(\vx\in H_\theta), \Pr(\vx\in U_\theta)\} (< p^{\rm max}(\vx)),\\
\eta(\vx)=&\mathbbm{1}_{\mathbb{R}^+}(\Pr(\vx\in U_\theta) - \Pr(\vx\in L_\theta)) \\
=&0.
\end{align*}
By applying a derivation similar to the previous one, we have
\begin{equation*}
E(\vx)=\{f(\vx)-\theta \leq 0\}.
\end{equation*}

Next, we prove that $E(\vx)=\left(-\frac{\epsilon}{2}<f(\vx) - \theta \leq \frac{\epsilon}{2}\right)$ when $r^{\rm max}(\vx)=\Pr(\vx \in U_\theta)$. From the assumption, the following equations hold.
\begin{align*}
\gamma(\vx)=&\max\{p^{\rm min}(\vx), \Pr(\vx\in U_\theta)\}, \\
=&\Pr(\vx\in U_\theta) (>p^{\rm max}(\vx))\\
\eta(\vx)=&\mathbbm{1}_{\mathbb{R}^+}(\Pr(\vx\in U_\theta) - p^{\rm max}(\vx)) \\
=&1.
\end{align*}
By using $\mathbb{E} \left[z(\vx)\right]=\Pr(\vx\in H_\theta)$, $\left|z(\vx) - \mathbb{E} \left[z(\vx)\right]\right|\leq  \Pr(\vx\in U_\theta)$ is transformed as follows:
\begin{equation*}
\Pr(\vx\in H_\theta)- \Pr(\vx\in U_\theta) \leq z(\vx)\leq  \Pr(\vx\in H_\theta) + \Pr(\vx\in U_\theta).
\end{equation*}
Since $\Pr(\vx\in U_\theta) > p^{\rm max}(\vx) = \max\{\Pr(\vx\in H_\theta), \Pr(\vx\in L_\theta)\}$, the following inequalities hold.
\begin{align*}
\Pr(\vx\in H_\theta) - \Pr(\vx\in U_\theta) <& 0, \\
\Pr(\vx\in H_\theta) + \Pr(\vx\in U_\theta) >& 1.
\end{align*}
From these relationships, $\left|z(\vx) - \mathbb{E} \left[z(\vx)\right]\right|\leq  \Pr(\vx \in U_\theta)$ is always true. Therefore, the event $E(\vx)$ holds when $w(\vx) = 1$.
Since $w(\vx):= \one_{\cE}(\hat{f}(\vx) - \theta)$, we can derive as follows:
\begin{equation*}
E(\vx)=\left\{-\frac{\epsilon}{2} < f(\vx) - \theta \leq \frac{\epsilon}{2}\right\}.
\end{equation*}
Thus, the lemma has been proven.

\end{proof}
\subsection{Proof of theorem~\ref{thm:stopping_criterion}}
Events $f(x)-\theta >0$, $f(x)-\theta < 0$ and $-\epsilon/2 < f(x)-\theta \leq \epsilon/2$ are denoted by $E_{H_\theta}(\vx)$, $E_{L_\theta}(\vx)$, and $E_{U_\theta}(\vx)$, respectively. From the definition of the $\epsilon$-accuracy, the following relationship holds.
\begin{equation}
\Pr((\tilde{H}_{\theta},\tilde{L}_{\theta},\tilde{U}_{\theta}) \text{ is $\epsilon$-accurate}) = \Pr\left(\left(\bigcap_{\vx \in \tilde{H}_\theta} E_{H_\theta}(\vx)\right)\cap\left(\bigcap_{\vx \in \tilde{L}_\theta} E_{L_\theta}(\vx)\right) \cap \left(\bigcap_{\vx \in \tilde{U}_\theta} E_{U_\theta}(\vx)\right)\right).
\label{eq:epsilon_accurate}
\end{equation}

On the other hand, from the lemma~\ref{lem:general_prob_ineq}, the following equation holds.
\begin{equation}
    \label{eq:general_ineq}
    p\left(\forall \vx \in \cX, \left|z(\vx) - \bE\left[z(\vx)\right]\right|\leq \gamma(\vx), w(\vx)\geq\eta(\vx)\right) \geq 1-\sum_{\vx\in \cX}r^{\rm min}(\vx).
\end{equation}
Therefore, we just prove the equivalence between the left side hand of Eq.\eqref{eq:general_ineq} and the right side hand of Eq.~\eqref{eq:epsilon_accurate}. The left side hand of Eq.\ref{eq:general_ineq} is transformed as follows:
\begin{equation*}
\Pr(\forall \vx \in \cX, \left|z(\vx) - \mathbb{E} \left[z(\vx)\right]\right|\leq  \gamma(\vx), w(\vx)\geq \eta(\vx))=\Pr\left(\bigcap_{\vx \in \cX} E(\vx)\right).
\end{equation*}
If we determine $(\tilde{H}_\theta, \tilde{L}_\theta, \tilde{U}_\theta)$ using Eqs.~\eqref{eq:classification_rule_H}, \eqref{eq:classification_rule_L}, \eqref{eq:classification_rule_D}, we can divide the $\cX$ to $(\tilde{H}_\theta, \tilde{L}_\theta, \tilde{U}_\theta)$ as follows:
\begin{equation*}
\Pr\left(\bigcap_{\vx \in \cX} E(\vx)\right) = \Pr\left(\left(\bigcap_{\vx \in \tilde{H}_\theta} E(\vx)\right)\cap\left(\bigcap_{\vx \in \tilde{L}_\theta} E(\vx)\right) \cap \left(\bigcap_{\vx \in \tilde{U}_\theta} E(\vx)\right)\right).
\end{equation*}
Since $(\tilde{H}_\theta, \tilde{L}_\theta, \tilde{U}_\theta)$ is determined by Eqs.~\eqref{eq:classification_rule_H}, \eqref{eq:classification_rule_L}, and \eqref{eq:classification_rule_D}, we can apply lemma~\ref{lem:specific_events_in_each_case} to the above equation. Therefore, the following equation holds.
\begin{align*}
&\Pr\left(\left(\bigcap_{\vx \in \tilde{H}_\theta} E(\vx)\right)\cap\left(\bigcap_{\vx \in \tilde{L}_\theta} E(\vx)\right) \cap \left(\bigcap_{\vx \in \tilde{U}_\theta} E(\vx)\right)\right) \\
=& \Pr\left(\left(\bigcap_{\vx \in \tilde{H}_\theta} E_{H_\theta}(\vx)\right)\cap\left(\bigcap_{\vx \in \tilde{L}_\theta} E_{L_\theta}(\vx)\right) \cap \left(\bigcap_{\vx \in \tilde{U}_\theta} E_{U_\theta}(\vx)\right)\right) \\
=&\Pr\left((\tilde{H}_{\theta},\tilde{L}_{\theta},\tilde{U}_{\theta}) \text{ is $\epsilon$-accurate}\right).
\end{align*}
From the above, we proved that the theorem~\ref{thm:stopping_criterion}.

\subsection{Proof of proposition~\ref{prop:lower_bound_performance_measure}}
In this subsection, we show the lower bound of performance measures such as accuracy, recall, precision, specificity, and F-score. For the performance measures, the following relationships hold.
\begin{proposition}
If we assume that $\tilde{H}_{\theta}$,$\tilde{L}_{\theta}$ and $\tilde{U}_{\theta}$ are determined by using the classification rule of Eqs.~\eqref{eq:classification_rule_H},\eqref{eq:classification_rule_L} and \eqref{eq:classification_rule_D}, then the following inequalities hold with probability $1-\sum_{\vx \in \cX}r^{\rm min}(\vx)$.
\begin{align*}
\text{Accuracy} \geq& \frac{|\tilde{H}_\theta|+|\tilde{L}_\theta|}{|\tilde{H}_\theta| + |\tilde{L}_\theta| + |\tilde{U}_\theta|}, \\
\text{Precision} \geq& \frac{|\tilde{H}_\theta|}{|\tilde{H}_\theta| + |\tilde{U}_\theta|}, \\
\text{Recall} \geq& \frac{|\tilde{H}_\theta|}{|\tilde{H}_\theta| + |\tilde{U}_\theta|}, \\
\text{Specificity} \geq& \frac{|\tilde{L}_\theta|}{|\tilde{L}_\theta|+|\tilde{U}_\theta|}, \\
\text{F-score} \geq& \frac{2 |\tilde{H}_\theta|}{2|\tilde{H}_\theta| + |\tilde{U}_\theta|}.
\end{align*}
\end{proposition}
\begin{proof}
From theorem~\ref{thm:stopping_criterion}, the following equation holds with probability $1-\sum_{\vx \in \cX}r^{\rm min}(\vx)$,
\begin{equation*}
    \Pr\left((\forall x \in \tilde{H}_\theta, f(x)-\theta >0) \cap (\forall x \in \tilde{L}_\theta, f(x)-\theta < 0) \cap (\forall x \in \tilde{U}_\theta, -\epsilon/2<f(x)-\theta \leq \epsilon/2)\right).
\end{equation*}
We assume that candidate points are considered positive if it is included in the upper-level set and negative if it is included in the lower-level set. Then, the event $(\forall x \in \tilde{H}_\theta, f(x)-\theta >0)$ means that all candidate points predicted as positive are truly positive. Therefore, $|\tilde{H}_\theta|$ is counted as true positive (TP). Similarly, the event  $(\forall x \in \tilde{L}_\theta, f(x)-\theta < 0)$ means that all candidate points predicted as negative are truly negative. Therefore, $|\tilde{L}_\theta|$ is counted as true negative (TN). It is unclear whether the elements of $\tilde{U}_\theta$ are counted as TP, TN, False Positive (FP), or False Negative (FN). In the worst-case scenario, all elements of $|\tilde{U}_\theta|$ are included in either FP or FN. Then, $|\tilde{U}_\theta|$ is counted as $FP + FN$. Therefore, we can derive the following lower bound.
\begin{equation*}
\text{F-score} = \frac{2 \times {\rm Precision} \times {\rm Recall}}{{\rm Precision} + {\rm Recall}} = \frac{2 {\rm TP}}{2{\rm TP} + {\rm FP} + {\rm FN}}
\geq \frac{2|\tilde{H}_\theta|}{2|\tilde{H}_\theta|+|\tilde{U}_\theta|}.
\end{equation*}
Since $(\tilde{H}_\theta, \tilde{L}_\theta, \tilde{U}_\theta)$ is $\epsilon$-accurate with probability $1-\sum_{\vx\in \cX}r^{\rm min}(\vx)$, the above lower bound also holds with probability $1-\sum_{\vx\in \cX}r^{\rm min}(\vx)$.

Similarly, for Accuracy, the following inequalities hold with $1-\sum_{\vx\in \cX}r^{\rm min}(\vx)$.
\begin{align*}
\text{Accuracy} =& \frac{{\rm TP} + {\rm TN}}{{\rm TP} + {\rm TN} + {\rm FP} + {\rm FN}}
\geq \frac{|\tilde{H}_\theta|+|\tilde{L}_\theta|}{|\tilde{H}_\theta|+|\tilde{L}_\theta|+|\tilde{U}_\theta|}, \\
\end{align*}

For the recall, we assume that all elements of $\tilde{U}_\theta$ are included in FN as the worst-case scenario. Then, the following lower bound holds with probability $1-\sum_{\vx\in \cX}r^{\rm min}(\vx)$,
\begin{equation*}
\text{Recall} = \frac{{\rm TP}}{{\rm TP} + {\rm FN}}
\geq \frac{|\tilde{H}_\theta|}{|\tilde{H}_\theta|+|\tilde{U}_\theta|}.
\end{equation*}
Similarly, assuming that all elements of $\tilde{U}_\theta$ are included in FP as the worst-case scenario for the precision and specificity, the following inequality holds with probability $1-\sum_{\vx\in \cX}r^{\rm min}(\vx)$,
\begin{align*}
\text{Precision} =& \frac{{\rm TP}}{{\rm TP} + {\rm FP}}
\geq \frac{|\tilde{H}_\theta|}{|\tilde{H}_\theta|+|\tilde{U}_\theta|}, \\
\text{Specificity} =& \frac{{\rm TN}}{{\rm FP} + {\rm TN}}
\geq \frac{|\tilde{L}_\theta|}{|\tilde{L}_\theta|+|\tilde{U}_\theta|}.
\end{align*}

Therefore, the proposition is proved.
\end{proof}

\subsection{Proof of corollary~\ref{col:prob_ineq}}

In the standard classification rule, $\vx$ is classified into $\tilde{H}_\theta$ when $\mu_N(\vx) - \beta \sigma_N(\vx) > \theta$, $\vx$ is classified into $\tilde{L}_\theta$ when $\mu_N(\vx) + \beta \sigma_N(\vx) < \theta$, and $\vx$ is classified into $\tilde{U}_\theta$ when the both of conditions are not satisfied. Regarding $\Phi(\beta)$ as $\Pr(\vx \in U_\theta)$, the standard classification rule is equivalent to the classification rule of Eqs.~\eqref{eq:classification_rule_H}, \eqref{eq:classification_rule_L}, and \eqref{eq:classification_rule_D} under the assumption that $\Phi(\beta)>0.5$. Let $\tilde{r}^{\rm min}(\vx)=\min\{\Pr(\vx \in H_{\theta}), \Pr(\vx \in L_{\theta}), 1-\Phi(\beta)\}$. Then, $\beta$ can be converted to $\epsilon(\vx) = g\inv(\Phi(\beta \mid \vx)$. In this case, the theorem~\ref{thm:stopping_criterion} is reformulated as follows:
\begin{equation*}
\Pr((\tilde{H}_{\theta},\tilde{L}_{\theta},\tilde{U}_{\theta}) \text{ is $\epsilon(\vx)$-accurate}) \geq  1-\sum_{\vx \in \cX}\tilde{r}^{\rm min}(\vx).
\end{equation*}

Therefore, the corollary~\ref{col:prob_ineq} is proved.

\section{Detailed Algorithm of the Proposed Method}
\subsection{How to determine the margin}
\label{seq:determine_margin}
We have considered $\epsilon$ as a given, but in practice, $\epsilon$ is not necessarily provided and must be determined based on some criteria. However, the appropriate value for $\epsilon$ can change depending on the variance of the prior distribution and the magnitude of the noise, making it challenging to set $\epsilon$ appropriately without prior knowledge of these factors. In the following, we describe a method to adaptively determine $\epsilon$ according to the variance of the prior distribution and the size of the noise, even when no prior knowledge about $\epsilon$ is available. When stopping LSE based on Eq.~\eqref{eq:stopping_criterion}, the average probability $r^{\rm avg}$ that must be ensured per candidate point is $r^{\rm avg} = \frac{1 - \delta}{|\cX|}$. Moreover, from $r^{\rm min}(\vx):=\min\{\Pr(\vx \in H_{\theta}), \Pr(\vx \in L_{\theta}), \Pr( \vx \notin U_{\theta})\}$, since the margin affects only $\Pr( \vx \notin U_{\theta})$, focusing solely on $\Pr( \vx \notin U_{\theta})$, for it to be $\frac{1 - \delta}{|\cX|}$, the probability that the true function is included in the margin, $\Pr( \vx \in U_{\theta})$, must be $1-\frac{1 - \delta}{|\cX|}$. 
Since the most difficult points for classification are those where the surrogate model's mean function values are near the threshold, considering the worst-case scenario where $\theta=\mu(\vx)$, we have
\begin{align}
    \epsilon = 2\sigma_N(\vx)\Phi^{-1}\left(1-\frac{1 - \delta}{2|\cX|}\right).
\end{align}
By setting $\epsilon$ such that this equation holds, $\epsilon$ can be set to make $\Pr( \vx \in U_{\theta})$ larger than $1-\frac{1 - \delta}{|\cX|}$ when $\sigma_N(\vx)$ reaches the desired level.

Since $\sigma_N(\vx)$ is influenced by the variance of the prior distribution and the precision parameter of observation noise $\lambda$, it is difficult to predetermine the desirable $\sigma_N(\vx)$. However, the posterior variance depends solely on the input and is fixed once the input is determined. Thus, the posterior variance of $\vx$ can be considered as an indicator of {\it{how much data have been collected}} at $\vx$. This is an idea similar to the effective sample size (ESS) used in the context of MCMC~\citep{doi:10.1080/10618600.1998.10474787} and survey sampling~\citep{kish1965survey}.
If the user specifies the number of measurements allowed per point, the corresponding posterior variance can be derived. 
This calculation is complex, but the relationship between the data count-like quantity and the posterior variance can be considered by looking at the decrease in the posterior variance $\sigma_L(\vx)$ when observing the same candidate point $L$ times. In this situation, $\sigma_L(\vx)$ is equivalent to the posterior variance of the 1D Gaussian distribution when $N$ samples are observed with the variance of the prior distribution as $k(\vx, \vx)$ and the accuracy parameter of the observation noise as $\lambda$, that is, $\sigma^2_L(\vx)=\frac{\lambda\inv k(\vx, \vx)}{\lambda\inv +L k(\vx, \vx)}$. 
Finally, $\epsilon$ is estimated using the following equation.
\begin{equation*}
    \epsilon(\vx) = 2 \sqrt{\frac{\lambda\inv k(\vx, \vx)}{\lambda\inv +L k(\vx, \vx)}}\Phi^{-1}\left(1-\frac{1-\delta}{2|\cX|}\right).
\end{equation*}
If $k$ is a non-stationary kernel, $\epsilon$ varies depending on the candidate point, whereas for a stationary kernel, it remains the same across all candidate points.
$L$ is a parameter that the user decides according to the problem; if the user wishes to reduce data acquisition costs at the expense of prediction accuracy, $L$ should be set smaller, and if the user wants to ensure prediction accuracy even at higher data acquisition costs, $L$ should be set larger. 
Empirically, for the 2D case, which is common in LSE, we confirmed that $L$ values between 1 and 5 work effectively.

In the Appendix~\ref{app:eps}, we demonstrate that our method for determining $\epsilon$ is less dependent on objective functions and noise variances, compared to directly setting $\epsilon$. Furthermore, we present the differences in LSE efficiency and compare the stopping times as $L$, the parameter used to determine $\epsilon$, is varied.

\subsection{The pseudo code}
\label{seq:pseudocode}
The pseudo code of the proposed method based on the method for determining $\epsilon$ is shown in Algorightm~\ref{alg:proposed_method}.

\begin{algorithm}[t]
\caption{Proposed LSE Algorithm with Stopping Criterion}
\label{alg:proposed_method}
\begin{algorithmic}
\REQUIRE observed data $S_N$, GP prior, a set of candidate points $\cX$, and hyperparameters $\delta>0$ and $L\in \mathbb{N}$ (or $\varepsilon >0$). 
\STATE {\bfseries Initialize:} $\tilde{H}_{\theta}, \tilde{L}_{\theta} \leftarrow \emptyset$, $\tilde{U}_{\theta} \leftarrow \cX$, $\delta \leftarrow 0$
\WHILE{$\delta < \tilde{\delta}$}
\STATE{\bfseries Step~1} Construct the predictive mean $\mu_N(\vx)$ and standard deviation $\sigma_N(\vx)$ defined in~\eqref{eq:posterior_mean} and \eqref{eq:posterior_variance}. 
Define $r^{\rm max}(\vx):=\max\{\Pr(\vx \in H_{\theta}), \Pr(\vx \in L_{\theta}), \Pr(\vx \in U_{\theta})\}$.
\STATE{\bfseries Step~2} Classify for all $\vx \in \cX$ as follows: 
\IF{$r^{\rm max}(\vx)=\Pr(\vx \in H_{\theta})$}
\STATE update the upper-level set: $\tilde{H}_{\theta} \leftarrow \tilde{H}_{\theta} \cup \{\vx \}$
\ELSIF{$r^{\rm max}(\vx)=\Pr(\vx \in L_{\theta})$}
\STATE update the lower-level set: $\tilde{L}_{\theta} \leftarrow \tilde{L}_{\theta} \cup \{\vx \}$
\ELSE
\STATE update the unclassified set $\tilde{U}_{\theta} \leftarrow \tilde{U}_{\theta} \cup \{\vx \}$.
\ENDIF
\STATE{\bfseries Step~3} Calculate $r^{\rm min}(\vx)$ for all $\vx \in \cX$ where
\begin{itemize}
    \item $r^{\rm min}(\vx) = \min\{\Pr (\vx \in H_{\theta}), \Pr (\vx \in L_{\theta}), \Pr (\vx \in U_{\theta}))\}$
    \item $\Pr (\vx \in H_{\theta} )= 1 - \Phi\left(\frac{\theta-\mu_N(\vx)}{\sigma_N(\vx)}\right)$
    \item $\Pr (\vx \in L_{\theta} ) = \Phi\left(\frac{\theta-\mu_N(\vx)}{\sigma_N(\vx)}\right)$
    \item $\Pr (\vx \notin U_{\theta} ) = 1 - \Phi\left(\frac{\theta+\epsilon/2-\mu_N(\vx)}{\sigma_N(\vx)}\right)+\Phi\left(\frac{\theta-\epsilon/2-\mu_N(\vx)}{\sigma_N(\vx)}\right)$
    \item $\epsilon = 2\sqrt{\frac{\lambda^{-1} k(\vx, \vx)}{\lambda^{-1} +L k(\vx, \vx)}}\Phi^{-1}\left(1-\frac{1 - \delta}{2|\cX|}\right)$
\end{itemize}
\STATE{\bfseries Step~4} Select the next evaluation point by maximizing the acquisition function as $\vx_N = {\rm arg}\max_{\vx\in \cX} r^{\rm min}(\vx)$ and observe $y_N = f(\vx_N) + \eta$, $\eta \sim \mathcal{N}(0, \lambda^{-1})$. 
\STATE{\bfseries Step~5} Update the dataset: $S_N \leftarrow S_N \cup \{(\vx_{N}, y_{N})\}$ 
\STATE{\bfseries Step~6} Calculate $\tilde{\delta} \leftarrow 1 - \sum_{\vx\in \cX}r^{\rm min}(\vx)$
\ENDWHILE
\end{algorithmic}
\end{algorithm}

\section{Additional experimental results}
\label{app:exp}

\subsection{The relationship between true probability of Theorem~\ref{thm:stopping_criterion} and its lower bound}
\label{app:thm_and_lb}

\begin{figure}[th!]
    \centering
    \begin{tabular}{ccc}
    \includegraphics[height=4cm,width=4cm]{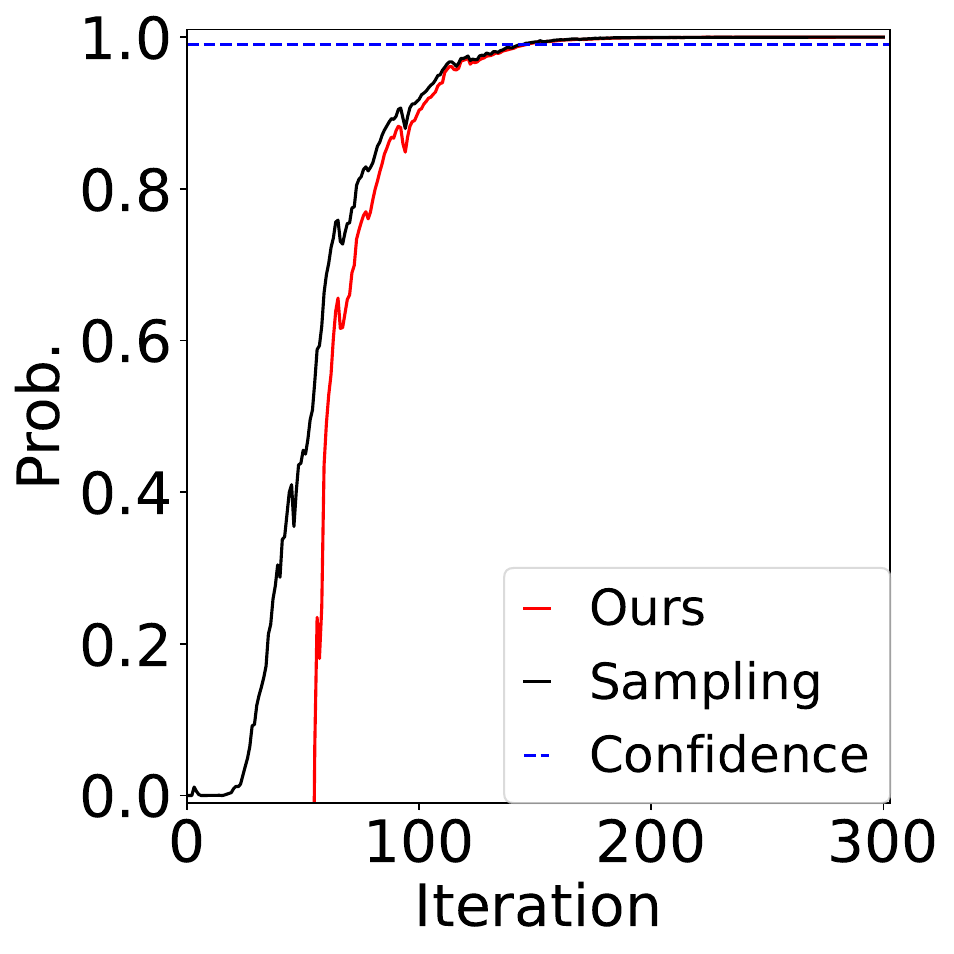}&    
    \includegraphics[height=4cm,width=4cm]{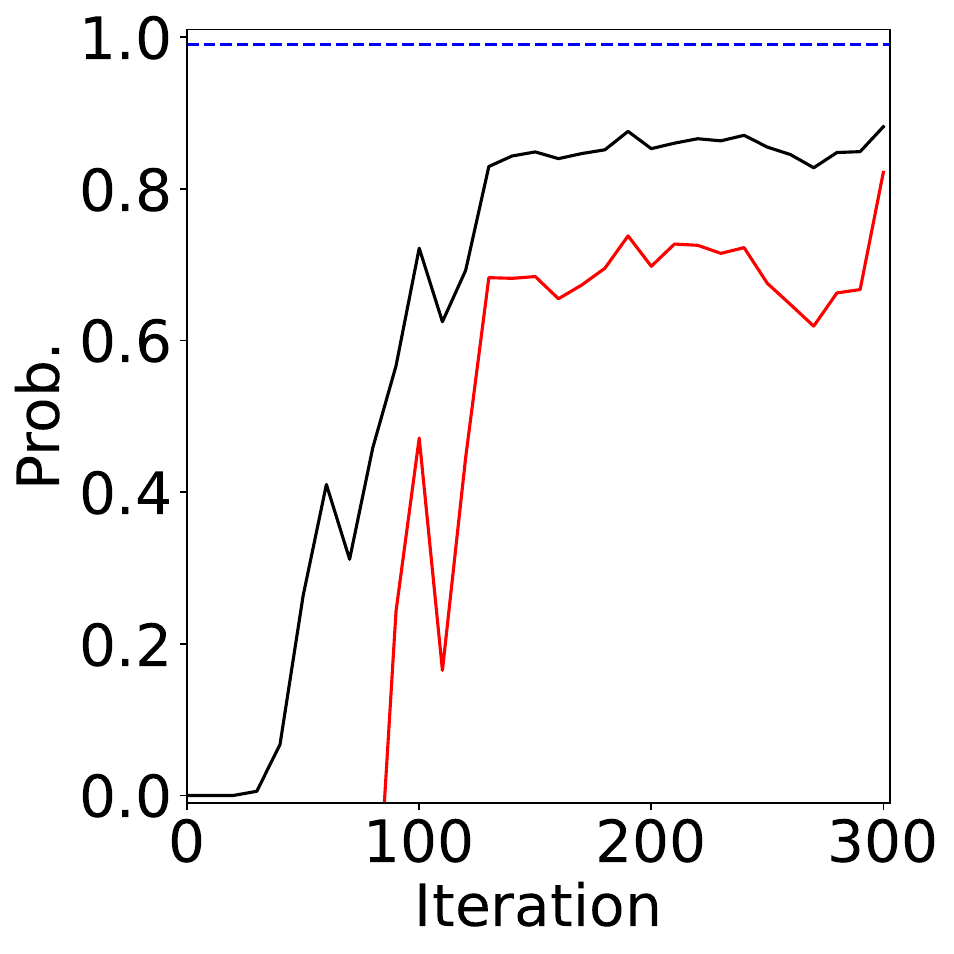}&
    \includegraphics[height=4cm,width=4cm]{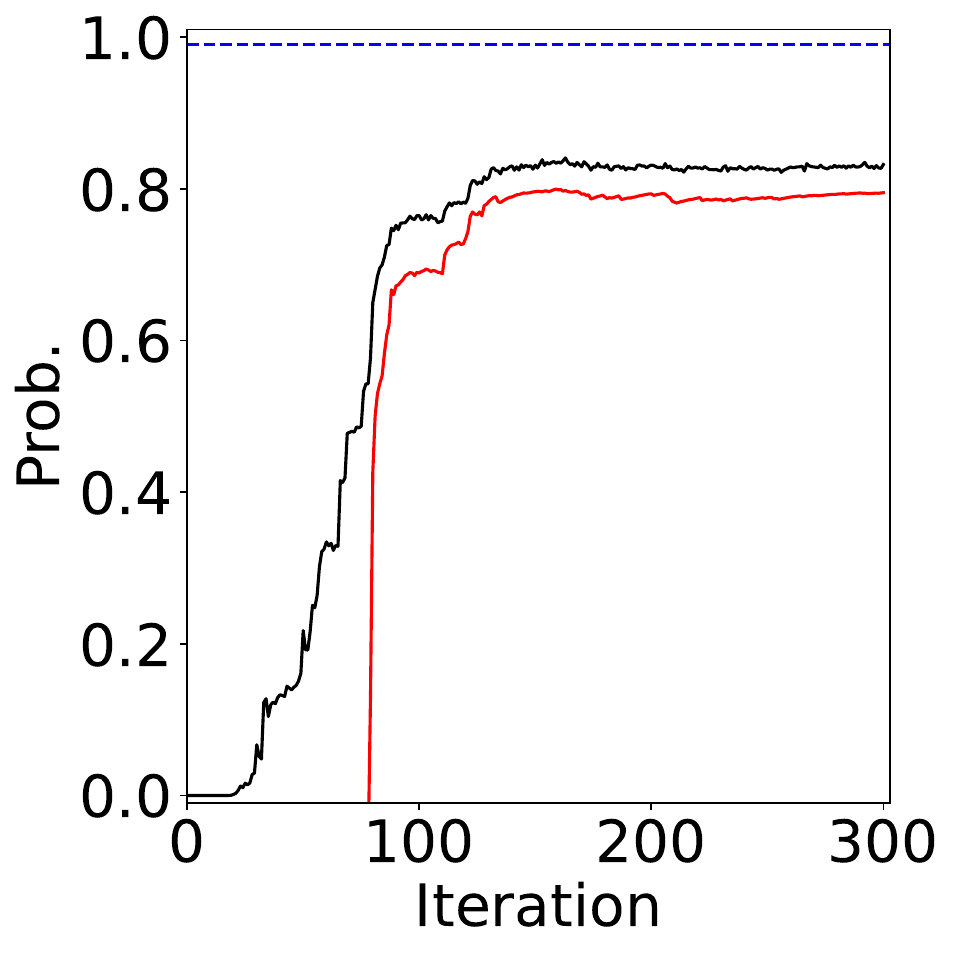}\\
    (a) Ours& (b) MELK& (c) MILE  \\
    \includegraphics[height=4cm,width=4cm]{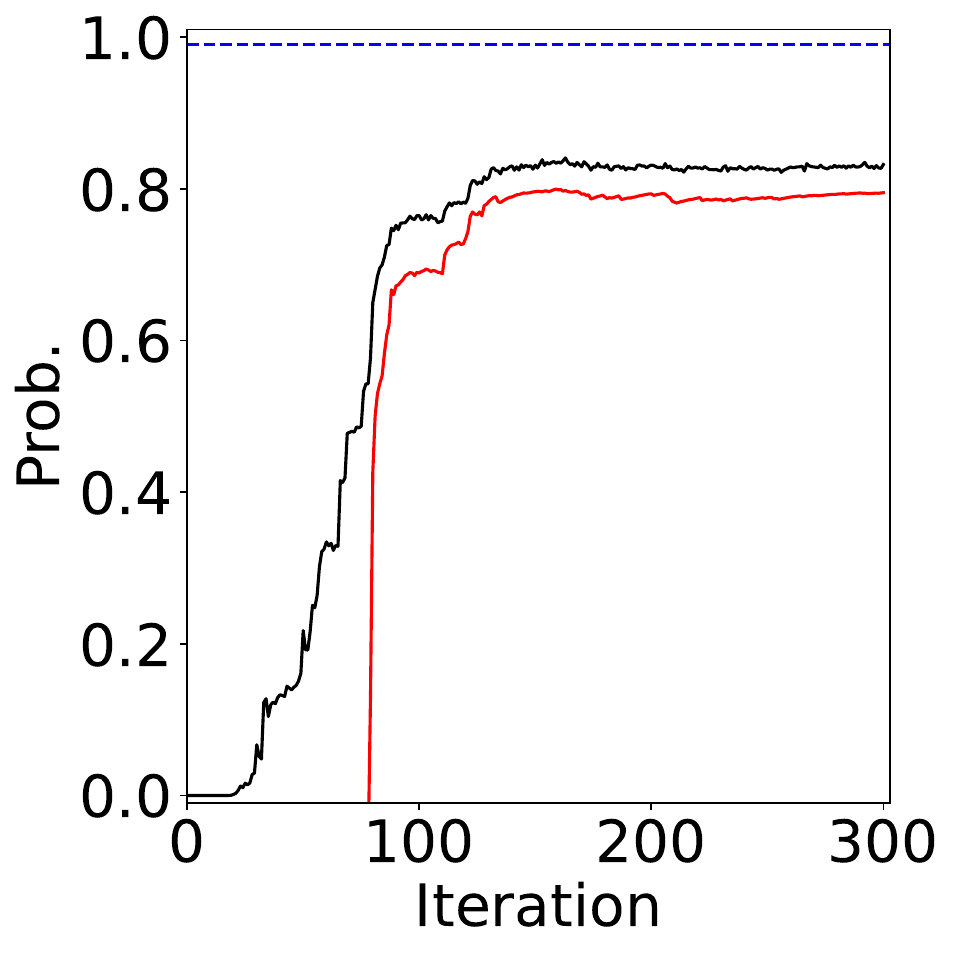}&    
    \includegraphics[height=4cm,width=4cm]{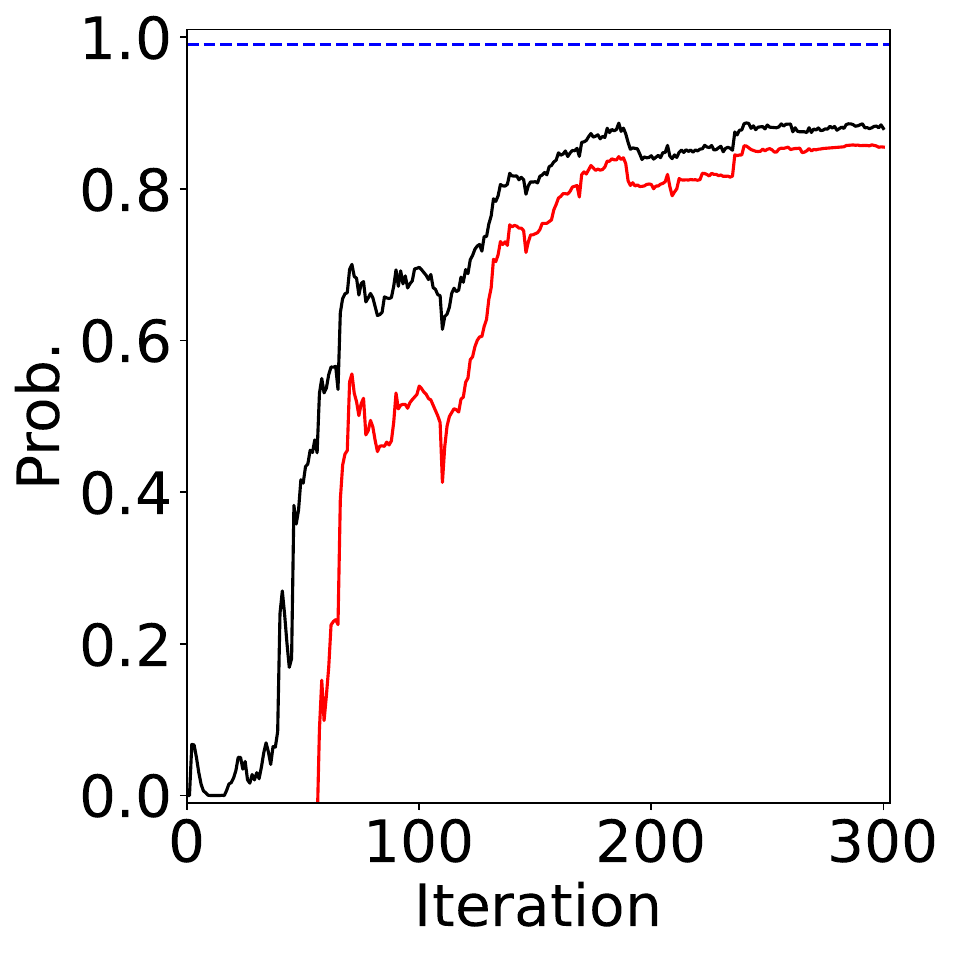}&
    \includegraphics[height=4cm,width=4cm]{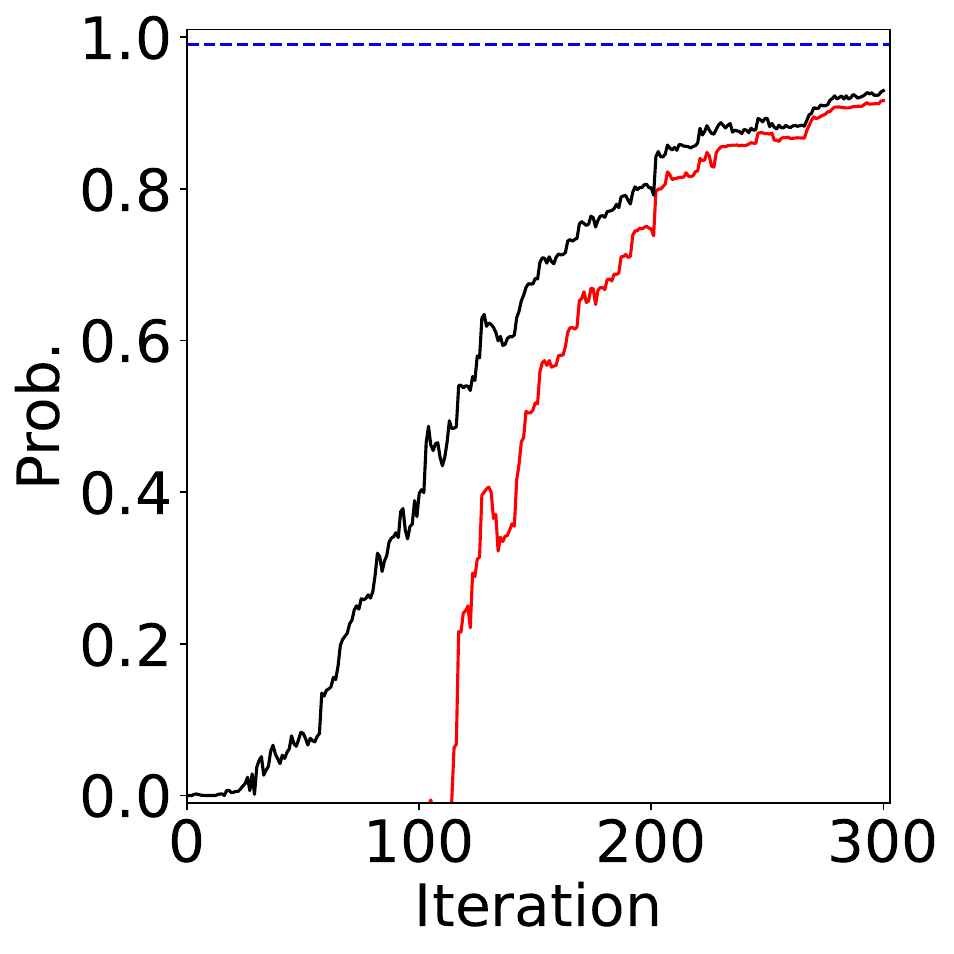}\\
    (d) RMILE& (e) Straddle& (f) US \\
    \end{tabular}
    \caption{
    True probability and its lower bound in {\tt{Branin}} function.}
    \label{fig_prob_ineq_branin}
\end{figure}

\begin{figure}[th!]
    \centering
    \begin{tabular}{ccc}
    \includegraphics[height=4cm,width=4cm]{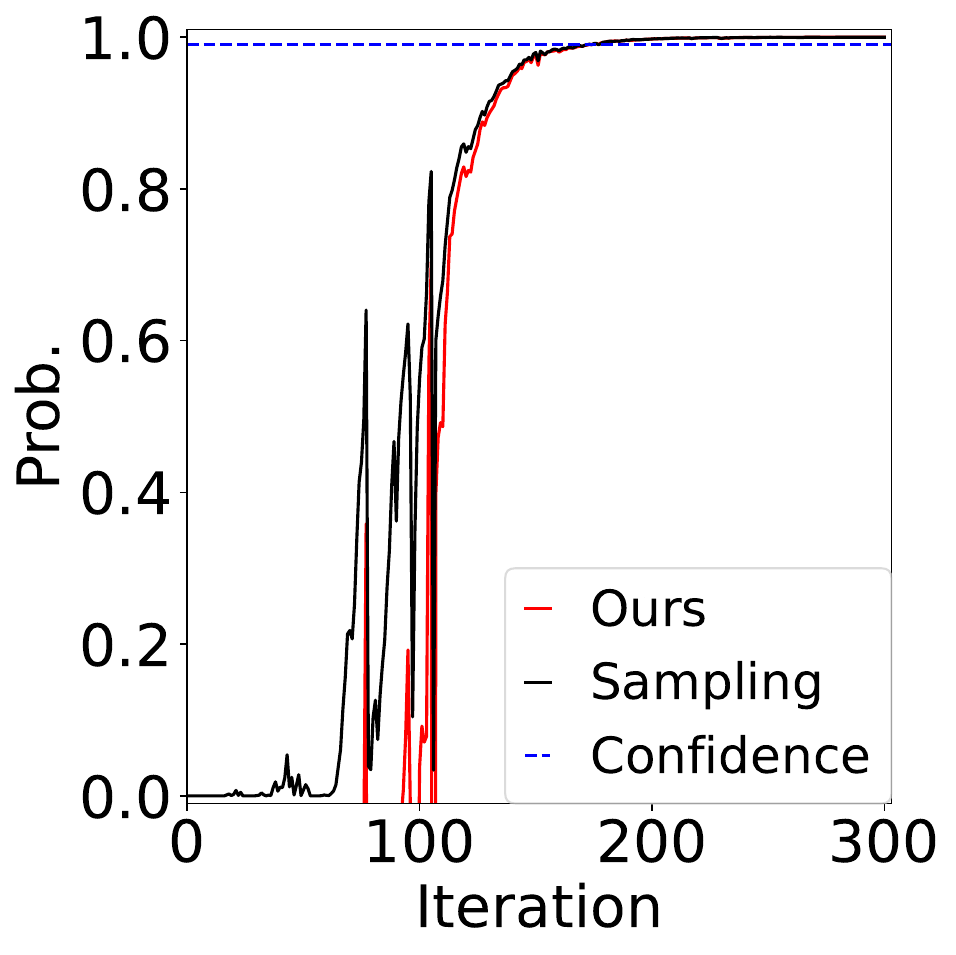}&    
    \includegraphics[height=4cm,width=4cm]{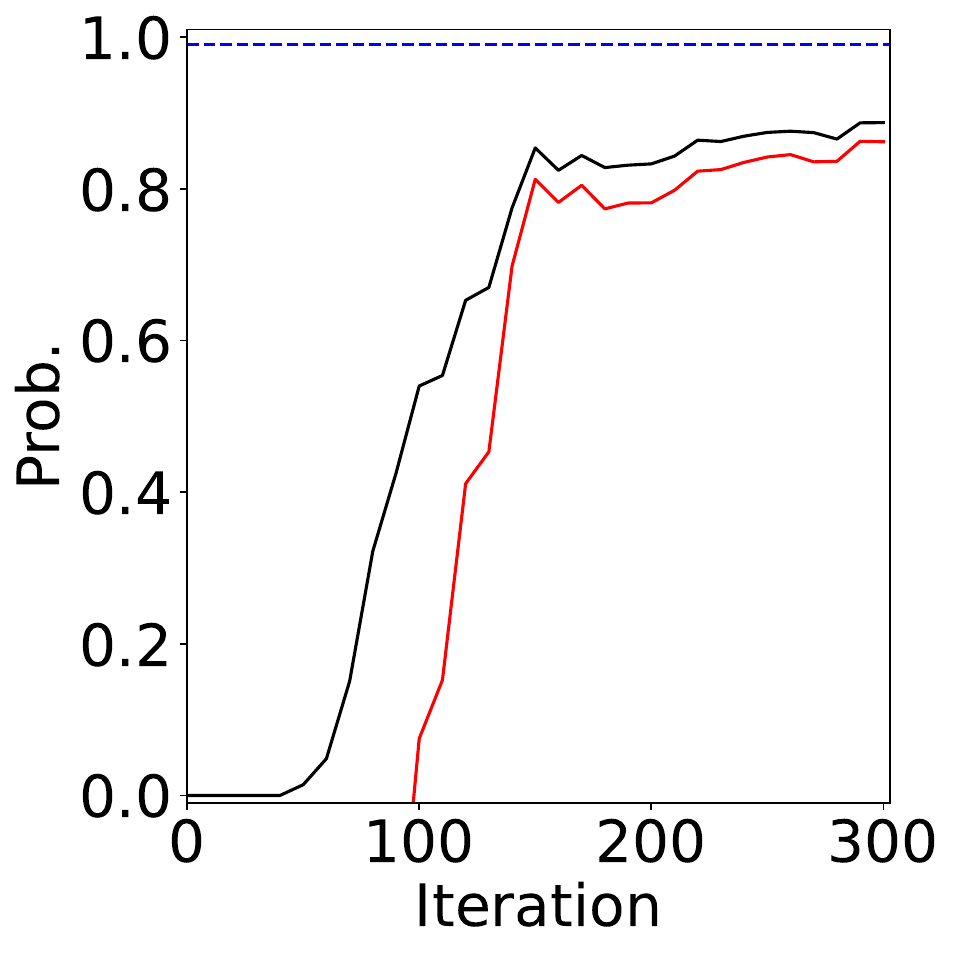}&
    \includegraphics[height=4cm,width=4cm]{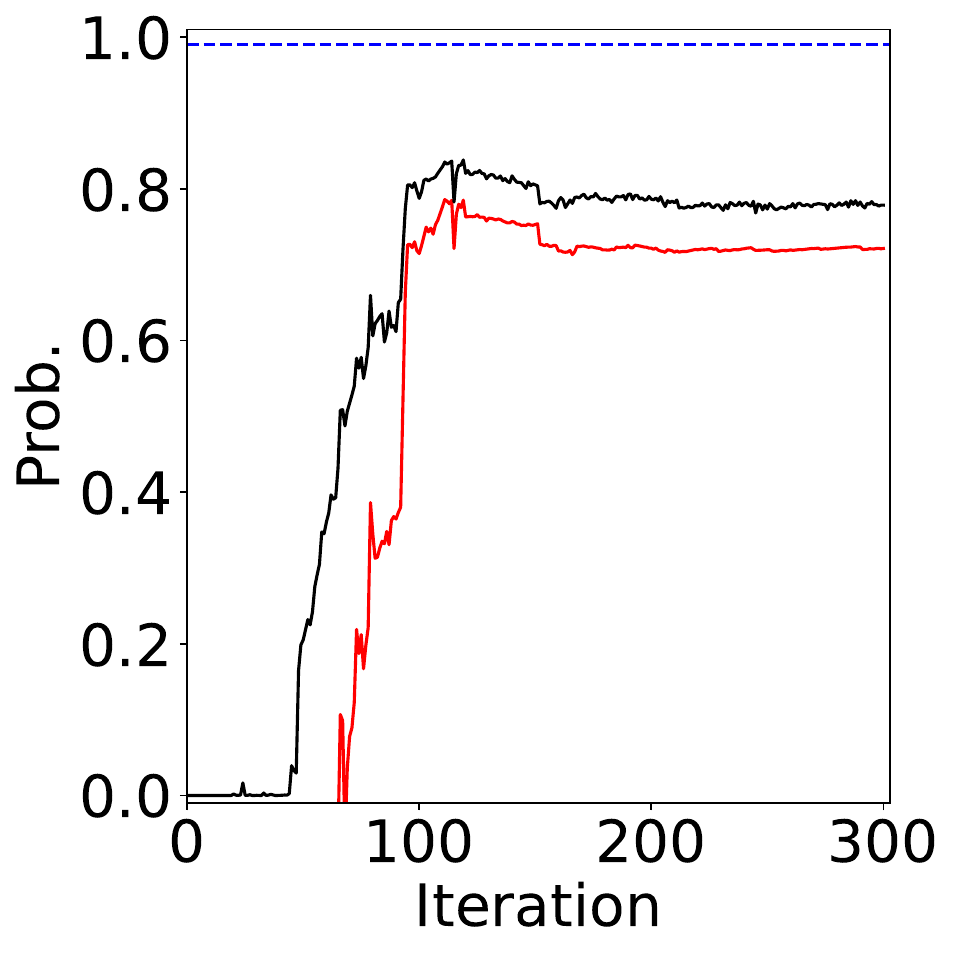}\\
    (a) Ours& (b) MELK& (c) MILE  \\
    \includegraphics[height=4cm,width=4cm]{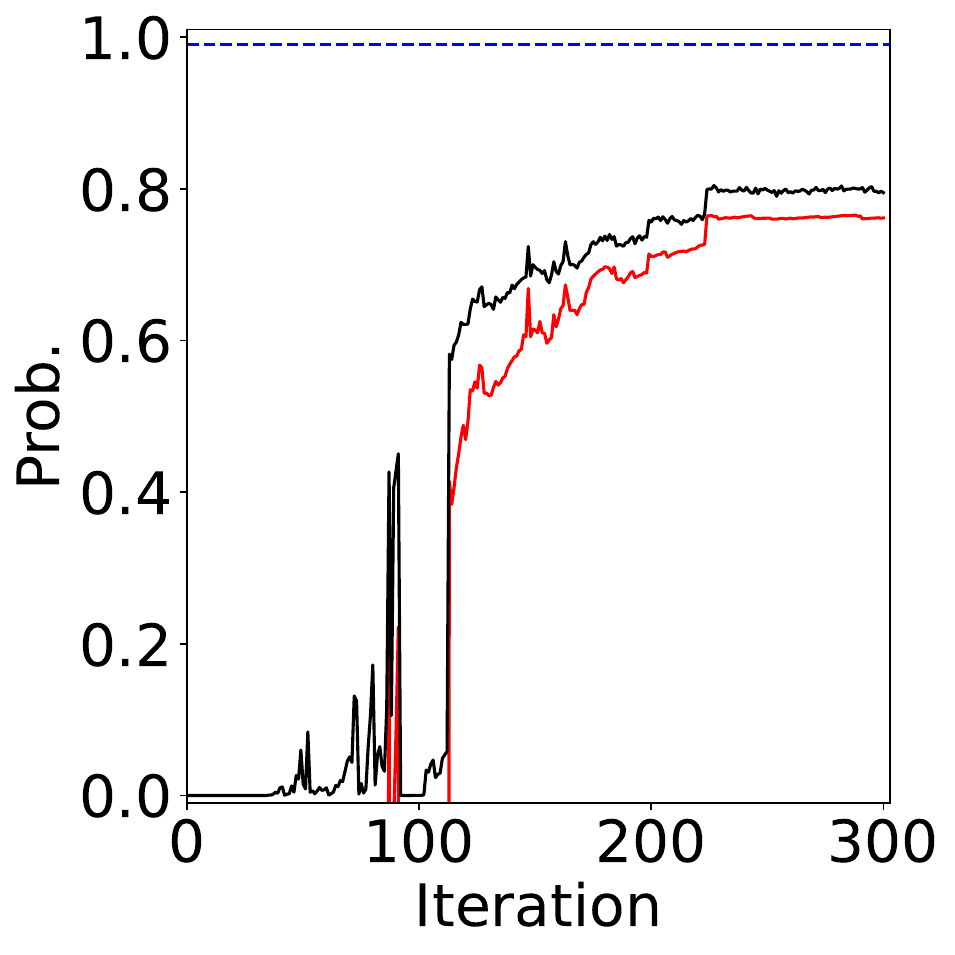}&    
    \includegraphics[height=4cm,width=4cm]{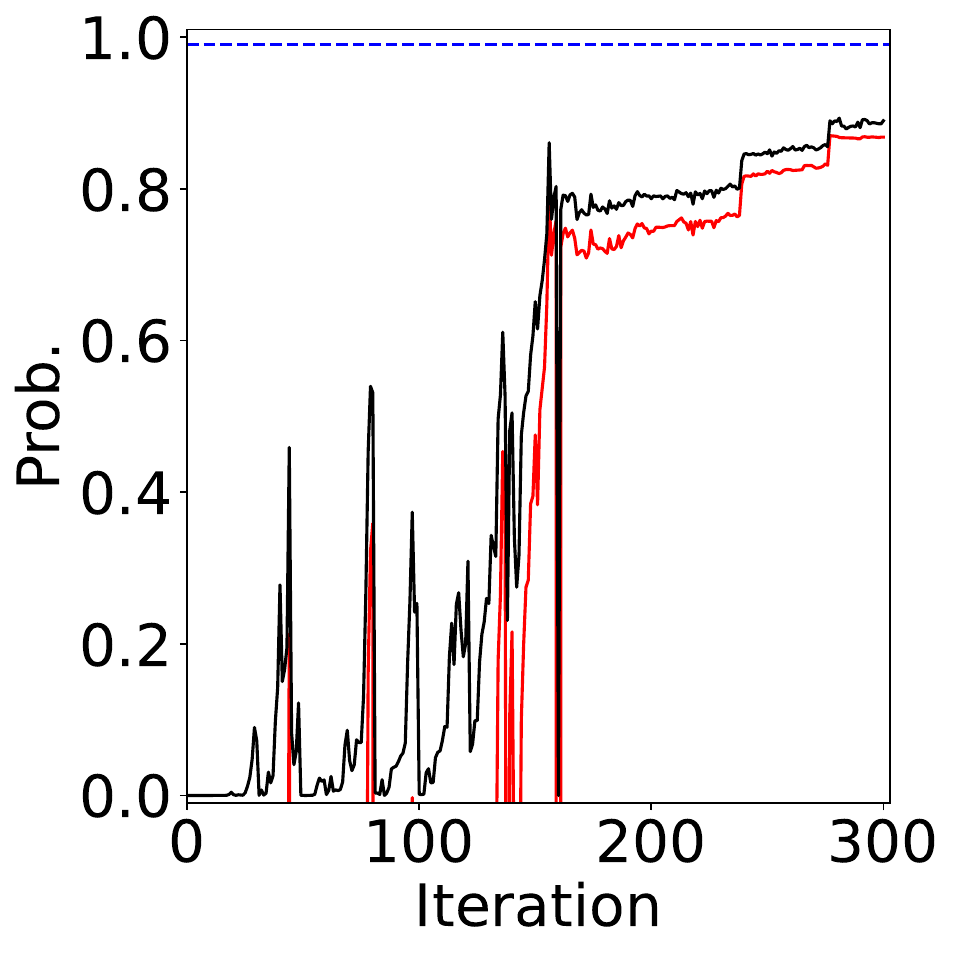}&
    \includegraphics[height=4cm,width=4cm]{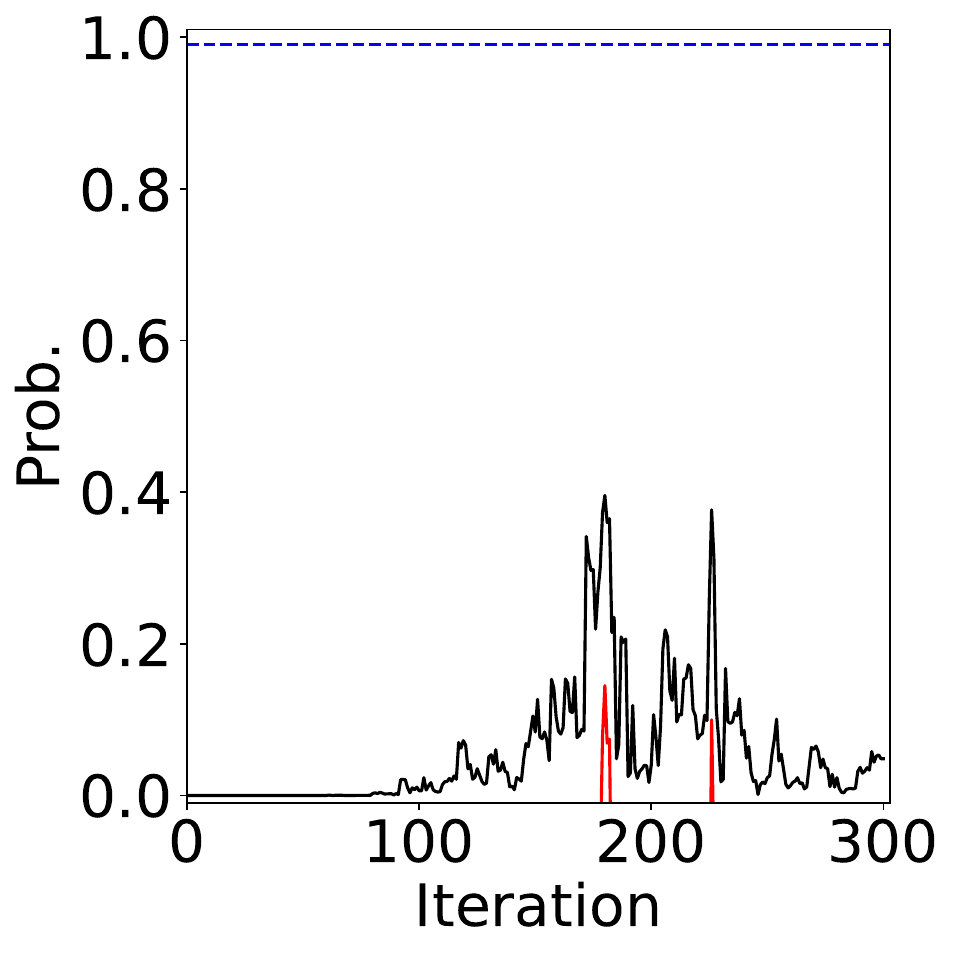}\\
    (d) RMILE& (e) Straddle& (f) US \\
    \end{tabular}
    \caption{
    True probability and its lower bound in {\tt{Rosenbrock}} function.}
    \label{fig_prob_ineq_rosenbrock}
\end{figure}

From Lemma~\ref{lem:general_prob_ineq}, the left-hand side of Eq.~\eqref{eq:ineq_of_sc} can be calculated by evaluating the following equation,
\begin{equation*}
\Pr(\forall \vx \in \mathcal{X}, \left|z(\vx) - \bE \left[z(\vx)\right]\right|\leq  \gamma(\vx), w(\vx)\geq \eta(\vx)).
\end{equation*}
This can be approximated by generating sample paths according to the Gaussian process posterior distribution. Specifically, by considering the generated sample paths as the true function and counting the number of times the equation is satisfied, we can compute this probability. In this subsection, we compare the probability obtained from sampling with its lower bound proposed in this study and discuss the tightness of the proposed lower bound.

In this experiment, we use two test functions, {\tt{Branin}} and {\tt{Rosenbrock}}, and compare the left-hand side and right-hand side of Eq.~\eqref{eq:ineq_of_sc} for each acquisition function. The experimental settings are the same as the experiment in the main text, and we set the number of sample paths generated for evaluating the left-hand side to $10,000$. There are numerical considerations when calculating the left-hand side of Eq.~\eqref{eq:ineq_of_sc}. Since $z(\vx)$ is a binary variable, $\bE[z(\vx)] = \Pr(\vx \in H_\theta)$, meaning that $|z(\vx) - \bE[z(\vx)]|$ can only take $\Pr(\vx \in H_\theta)$ or $\Pr(\vx \in L_\theta)$. Therefore, for $\gamma(\vx) = p^{\rm min}(\vx)$, $|z(\vx) - \bE[z(\vx)]| \leq \gamma(\vx)$ only holds when the equality holds. Although it is correct to use Eq.~\eqref{eq:ineq_of_sc} theoretically, calculating whether equality holds numerically can lead to numerical errors. To avoid these errors, we modify the actual calculation by setting $\gamma(\vx) = \max\{(p^{\rm min}(\vx) + p^{\rm max}(\vx))/2, \Pr(\vx \in U_\theta)\}$.

The experimental results are shown in Figs.~\ref{fig_prob_ineq_branin} and \ref{fig_prob_ineq_rosenbrock}. The results indicate that when the value of the left-hand side is small, the lower bound becomes loose, but as the left-hand side approaches 1, the lower bound becomes tighter. Notably, when the lower bound reaches the confidence parameter $\delta = 0.99$ used in this study, the bound is considerably tight, and there is no difference in the stopping timing whether it is computed approximately using sampling or using the lower bound. This is because the lower bound becomes tighter as the probabilities of each candidate point become sufficiently high, given that no bounds other than the uniform bound are used to calculate the lower bound.

Next, we compare the left-hand side of Eq.~\eqref{eq:ineq_of_sc} when using the proposed acquisition function versus other acquisition functions. In this case, the left-hand side of Eq.~\eqref{eq:ineq_of_sc} converges to 1 for the proposed acquisition function, while it does not for other acquisition functions. Especially, in the case of US for {\tt{Rosenbrock}}, even if LSE progresses, the true probability remains small. Therefore, its lower bound takes a negative value, resulting in a trivial lower bound. This is because other acquisition functions do not search further to improve the classification accuracy of candidate points once it exceeds a certain threshold, whereas the proposed acquisition function continues to search to increase the probability of $\epsilon$-accuracy, even for candidate points that have already achieved a certain level of classification accuracy. Actually, we show that the proposed stopping criterion cannot stop LSE when we use the other acquisition functions in Appendix~\ref{app:otherAF}. Therefore, it can be said that the combination of the proposed acquisition function and stopping criterion realizes efficient LSE. 

\subsection{The effect of the number of candidate points on the stopping timing}
\label{app:effect_n_candidates}

\begin{figure}[th!]
    \centering
    \begin{tabular}{cc}
    \includegraphics[height=4cm,width=4cm]{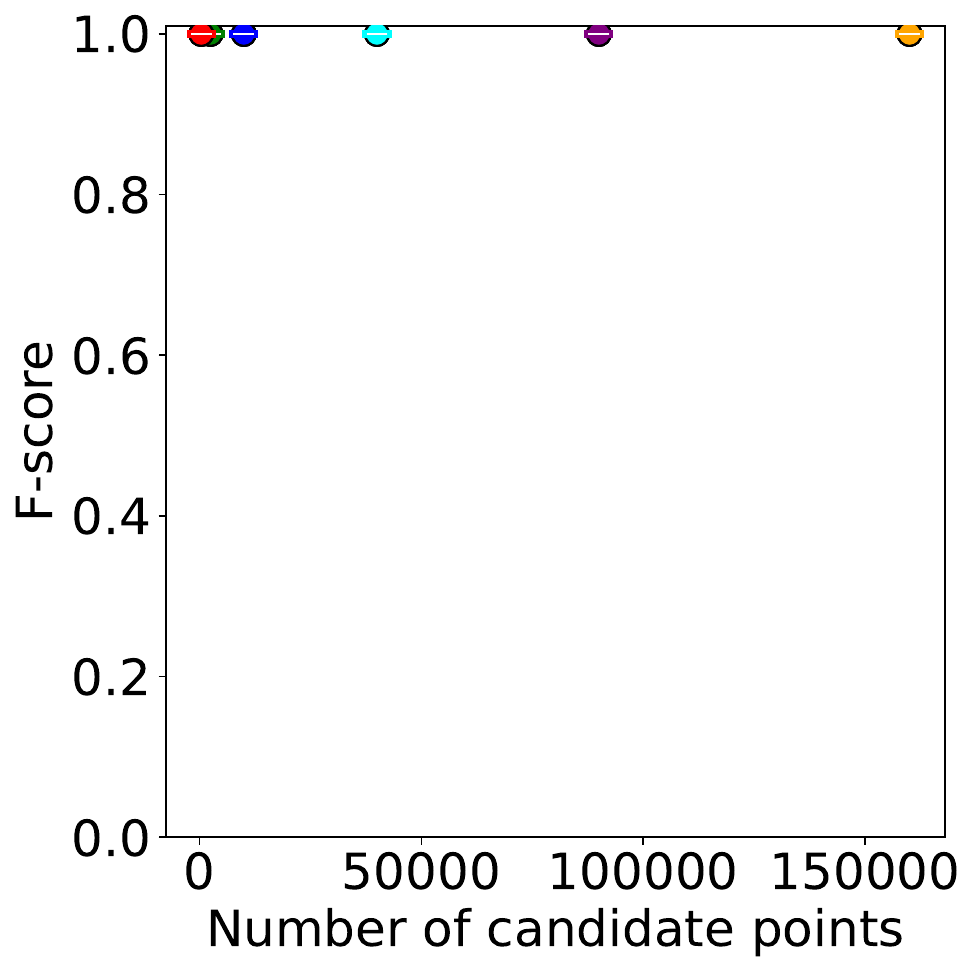}&    
    \includegraphics[height=4cm,width=4cm]{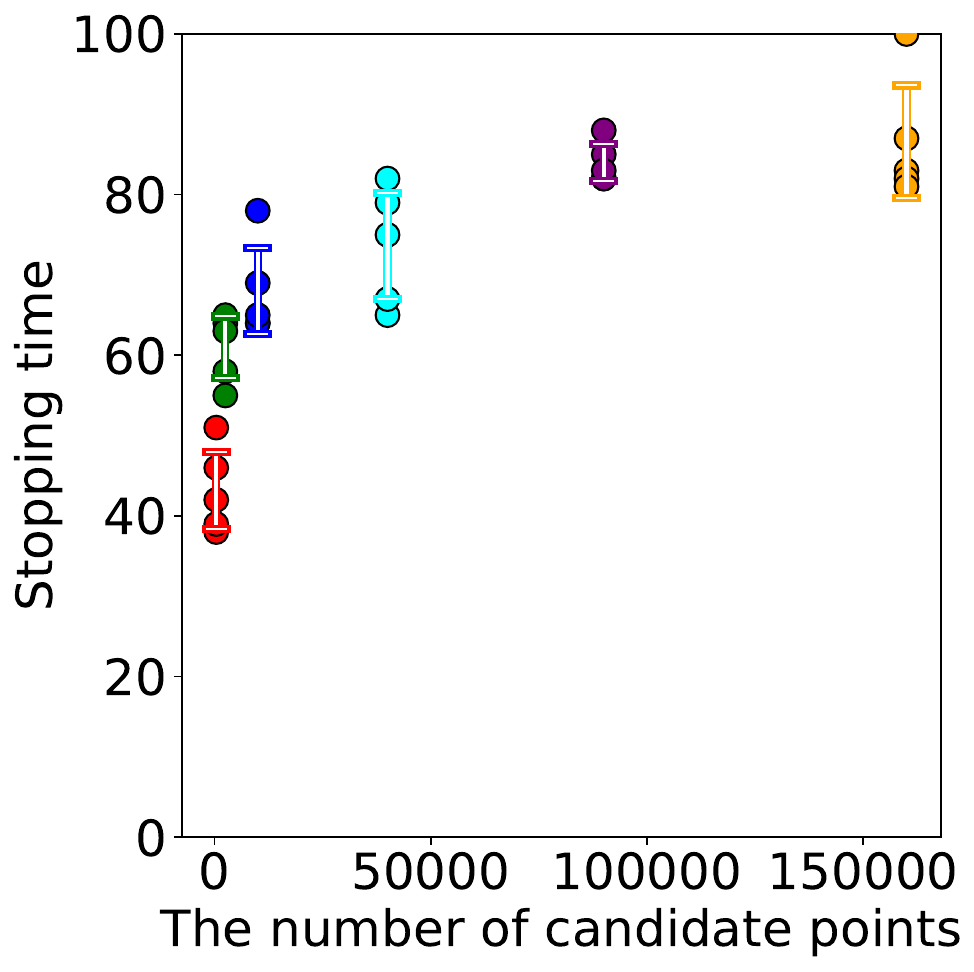} \\
    (a) F1 score($\sigma_{\rm noise}=0$)& (b) Stopping time($\sigma_{\rm noise}=0$)  \\
    \includegraphics[height=4cm,width=4cm]{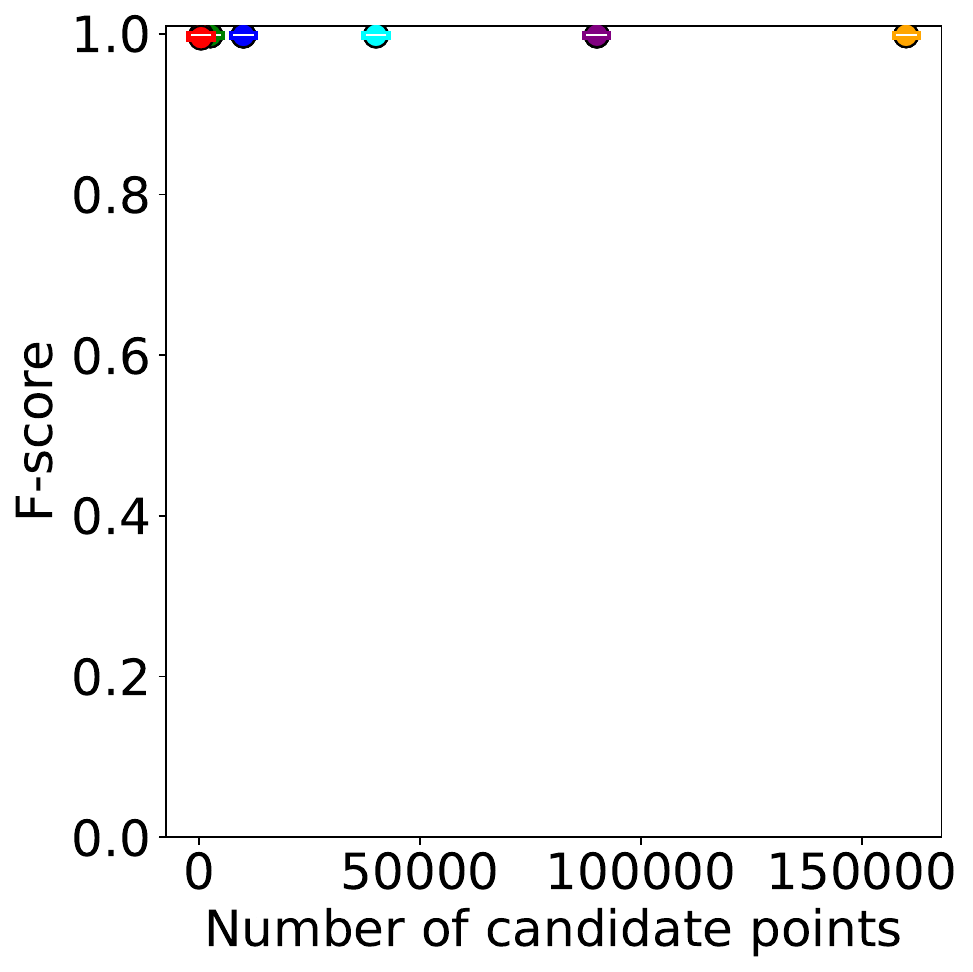}&    
    \includegraphics[height=4cm,width=4cm]{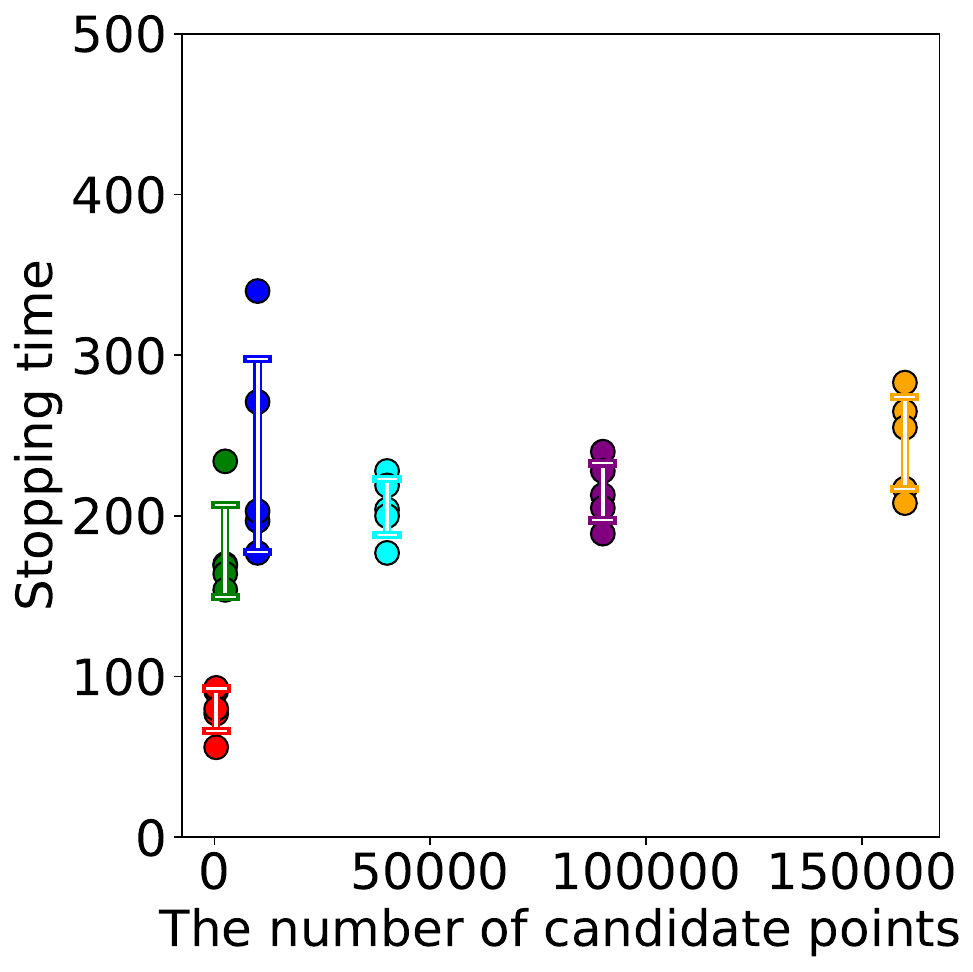}\\
    (c) F1 score($\sigma_{\rm noise}=10$)& (b) Stopping time($\sigma_{\rm noise}=10$) \\
    \end{tabular}
    \caption{
    True probability and its lower bound in {\tt{Rosenbrock}} function.}
    \label{fig_increase_candidate_points}
\end{figure}

Since the proposed stopping criterion uses Boolean inequalities, it is possible that the stopping timing may be delayed as the number of candidate points increases. In this experiment, we evaluate the impact of increasing the number of candidate points on the stopping timing of LSE and demonstrate that the proposed method can still stop even as the number of candidates increases. For this experiment, the {\tt rosenbrock} function is used as the test function, and the LSE threshold is set to $\theta = 100$. Additionally, two cases are evaluated: one without observation noise and the other with Gaussian noise having a variance of $\sigma^2_{\rm noise} = 10^2$.

The experimental results are shown in Fig.~\ref{fig_increase_candidate_points}. From these results, it can be observed that the F-score of the stopping timing shows almost no difference even as the number of candidate points increases, and the increase in stopping timing is slight with respect to the increase in candidate points. This is because the proposed stopping criterion does not rely on any lower bounds other than Boole's inequalities. Specifically, the probability that each individual candidate point is $\epsilon$-accurate can be calculated precisely. When the probability that each candidate point is $\epsilon$-accurate is sufficiently high, the lower bound using Boolean inequalities also becomes tight. Therefore, it can be concluded that the proposed stopping criterion is slight affected by the increase in the number of candidate points.

\subsection{Other test functions}
\label{app:testfunc}

\begin{figure}[th!]
    \centering
    \begin{tabular}{ccc}
    \includegraphics[height=4cm,width=4cm]{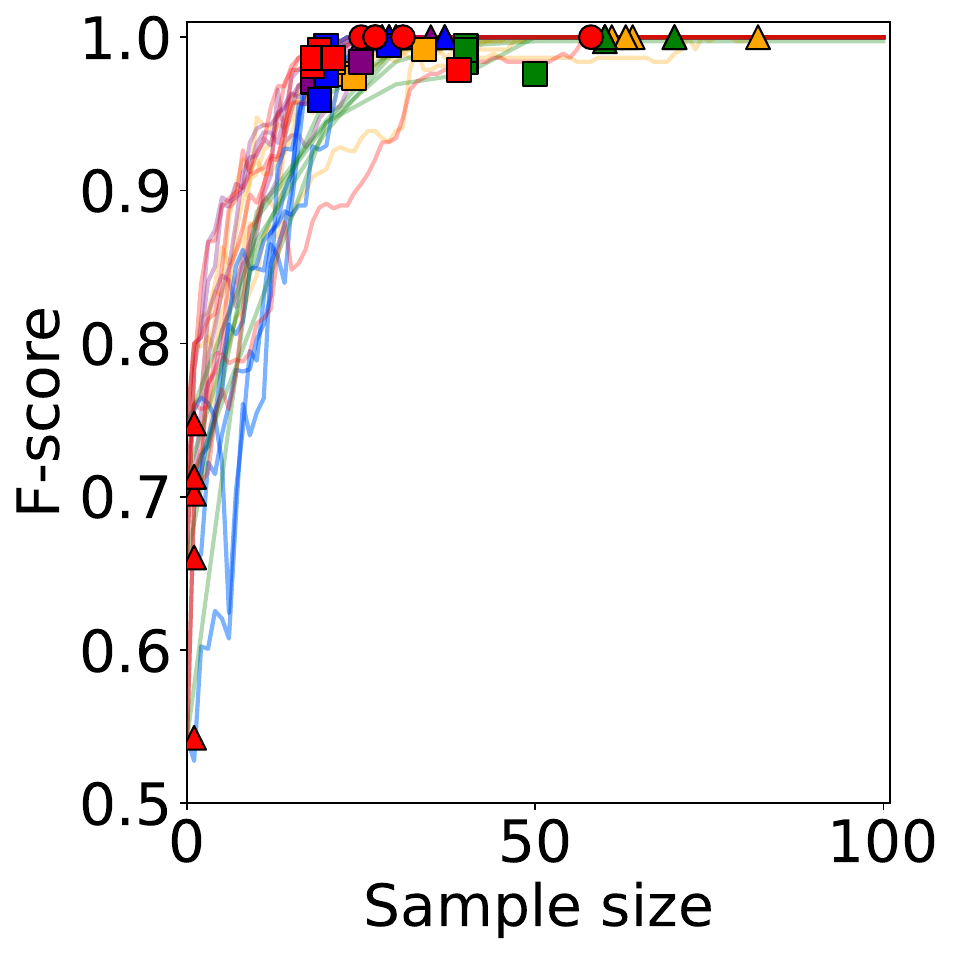}&    
    \includegraphics[height=4cm,width=4cm]{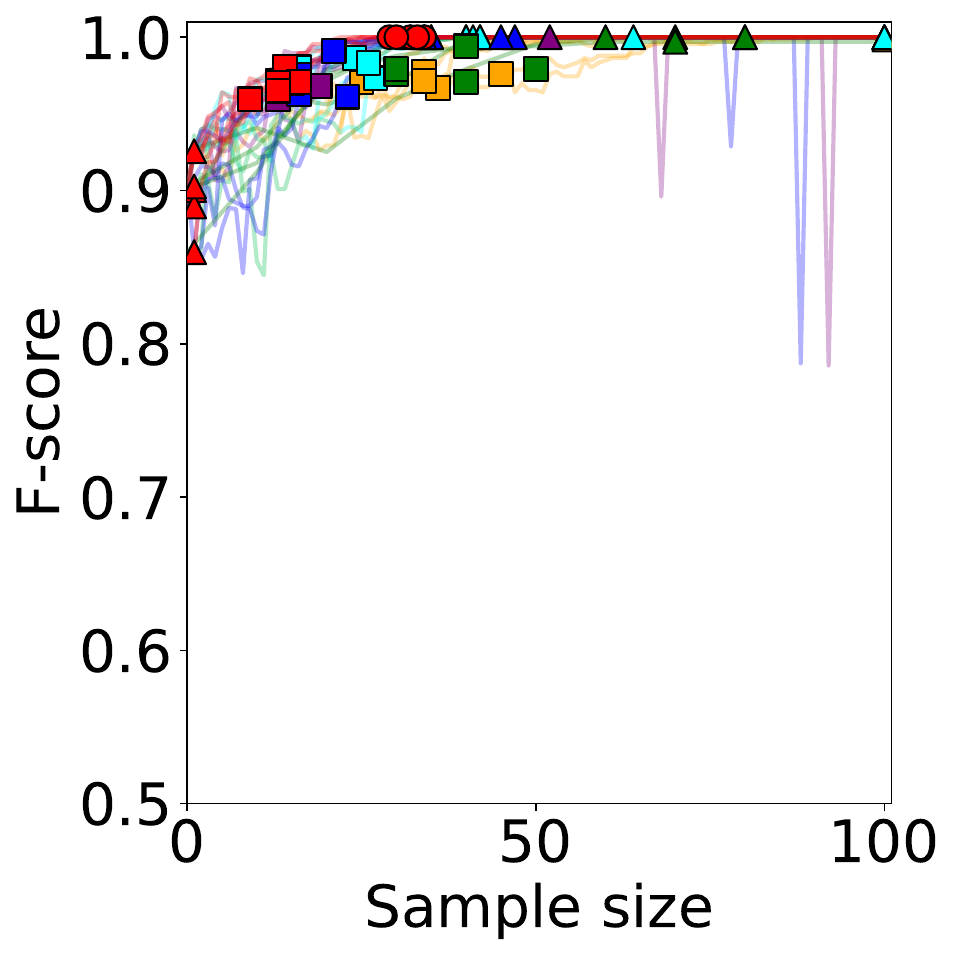}&
    \includegraphics[height=4cm,width=4cm]{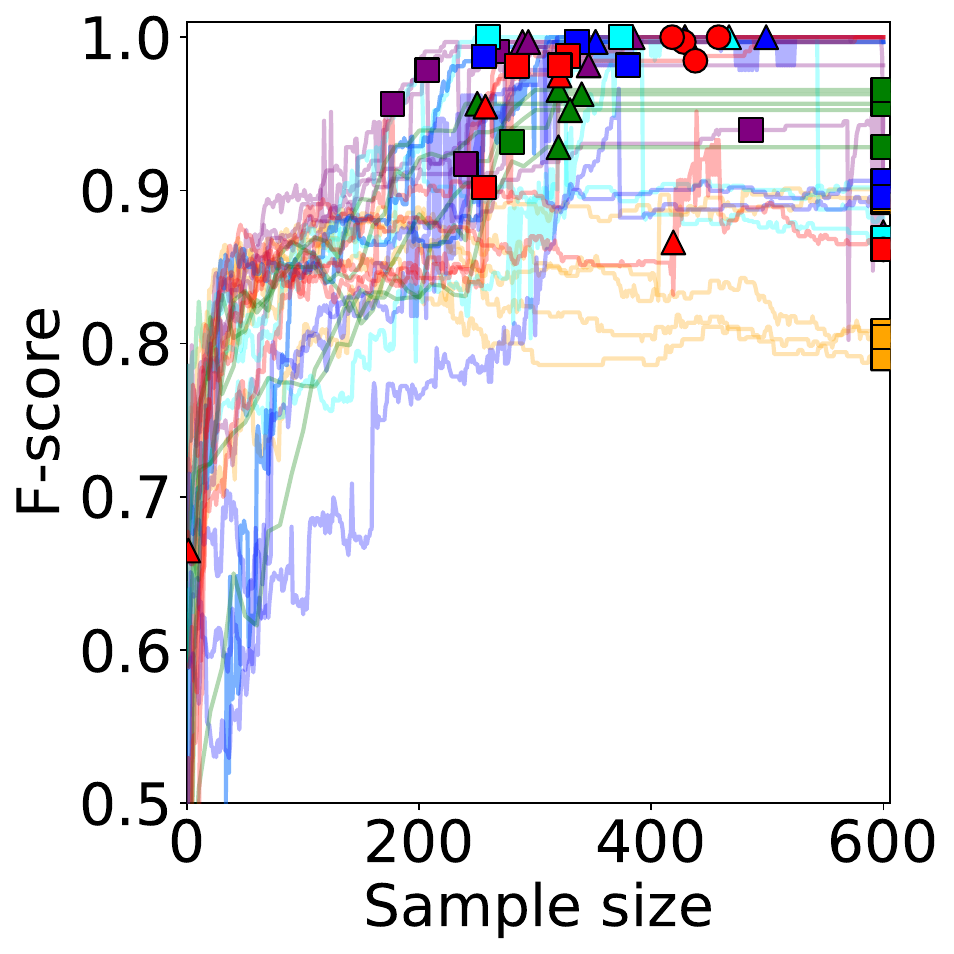}\\
    (a) {\tt{Sphere}}($\sigma_{\rm noise}=0$)& (b) {\tt{Rosenbrock}}($\sigma_{\rm noise}=0$)& (c) {\tt{Cross in tray}}($\sigma_{\rm noise}=0$)  \\
    \includegraphics[height=4cm,width=4cm]{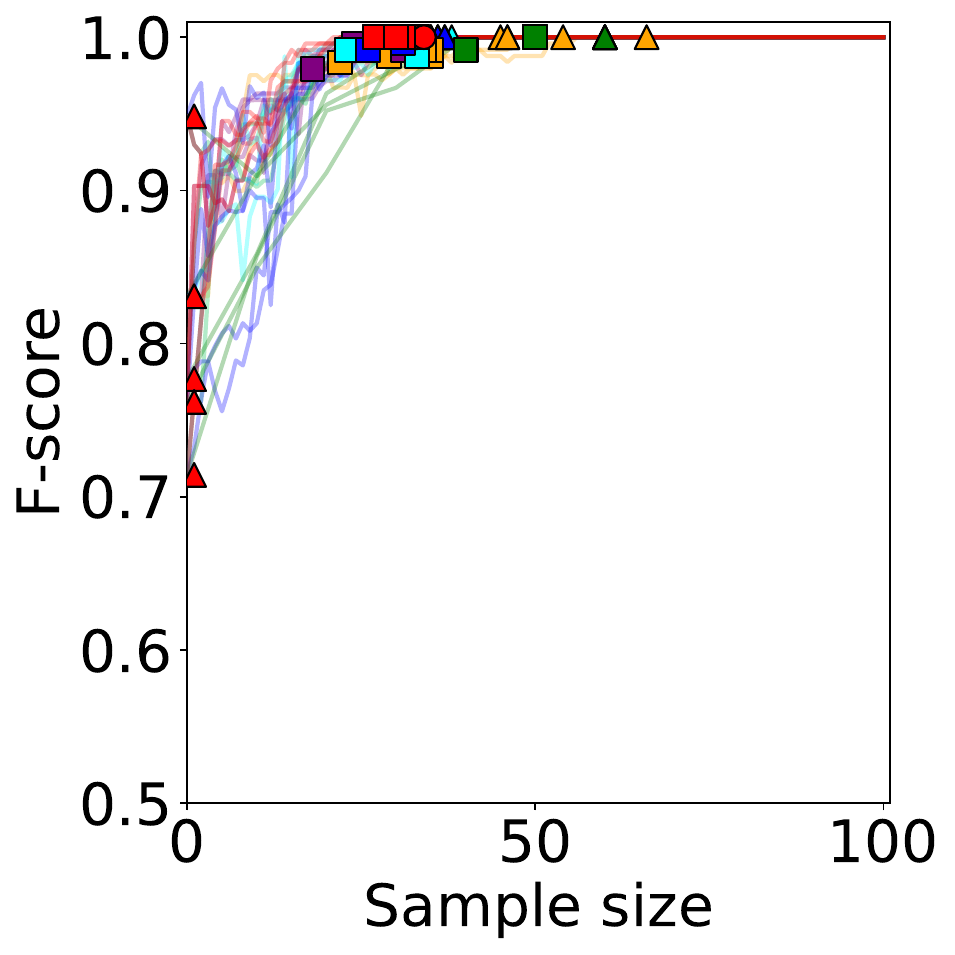}&    
    \includegraphics[height=4cm,width=4cm]{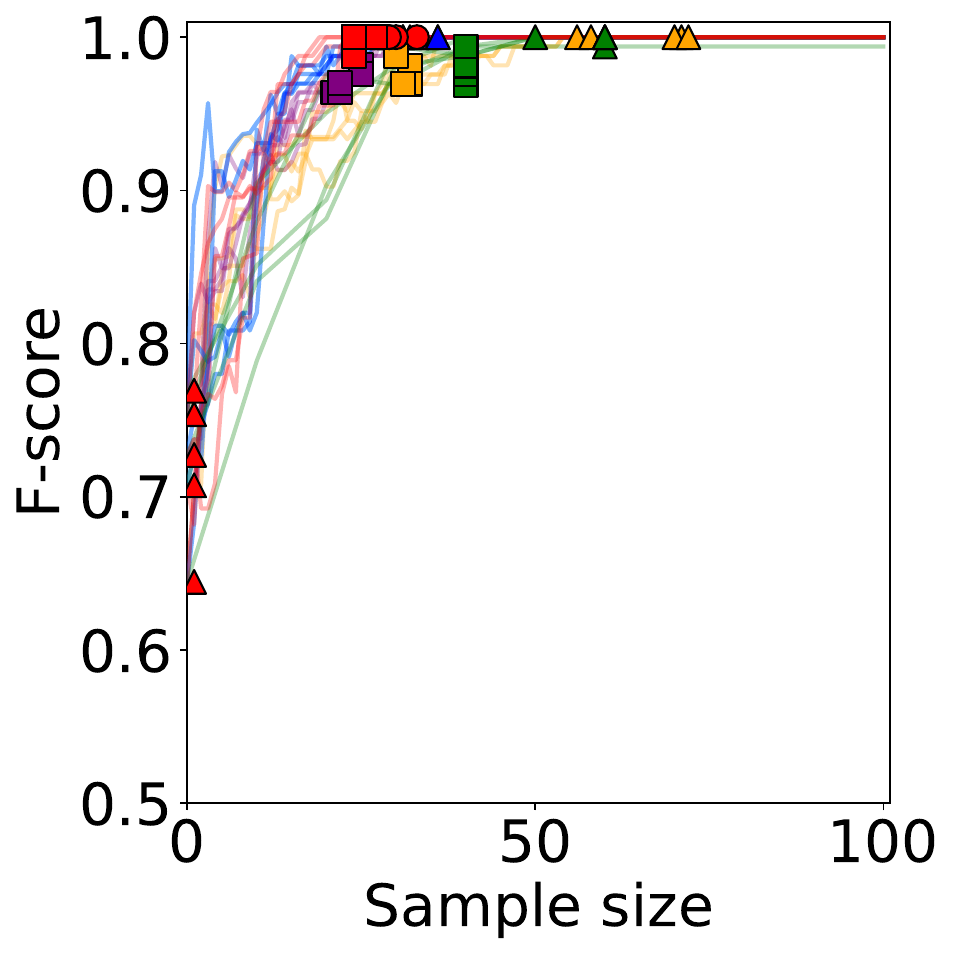}&
    \includegraphics[height=4cm,width=4cm]{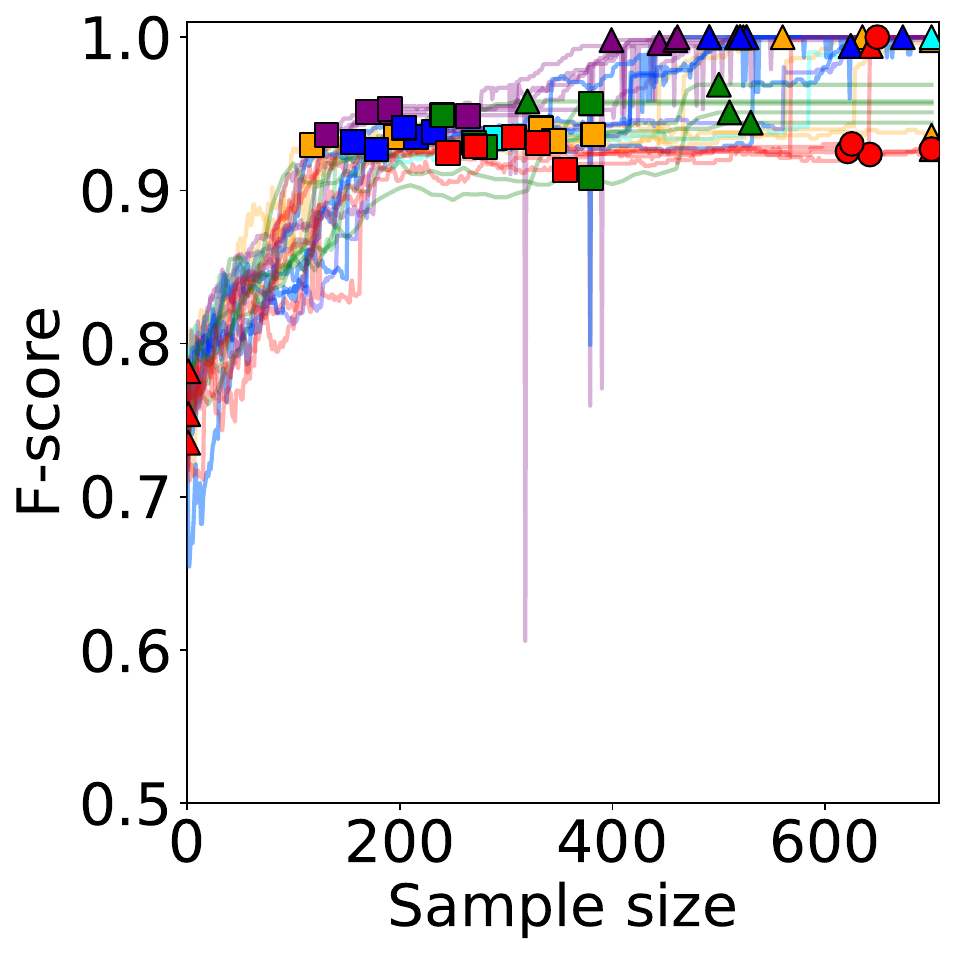}\\
    (d) {\tt{Booth}}($\sigma_{\rm noise}=0$)& (e) {\tt{Branin}}($\sigma_{\rm noise}=0$)& (f) {\tt{Holder table}}($\sigma_{\rm noise}=0$)  \\
    \includegraphics[height=4cm,width=4cm]{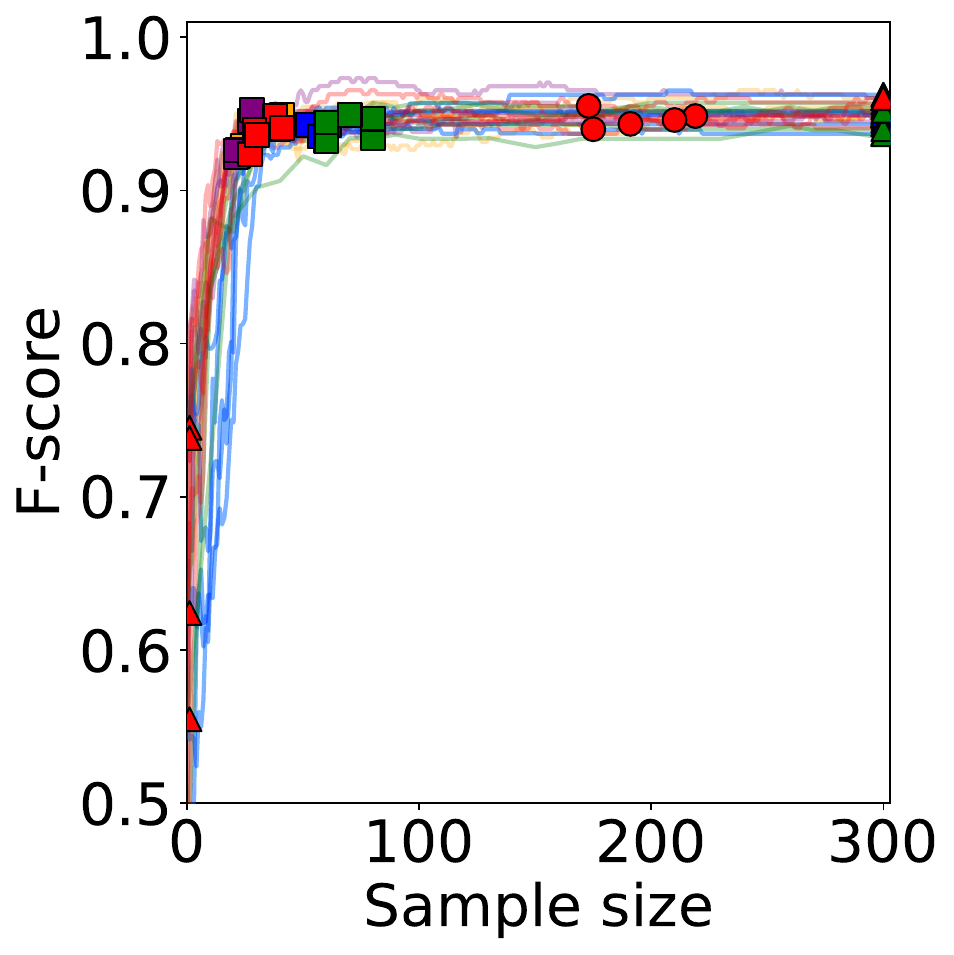}&    
    \includegraphics[height=4cm,width=4cm]{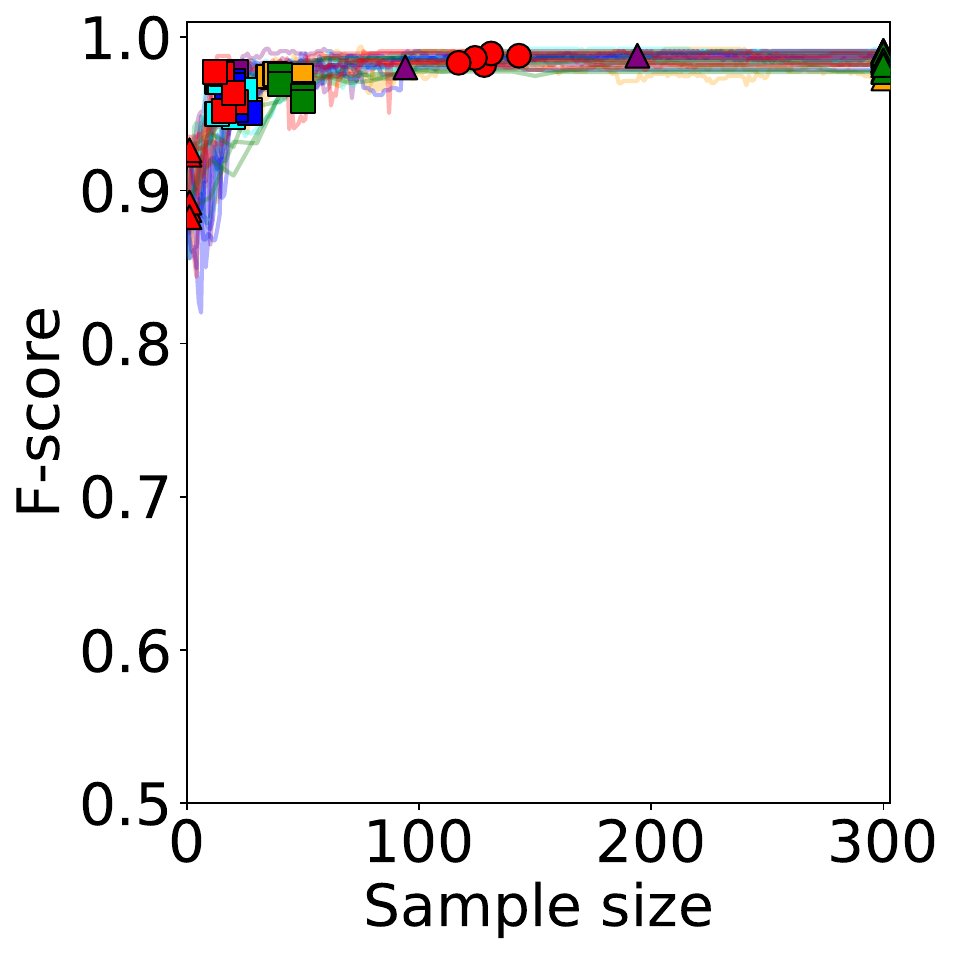}&
    \includegraphics[height=4cm,width=4cm]{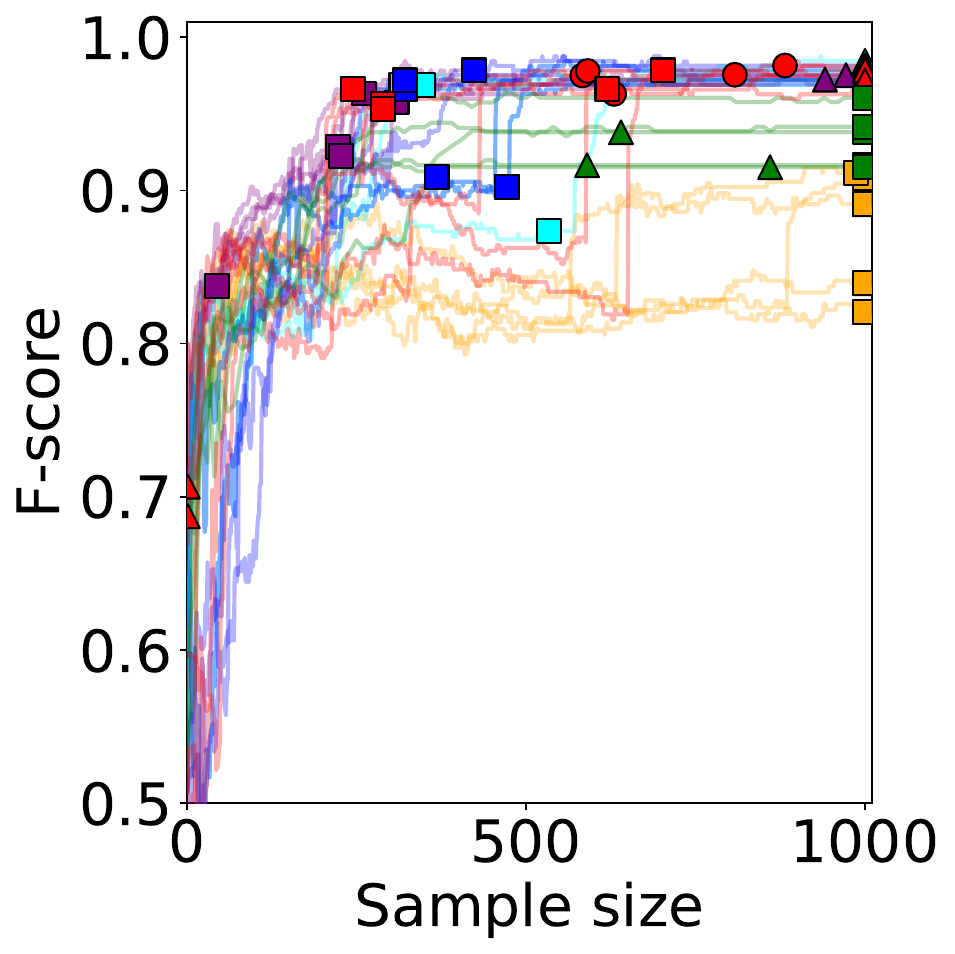}\\
    (g) {\tt{Sphere}}($\sigma_{\rm noise}=0.5$)& (h) {\tt{Rosenbrock}}($\sigma_{\rm noise}=30$)& (i) {\tt{Cross in tray}}($\sigma_{\rm noise}=0.01$)  \\
    \includegraphics[height=4cm,width=4cm]{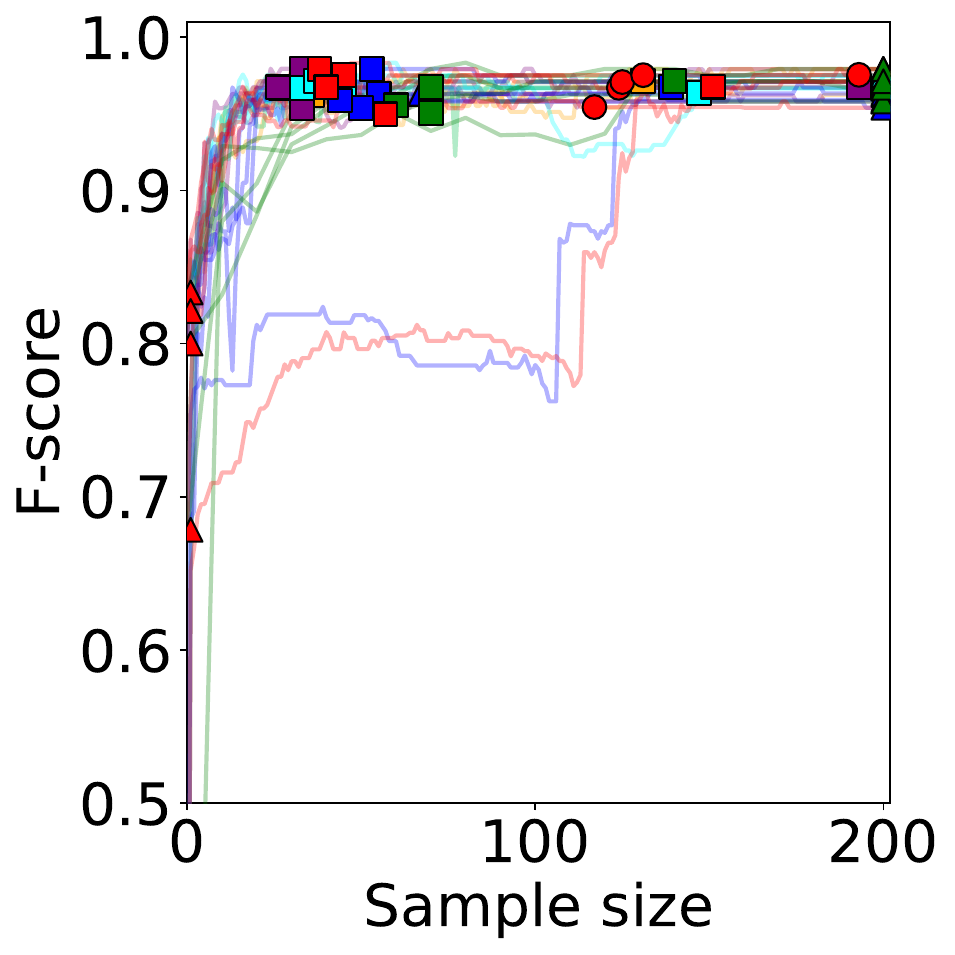}&    
    \includegraphics[height=4cm,width=4cm]{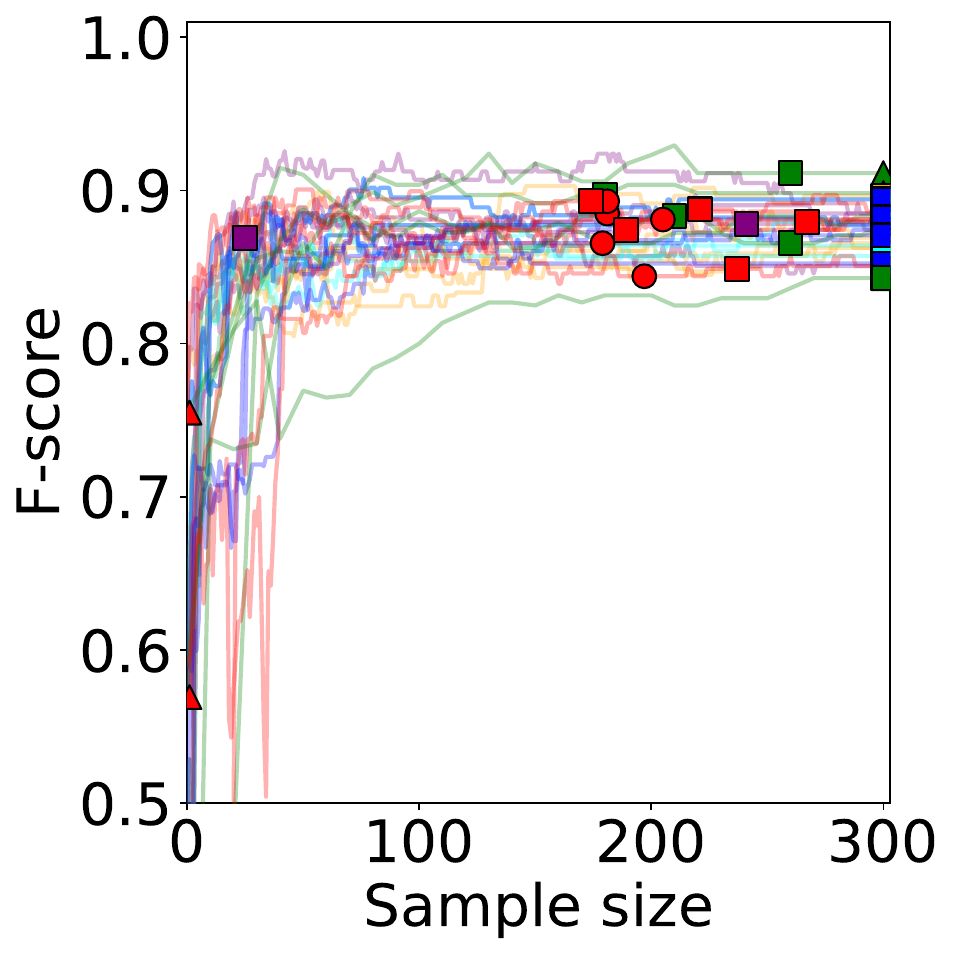}&
    \includegraphics[height=4cm,width=4cm]{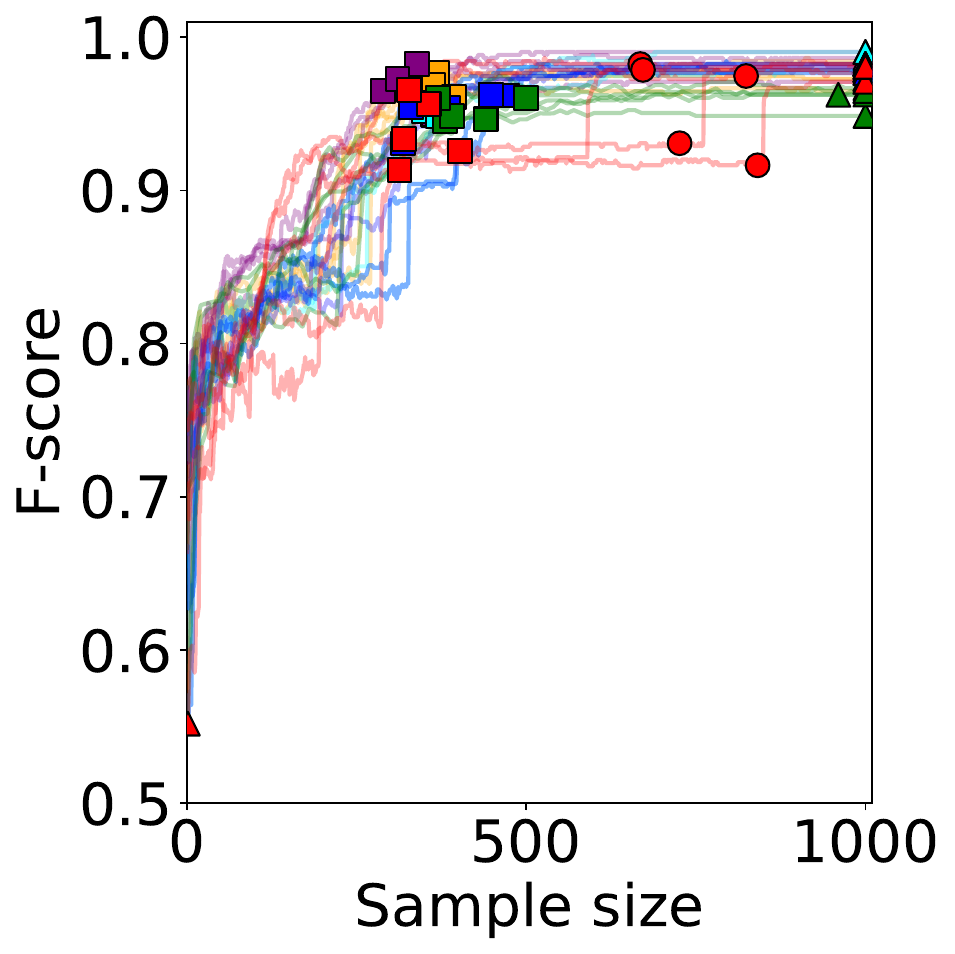}\\
    (j) {\tt{Booth}}($\sigma_{\rm noise}=30$)& (k) {\tt{Branin}}($\sigma_{\rm noise}=10$)& (l) {\tt{Holder table}}($\sigma_{\rm noise}=0.3$)
    \end{tabular}
    \caption{
    F-score and stopped time using each acquisition function with proposed (Our) and fully classified (FC) criteria. (a) -- (f) noise-free case. (g) -- (l) noise addition case. Error bars mean standard deviation. Each label has the same meaning as in Fig.~\ref{fig_result_test_func}
    }
    \label{fig_result_test_func_fs}
\end{figure}

In this subsection, we demonstrate that the proposed method can stop LSE when applied to various test functions. We also show results without adding observation noise, illustrating that the standard stopping criterion (fully classified: FC) can stop LSE in noise-free scenarios, but fails when noise is present. In this experiment, we use the {\tt{Sphere}} function, {\tt{Rosenbrock}} function, and {\tt{Cross in tray}} function, as well as the {\tt{Booth}} function, {\tt{Branin}} function, and {\tt{Holder table}} function. The thresholds for these test functions are set as follows: $\theta=20$ for the {\tt{Sphere}} function, $\theta=100$ for the {\tt{Rosenbrock}} function, $\theta=-1.5$ for the {\tt{Cross in tray}} function, $\theta=500$ for the {\tt{Booth}} function, $\theta=100$ for the {\tt{Branin}} function, and $\theta=-3$ for the {\tt{Holder table}} function. When adding observation noise, the noise levels are set to $\sigma_{\rm noise}=2$ for the {\tt{Sphere}} function, $\sigma_{\rm noise}=30$ for the {\tt{Rosenbrock}} function, $\sigma_{\rm noise}=0.01$ for the {\tt{Cross in tray}} function, $\sigma_{\rm noise}=30$ for the {\tt{Booth}} function, $\sigma_{\rm noise}=20$ for the {\tt{Branin}} function, and $\sigma_{\rm noise}=0.3$ for the {\tt{Holder table}} function. Other experimental settings remain the same as in the main text.

Figure~\ref{fig_result_test_func_fs} shows the F-scores for the iteration (number of evaluation) when different acquisition function is used for exploration. From these results, it can be seen that there are no significant differences between the acquisition functions, except for the baseline method US. Additionally, comparing the cases with and without noise, while the convergence point of the F-score changes due to noise, the efficiency of the acquisition functions is not significantly affected by the presence of noise. Regarding the timing of each stop criterion, all stopping criteria can stop LSE in the noise-free case. On the other hand, in the presence of noise, similar to the experimental results in the main text, the FC criterion fails to stop LSE, whereas both the FS criterion and the proposed criterion can stop it. While the FS criterion tends to stop LSE aggressively, there are cases, such as in Fig.~\ref{fig_result_test_func_fs}(k), where it fails to do so. In contrast, the proposed criterion stops LSE more conservatively but is able to stop LSE even in cases like Fig.~\ref{fig_result_test_func_fs}(k).

\subsection{Applicability of the proposed stopping criterion to other acquisition functions}
\label{app:otherAF}

The proposed stopping criterion terminates LSE when the classification probability $p^{\rm max}(\vx)=\max\{\Pr (\vx \in H_{\theta} ),\Pr (\vx \in L_{\theta} )\}$ of all candidate points becomes large enough, or the probability that the true function is included in $\cE$ is large enough. Since this stopping criterion does not assume anything about the acquisition function, it can be, in principle, applied to other acquisition functions as well. In this subsection, we verify whether LSE can be stopped using the proposed stopping criterion with other acquisition functions. The test functions, noise levels, and other experimental settings are the same as in the Appendix~\ref{app:testfunc}.

Figure~\ref{fig_result_test_func_other_acq_func} shows the stopping times for each acquisition function. The results indicate that, in the absence of noise, the proposed stopping criterion can stop LSE at the timing when the F-score converges, similar to the standard stopping criterion (FC). However, when observation noise is added, most acquisition functions, except the proposed one, fail to stop LSE. This is because the proposed stopping criterion is conservative. Specifically, in this experiment, the threshold for the stopping criterion is set to $\delta=0.99$, meaning that the average classification probability for each candidate point needs to be $1 - (1 - \delta)/|\cX| = 0.999975$ for LSE to stop. 

In contrast, typical LSE classifies a candidate point to the upper level set when $\mu_N(\vx) - \beta\sigma_N(\vx) > \theta$, which is equivalent to classifying the point when $\Pr (\vx \in H_{\theta} )= \Phi((\mu_N(\vx) - \theta)/\sigma_N(\vx)) > \Phi(\beta)$. Similarly, for the lower level set, a point is classified when $\Pr (\vx \in L_{\theta} ) = \Phi((\theta - \mu_N(\vx))/\sigma_N(\vx)) > \Phi(\beta)$. In this experiment, $\beta = 1.96$, so a candidate point is classified into the upper or lower level set when its classification probability exceeds $\Phi(\beta) \approx 0.975$. 

Typically, once a candidate point is classified into the upper or lower level set, it no longer needs to be explored, making it difficult to meet the stopping criterion requirements, even with repeated exploration. Therefore, it is hard for acquisition functions other than the proposed method, which prioritizes exploring the candidate points with the lowest classification probability, to stop LSE using the proposed stopping criterion.

\begin{figure}[th!]
    \centering
    \begin{tabular}{ccc}
    \includegraphics[height=4cm,width=4cm]{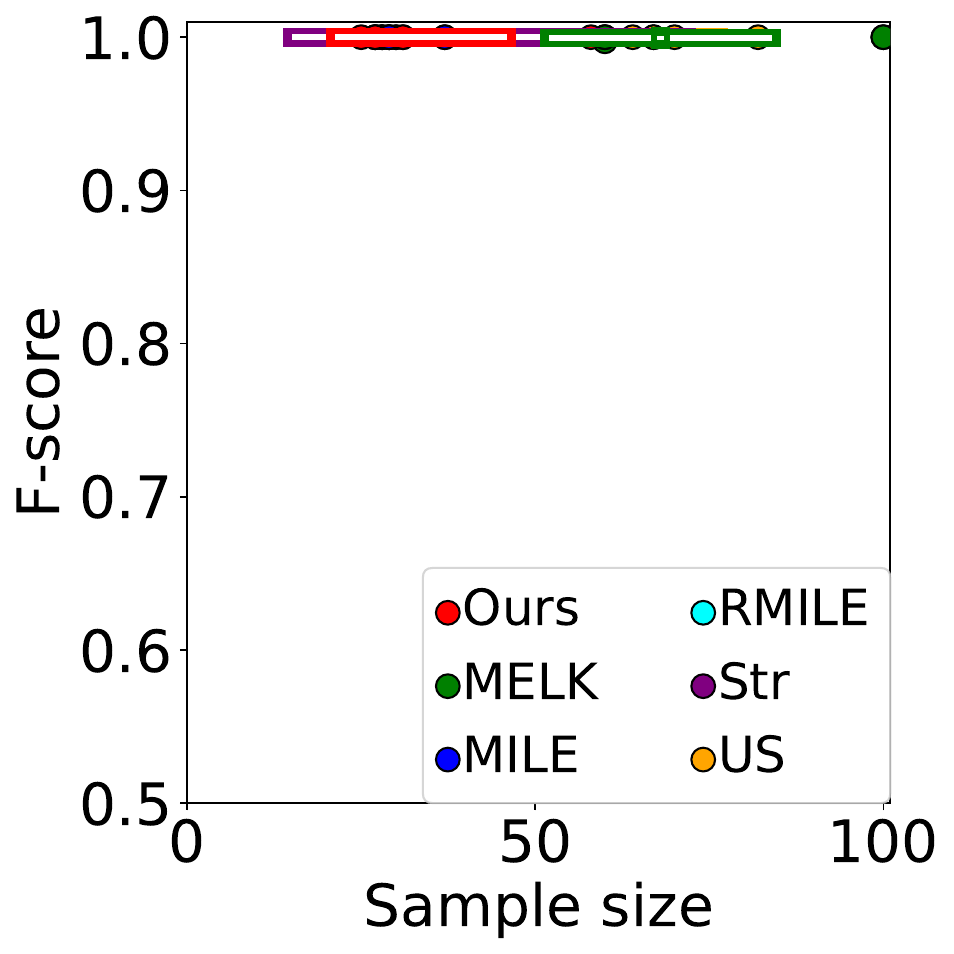}&    
    \includegraphics[height=4cm,width=4cm]{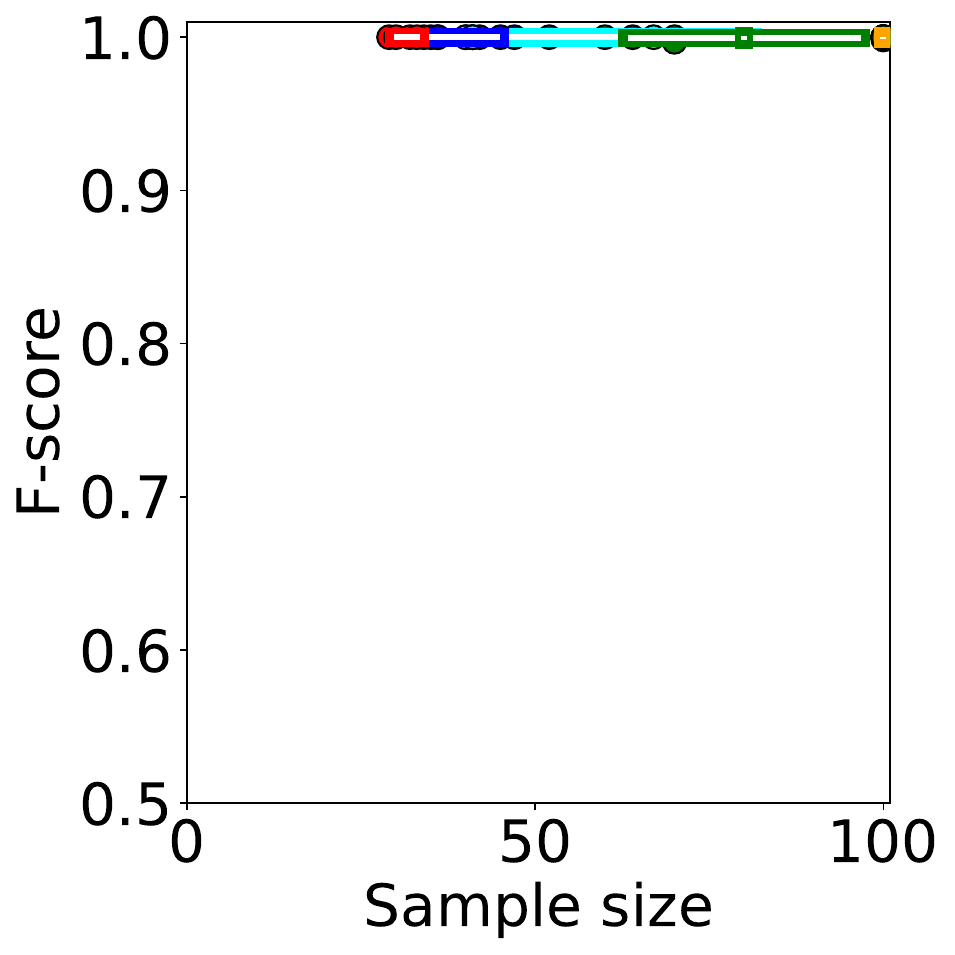}&
    \includegraphics[height=4cm,width=4cm]{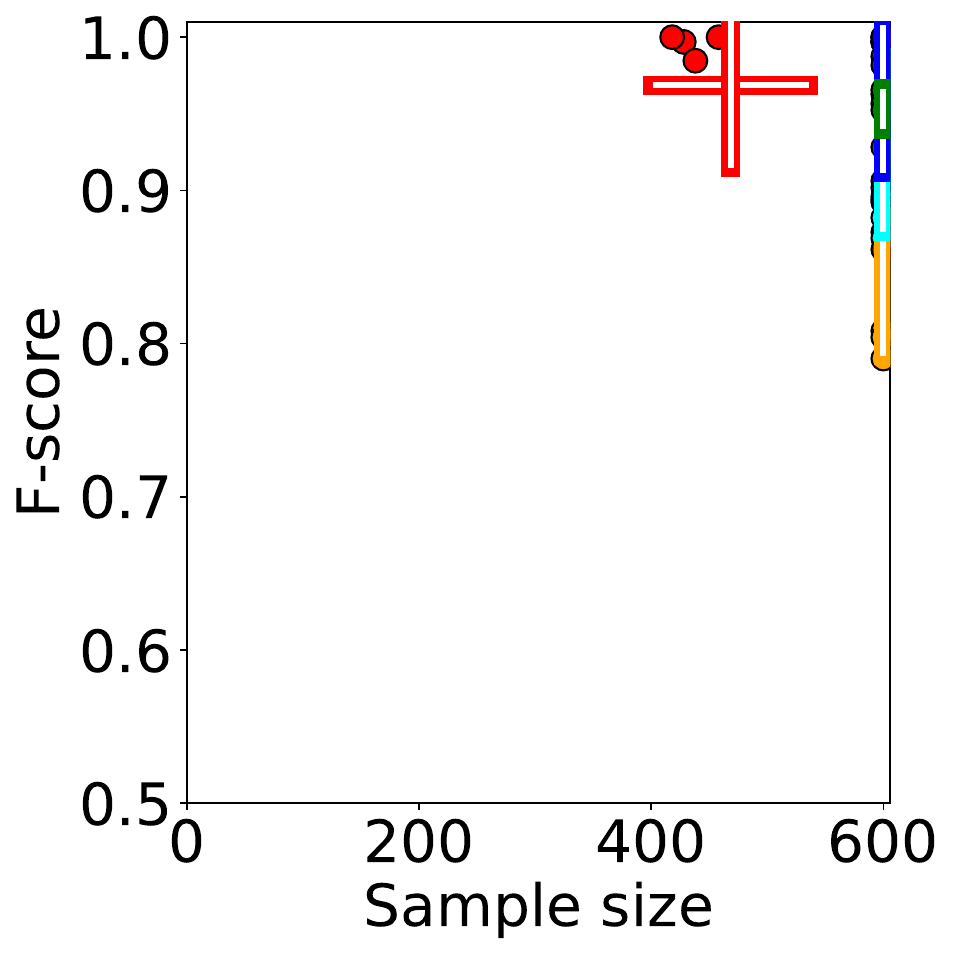}\\
     (a) {\tt{Sphere}}($\sigma_{\rm noise}=0$)& (b) {\tt{Rosenbrock}}($\sigma_{\rm noise}=0$)& (c) {\tt{Cross in tray}}($\sigma_{\rm noise}=0$)  \\
   \includegraphics[height=4cm,width=4cm]{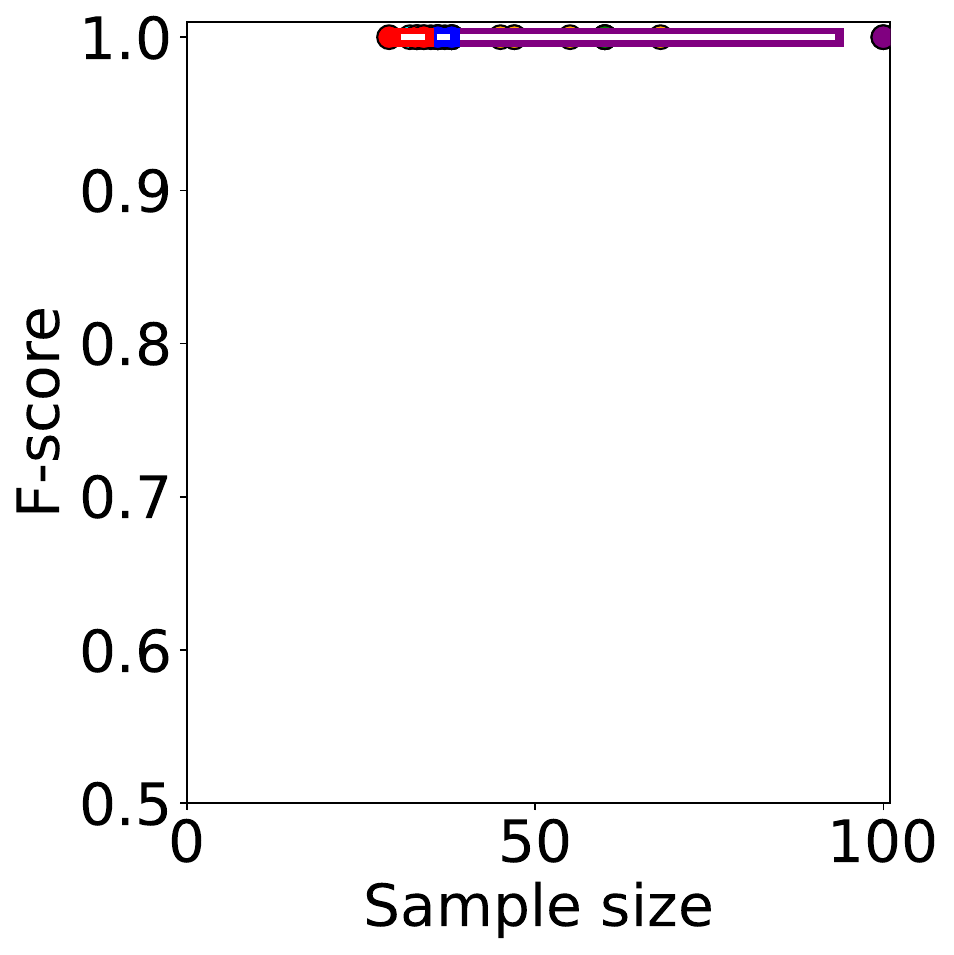}&    
    \includegraphics[height=4cm,width=4cm]{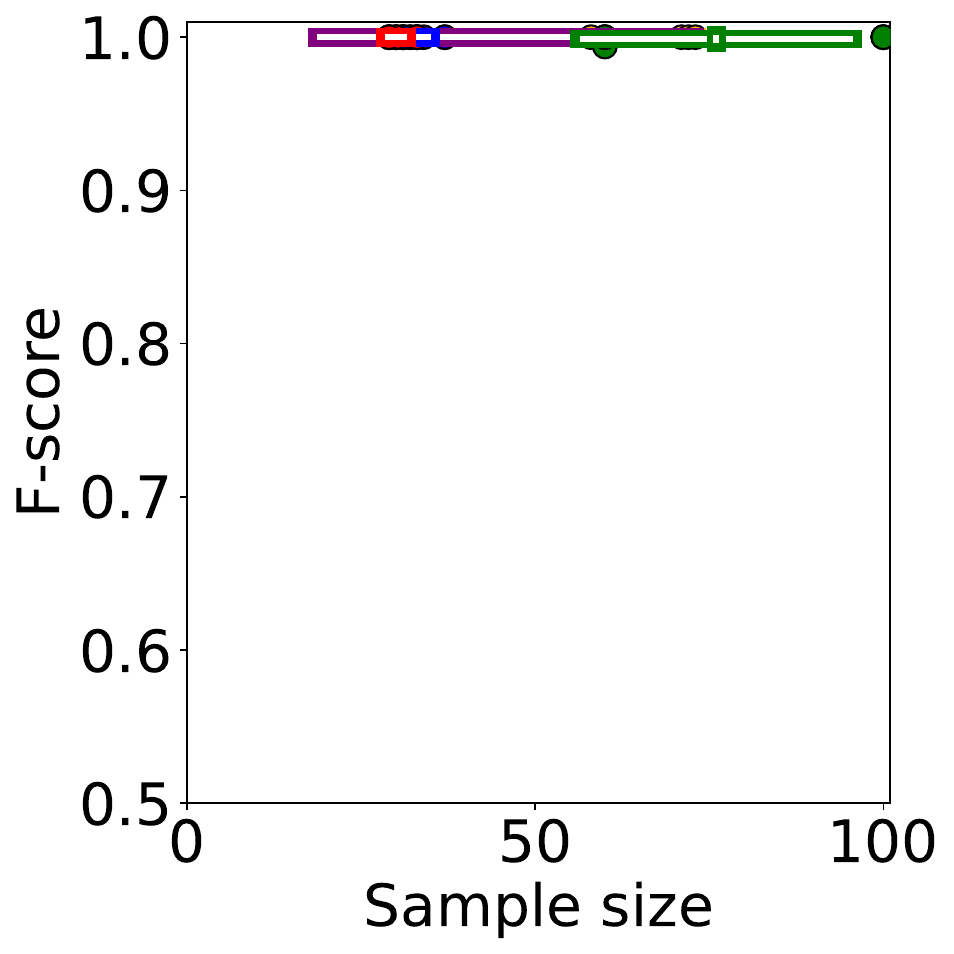}&
    \includegraphics[height=4cm,width=4cm]{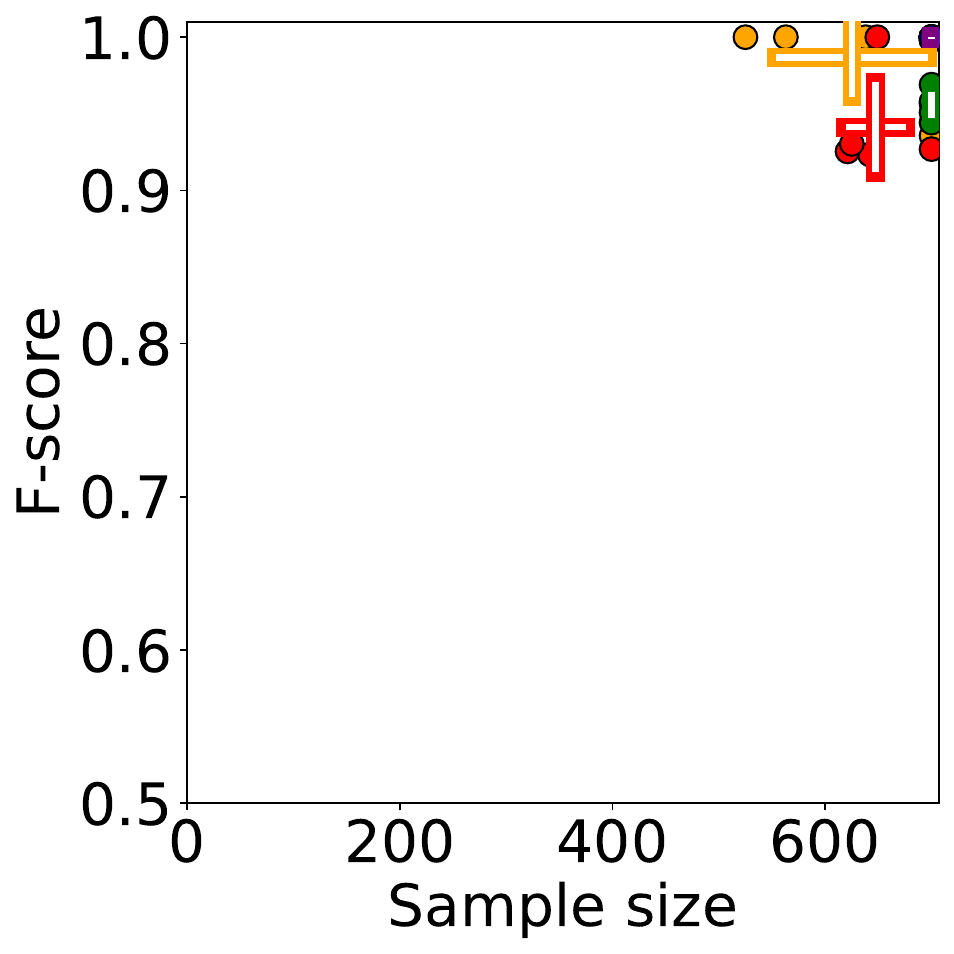}\\
    (d) {\tt{Booth}}($\sigma_{\rm noise}=0$)& (e) {\tt{Branin}}($\sigma_{\rm noise}=0$)& (f) {\tt{Holder table}}($\sigma_{\rm noise}=0$)  \\
    \includegraphics[height=4cm,width=4cm]{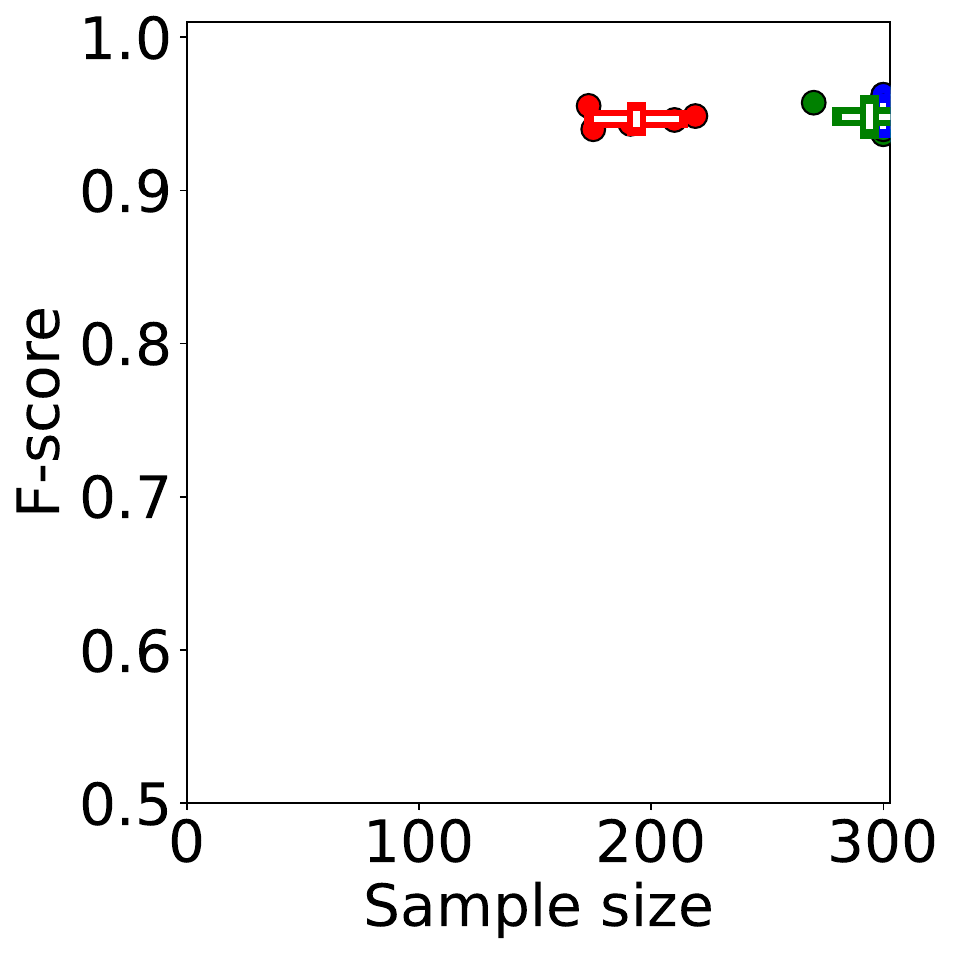}&    
    \includegraphics[height=4cm,width=4cm]{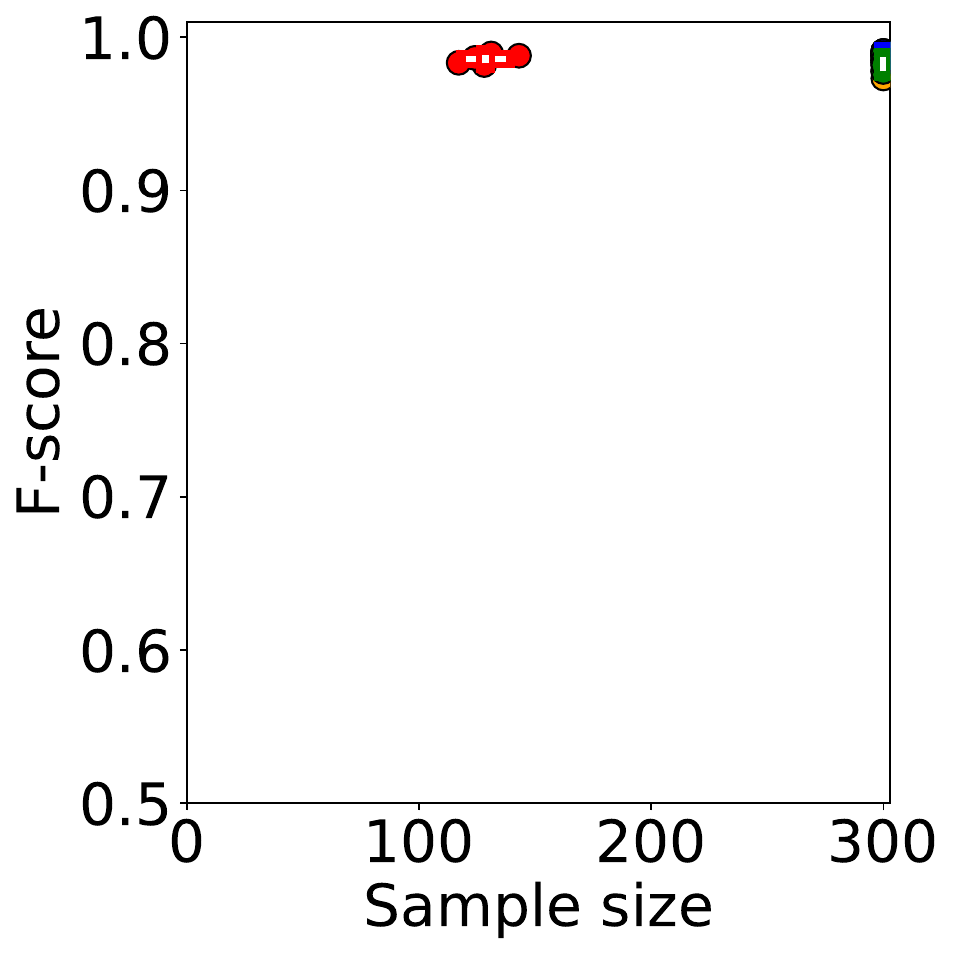}&
    \includegraphics[height=4cm,width=4cm]{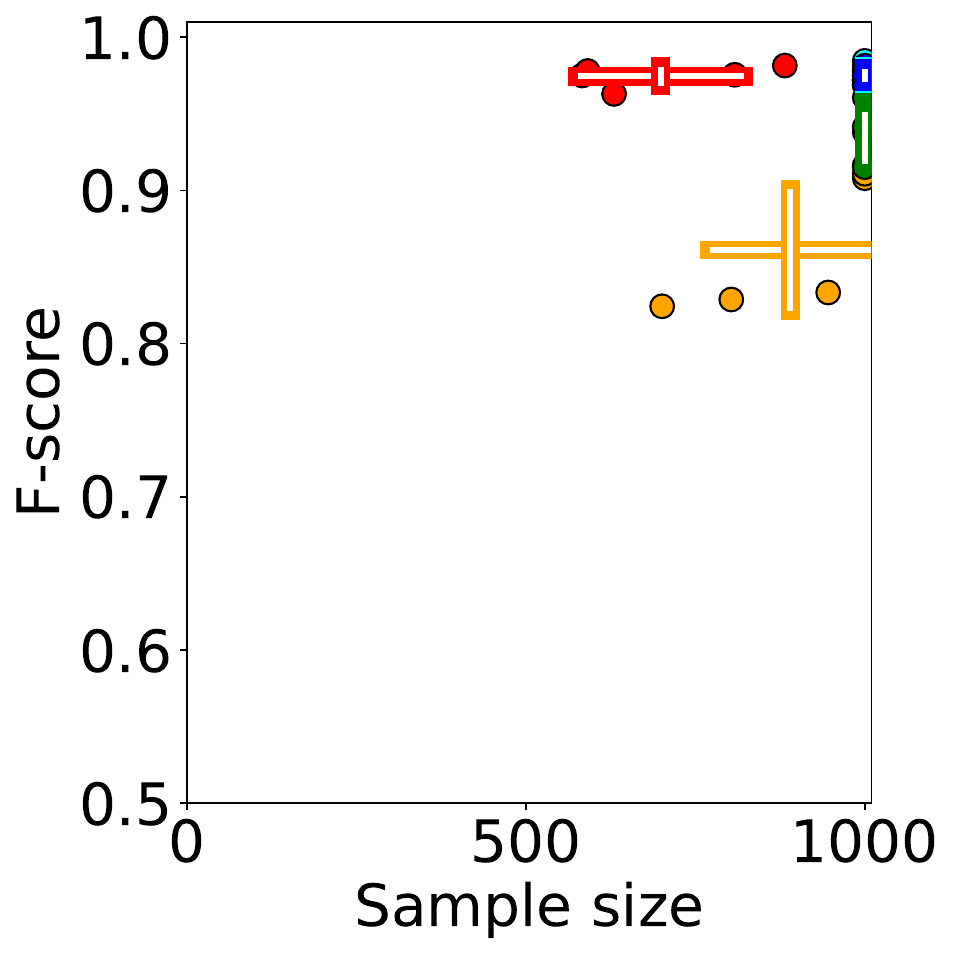}\\
    (g) {\tt{Sphere}}($\sigma_{\rm noise}=0.5$)& (h) {\tt{Rosenbrock}}($\sigma_{\rm noise}=30$)& (i) {\tt{Cross in tray}}($\sigma_{\rm noise}=0.01$)  \\
    \includegraphics[height=4cm,width=4cm]{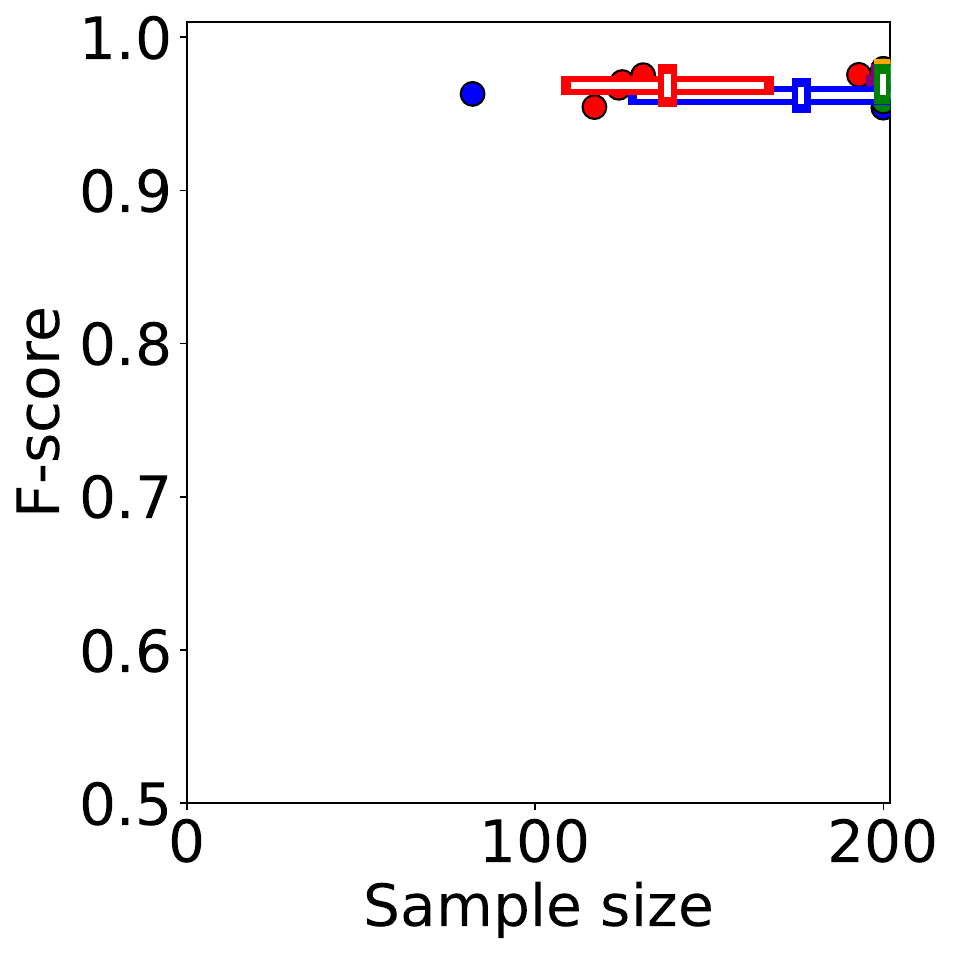}&    
    \includegraphics[height=4cm,width=4cm]{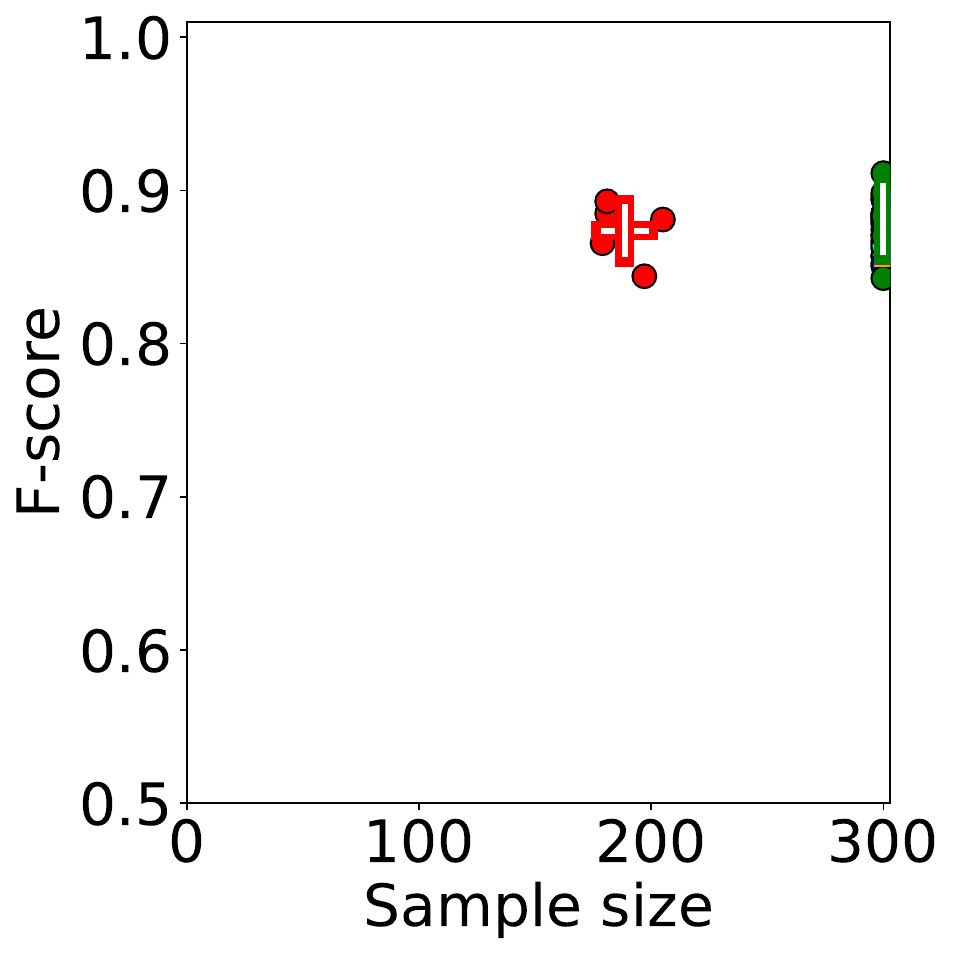}&
    \includegraphics[height=4cm,width=4cm]{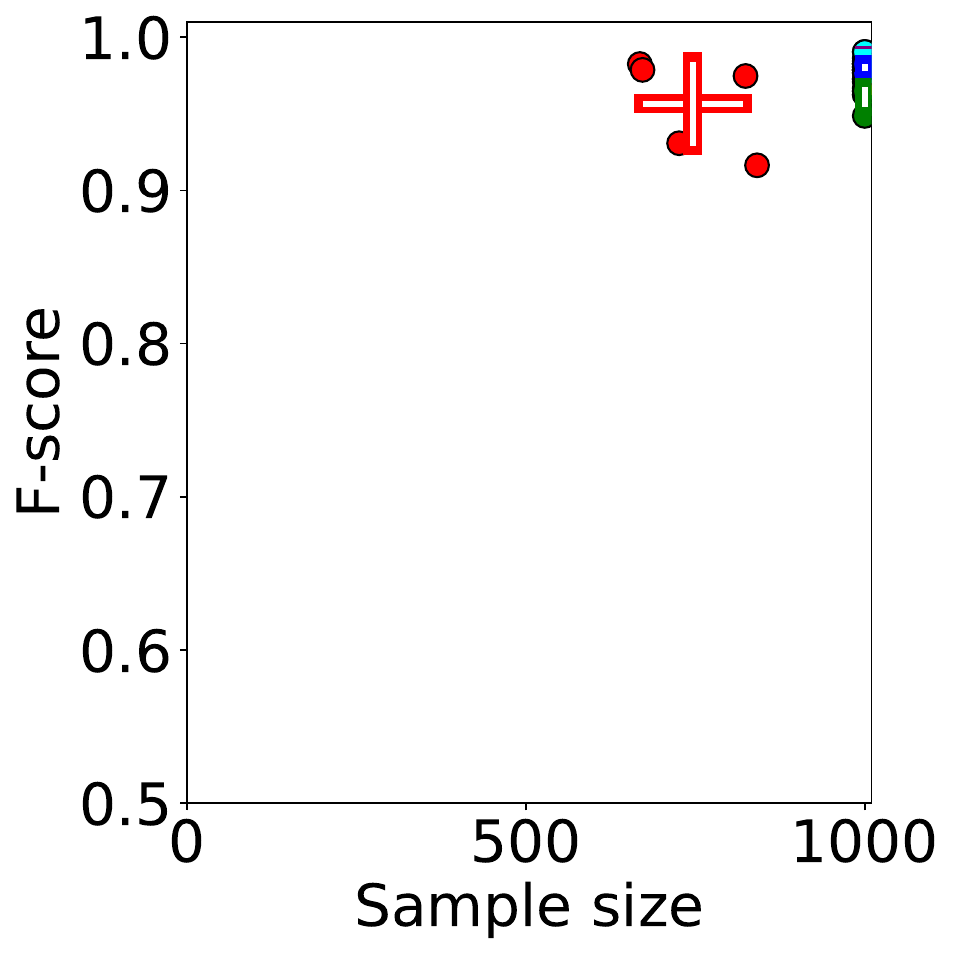}\\
    (j) {\tt{Booth}}($\sigma_{\rm noise}=30$)& (k) {\tt{Branin}}($\sigma_{\rm noise}=10$)& (l) {\tt{Holder table}}($\sigma_{\rm noise}=0.3$)
    \end{tabular}
    \caption{
    Stopped time using each acquisition function with the proposed stopping criterion. (a) -- (f) noise-free case. (g) -- (l) noise addition case. Error bars mean standard deviation. Error bars mean standard deviation.
    }
    \label{fig_result_test_func_other_acq_func}
\end{figure}

\begin{figure}[th!]
    \centering
    \begin{tabular}{cc}
    \includegraphics[height=4cm,width=4cm]{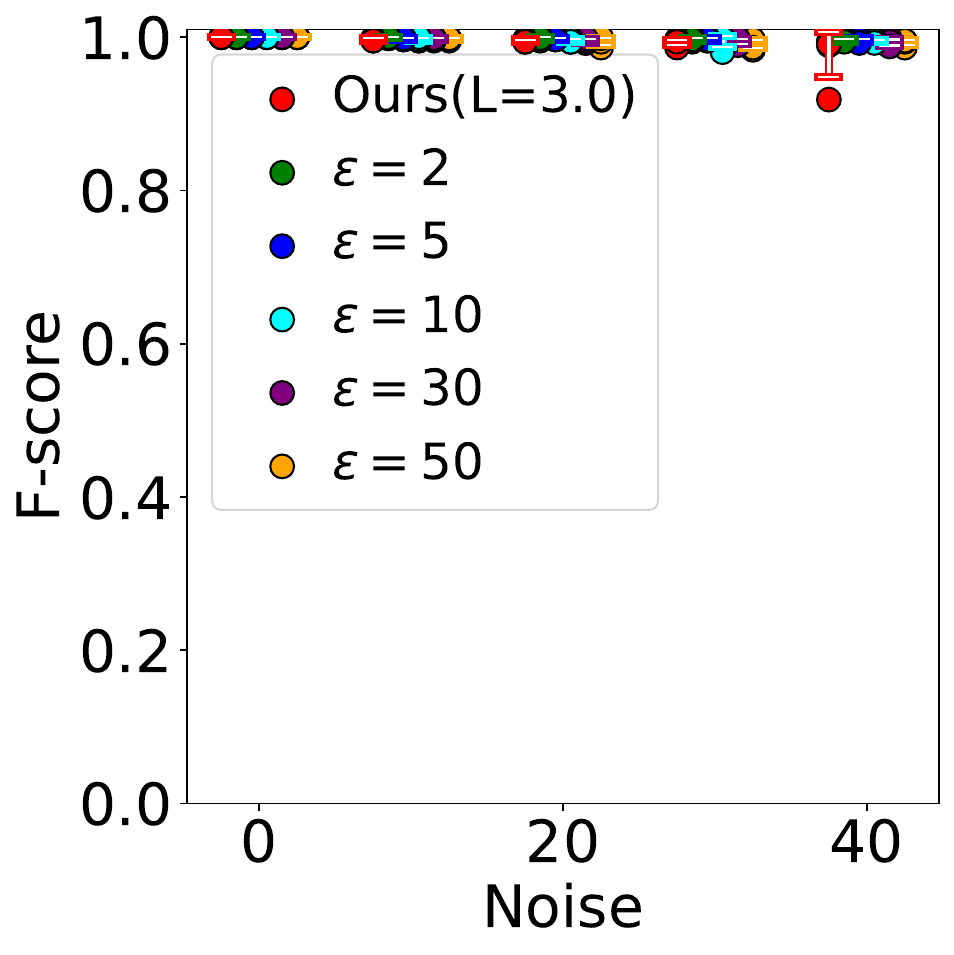}&    
    \includegraphics[height=4cm,width=4cm]{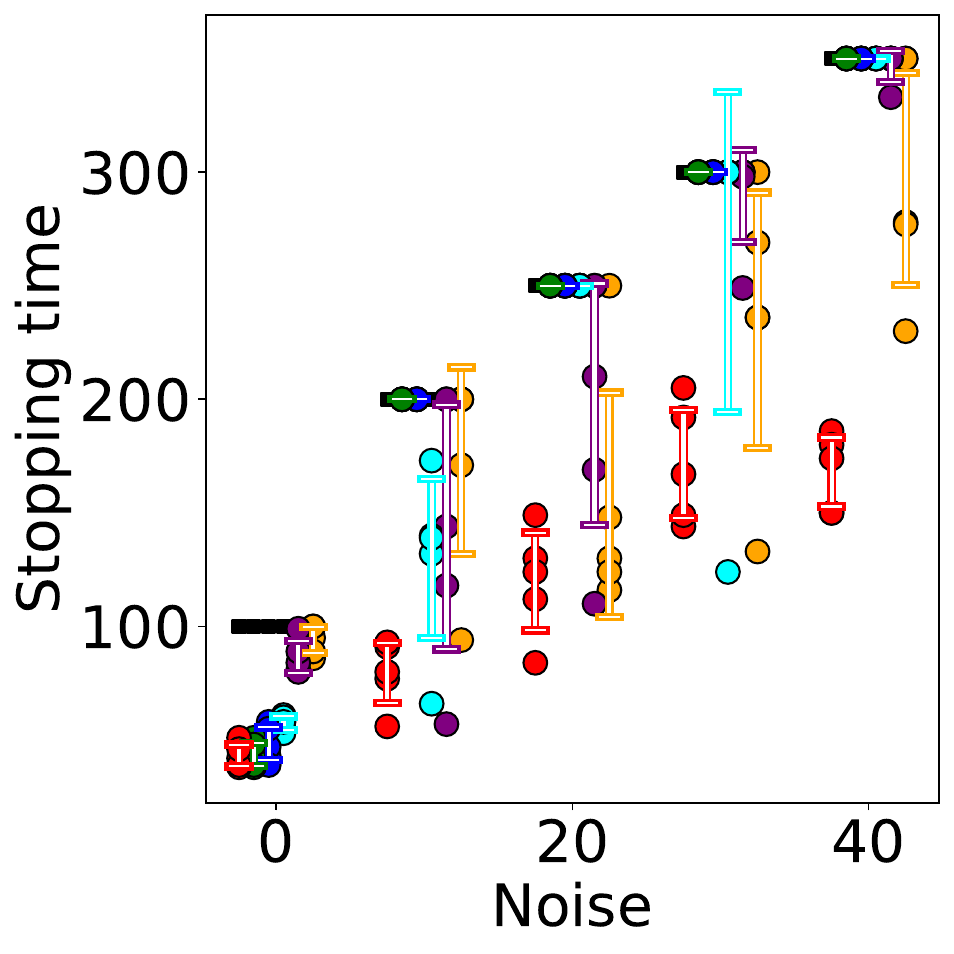}\\
     (a) F-score of {\tt{Rosenbrock}}& (b) Stopping time of {\tt{Rosenbrock}}  \\
   \includegraphics[height=4cm,width=4cm]{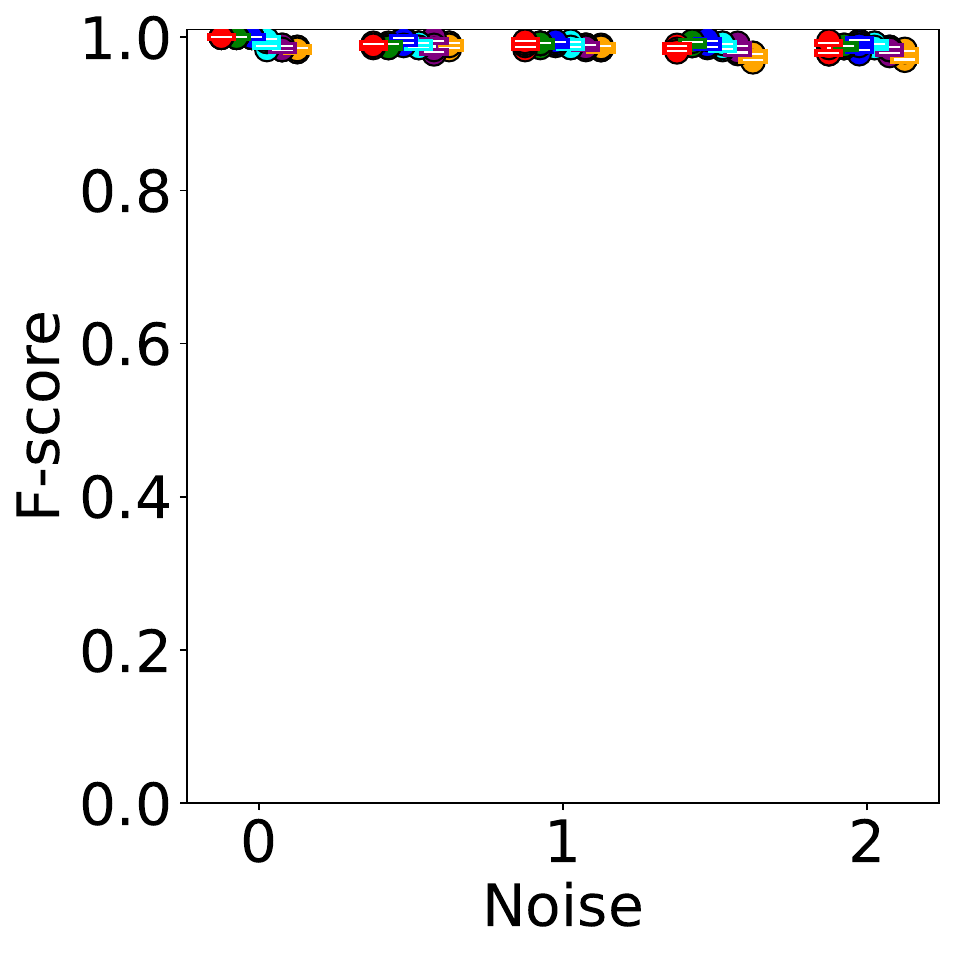}&    
    \includegraphics[height=4cm,width=4cm]{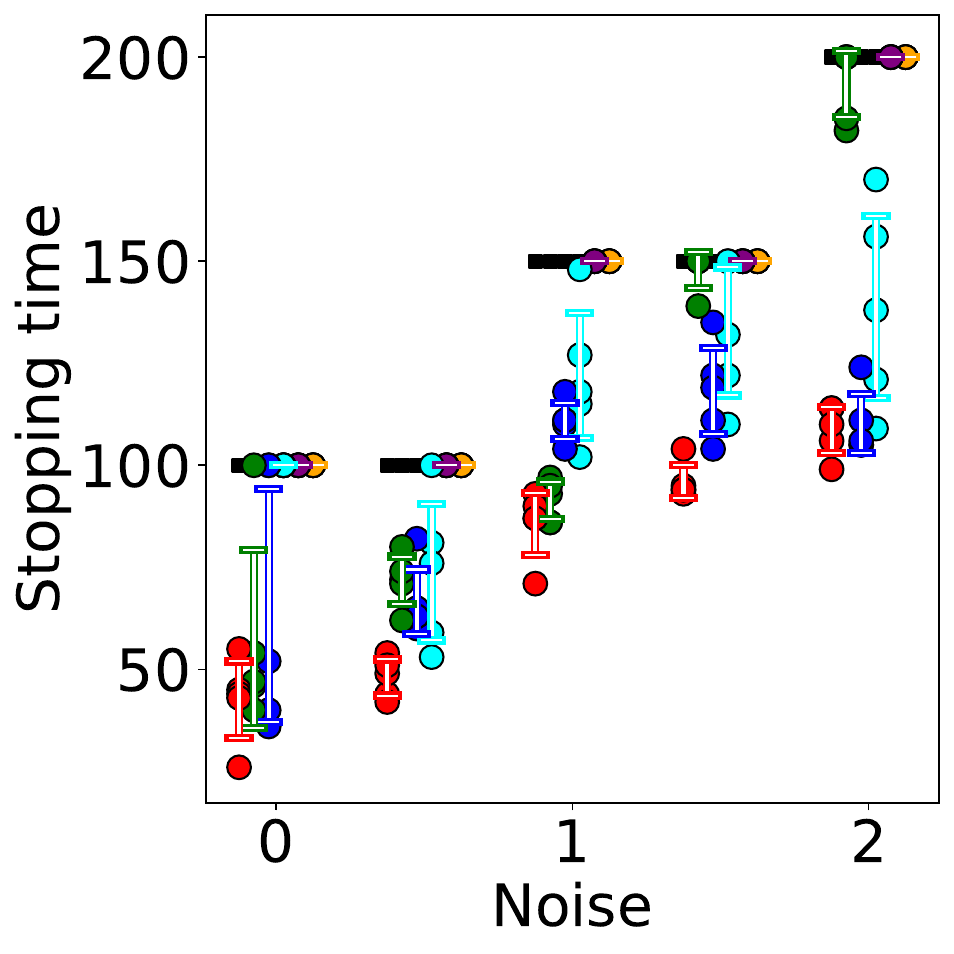}\\
    (c) F-score of {\tt{Sphere}}& (d) Stopping time of {\tt{Sphere}}
    \end{tabular}
    \caption{
    The impact of noise variance and the range of function in the case of the proposed method and the method for setting $\epsilon$ directly. The black line of (b) and (d) means budget. To make it easier to understand, a jitter is added to the x-coordinate to prevent overlapping of the drawings.
    }
    \label{fig_result_robustness_of_L}
\end{figure}

\subsection{The robustness of our margin-setting method}
\label{app:eps}

In this section, we demonstrate that the method for determining $\epsilon$ proposed in the Appendix~\ref{seq:determine_margin} is less dependent on noise variance and the range of the objective function than directly setting $\epsilon$. To illustrate this, we compare the stopping timings using five different $\epsilon$ values $\{2, 5, 10, 30, 50\}$ with the proposed method. The test functions used for this evaluation are the {\tt{rosenbrock}} and {\tt{sphere}} functions, and the experiments assess the impact of varying Gaussian noise variance in five patterns for each function. For the {\tt{rosenbrock}} function, the noise variances are set to $\{0^2, 10^2, 20^2, 30^2, 40^2\}$, and for the {\tt{sphere}} function, the variances are set to $\{0^2, 0.5^2, 1.0^2, 1.5^2, 2.0^2\}$. In the experiments, the proposed method for determining $\epsilon$ sets $L=3$, and all other parameters follow the settings used for the experiment of the test functions.

The experimental results for the {\tt{rosenbrock}} function are shown in Figs.~\ref{fig_result_robustness_of_L}(a) and (b). From Fig.~\ref{fig_result_robustness_of_L}(a), we can observe that, regardless of the method for determining $\epsilon$, no reduction in F-scores occurs, and the LSE does not stop until the F-scores are close enough to their limit, even when noise is varied. However, from Fig.~\ref{fig_result_robustness_of_L}(b), we see that when $\epsilon=2$, $\epsilon=5$, and $\epsilon=10$, the LSE can stop before using the entire budget in the noise-free case, but as the noise increases, it fails to stop even when the entire budget is used. Furthermore, for $\epsilon=30$ and $\epsilon=50$, the LSE stops earlier than the proposed method when the noise variance is up to $20^2$ and $30^2$, respectively, but as the noise increases further, the LSE fails to stop even after using the entire budget. Next, in the results for the {\tt{sphere}} function shown in Figs.~\ref{fig_result_robustness_of_L}(c) and (d), we find that, unlike in the {\tt{rosenbrock}} function, where $\epsilon=30$ and $\epsilon=50$ successfully stopped the LSE, they stop prematurely before the F-scores converge in the {\tt{sphere}} function. In contrast, $\epsilon=5$, which failed to stop the LSE in the {\tt{rosenbrock}} function, successfully stops the LSE in the {\tt{sphere}} function. This indicates that the appropriate $\epsilon$ must be chosen depending on the target function and the level of noise when directly setting $\epsilon$.

On the other hand, the proposed method can stop the LSE with relatively consistent accuracy across different functions and noise levels, despite using the same parameter settings. Thus, it has been demonstrated that the method for determining $\epsilon$ proposed in Appendix~\ref{seq:determine_margin} is less dependent on data noise or the range of the objective function than directly setting $\epsilon$. 

In conclusion, if an acceptable classification error $\epsilon$ is given, we use that value; otherwise, we can use $L$ as an easy-to-use knob.

\begin{figure}[th!]
    \centering
    \begin{tabular}{ccc}
    \includegraphics[height=4cm,width=4cm]{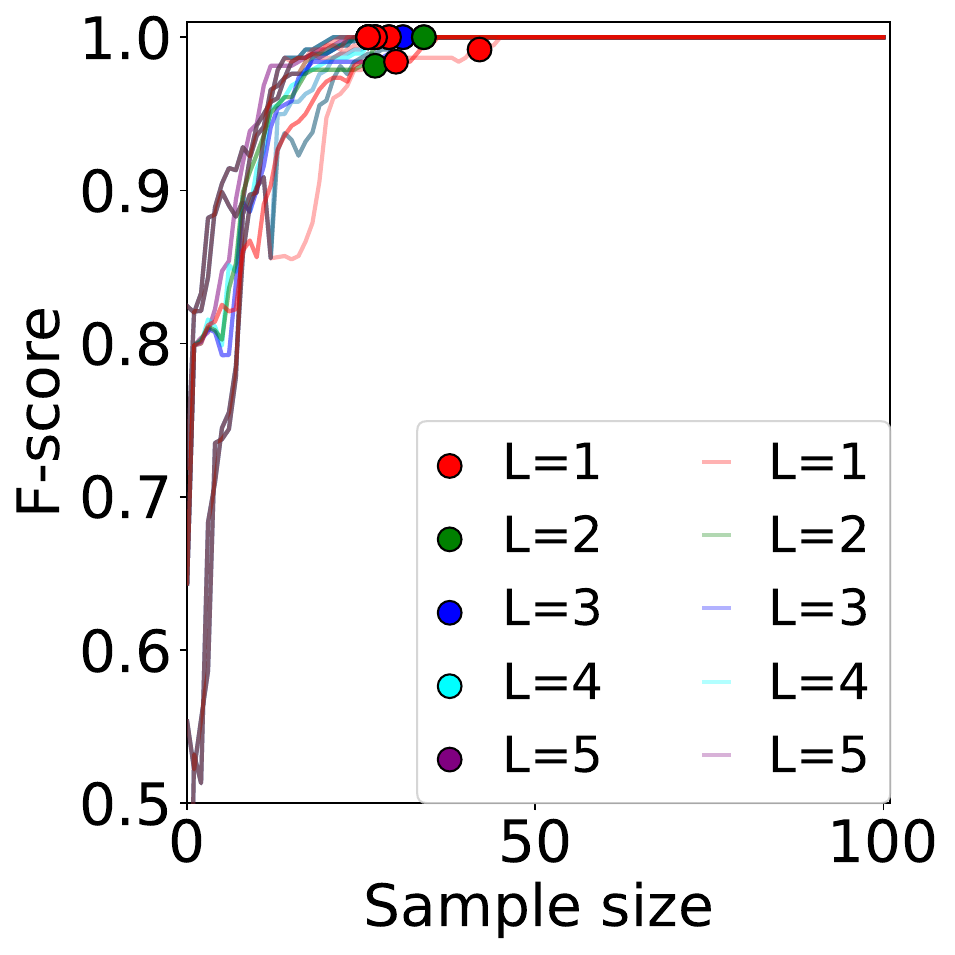}&    
    \includegraphics[height=4cm,width=4cm]{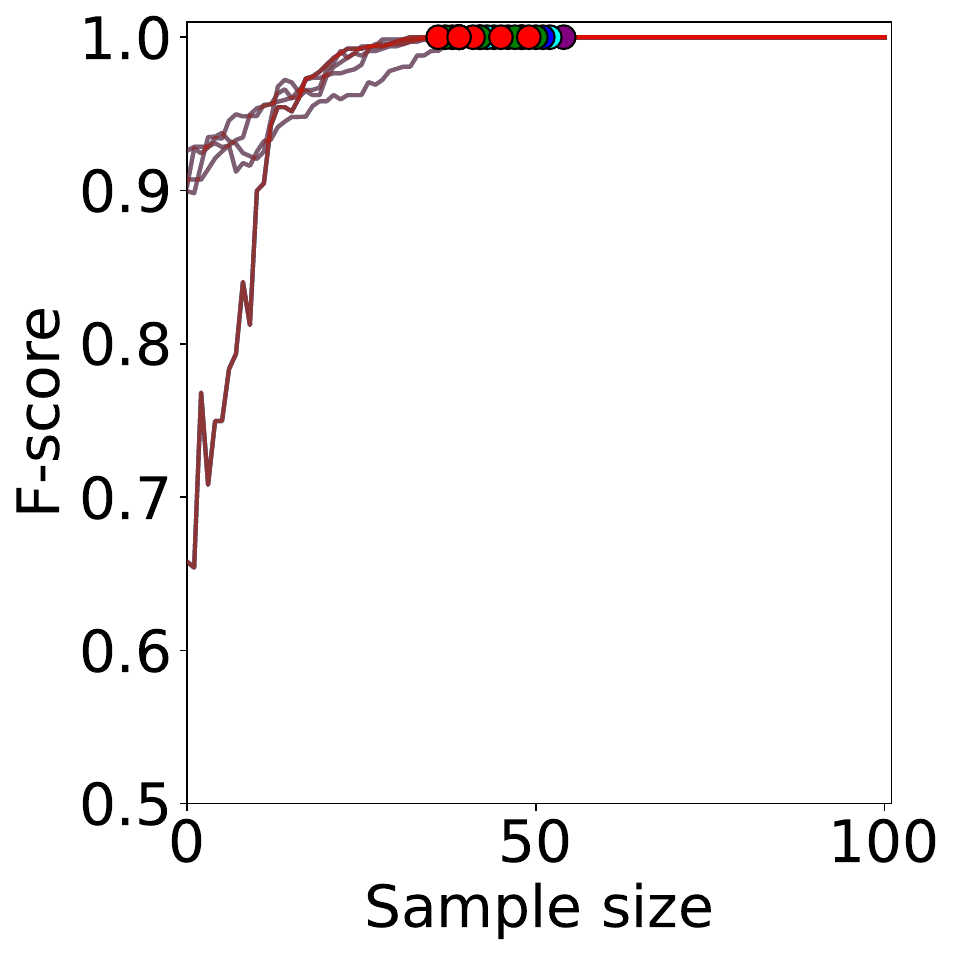}&
    \includegraphics[height=4cm,width=4cm]{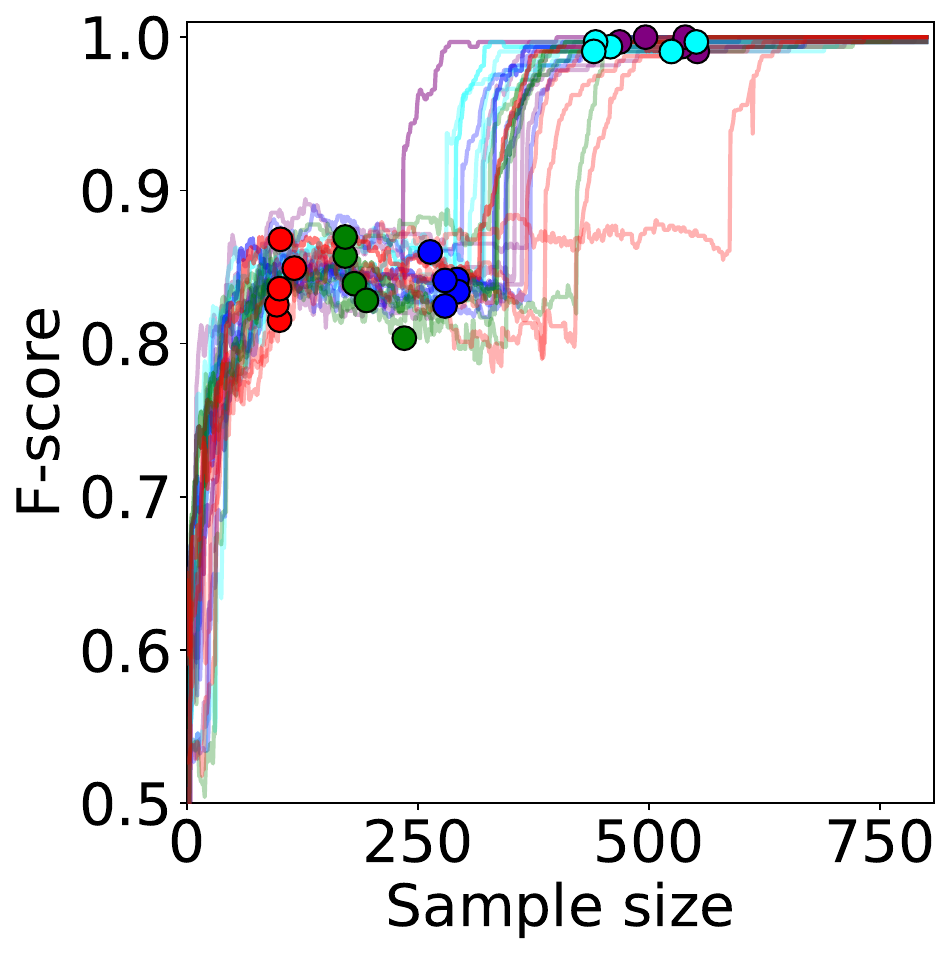}\\
     (a) {\tt{Sphere}}($\sigma_{\rm noise}=0$)& (b) {\tt{Rosenbrock}}($\sigma_{\rm noise}=0$)& (c) {\tt{Cross in tray}}($\sigma_{\rm noise}=0$)  \\
   \includegraphics[height=4cm,width=4cm]{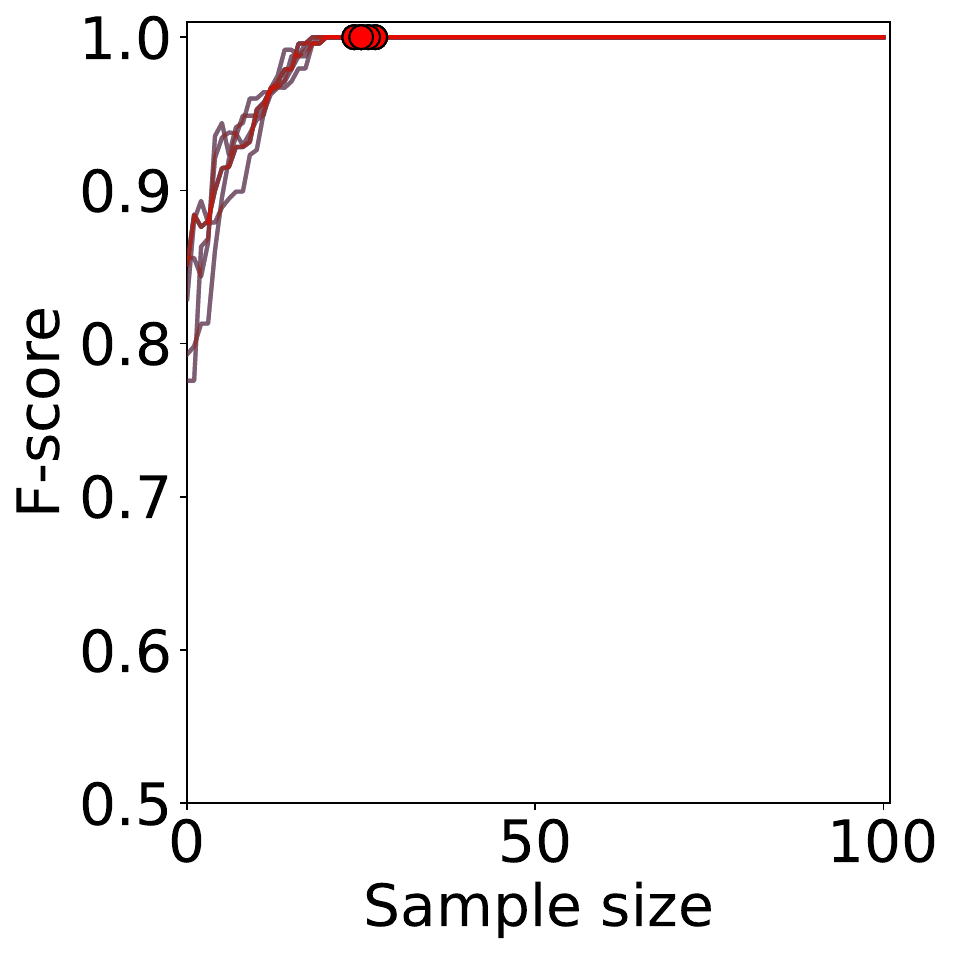}&    
    \includegraphics[height=4cm,width=4cm]{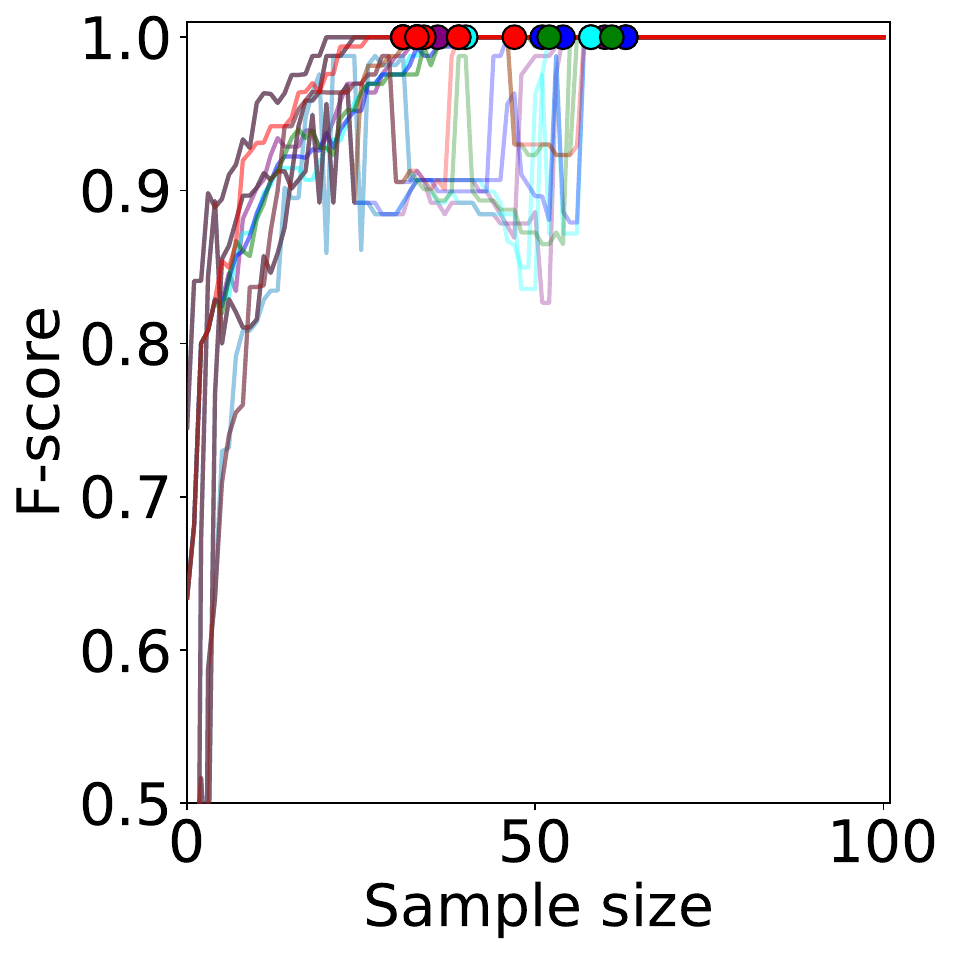}&
    \includegraphics[height=4cm,width=4cm]{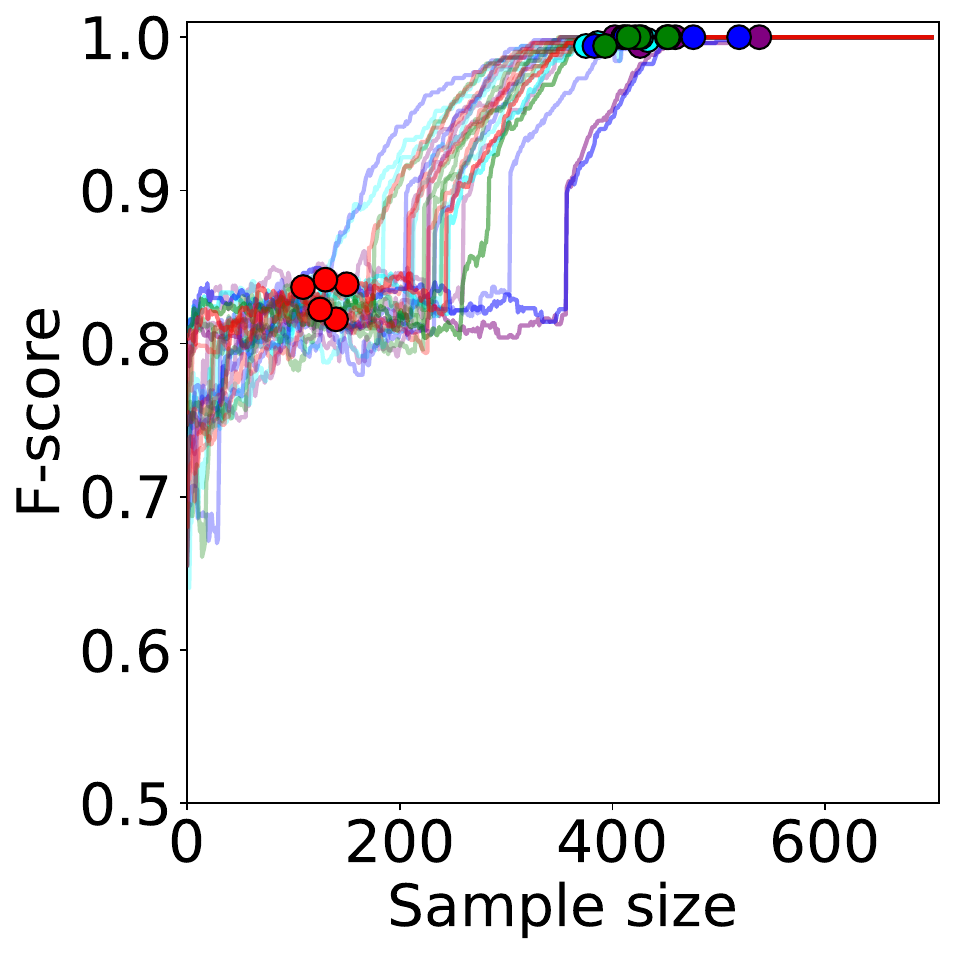}\\
    (d) {\tt{Booth}}($\sigma_{\rm noise}=0$)& (e) {\tt{Branin}}($\sigma_{\rm noise}=0$)& (f) {\tt{Holder table}}($\sigma_{\rm noise}=0$)  \\
    \includegraphics[height=4cm,width=4cm]{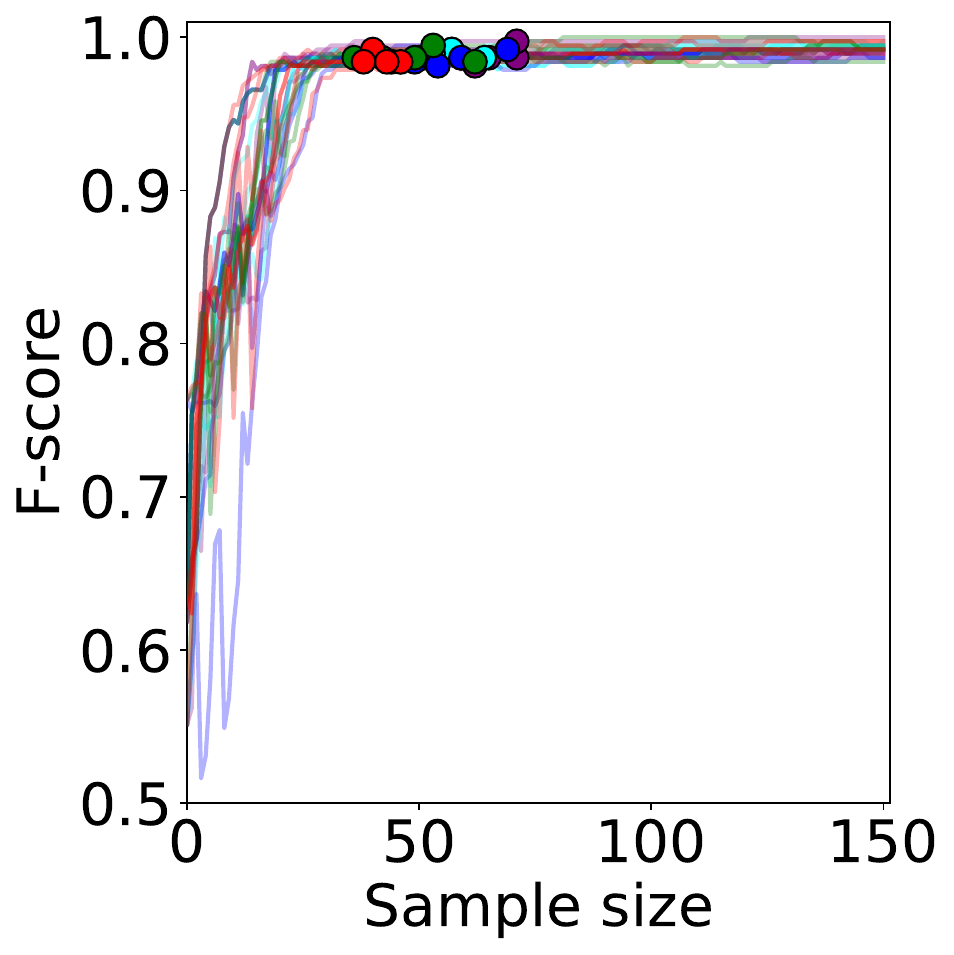}&    
    \includegraphics[height=4cm,width=4cm]{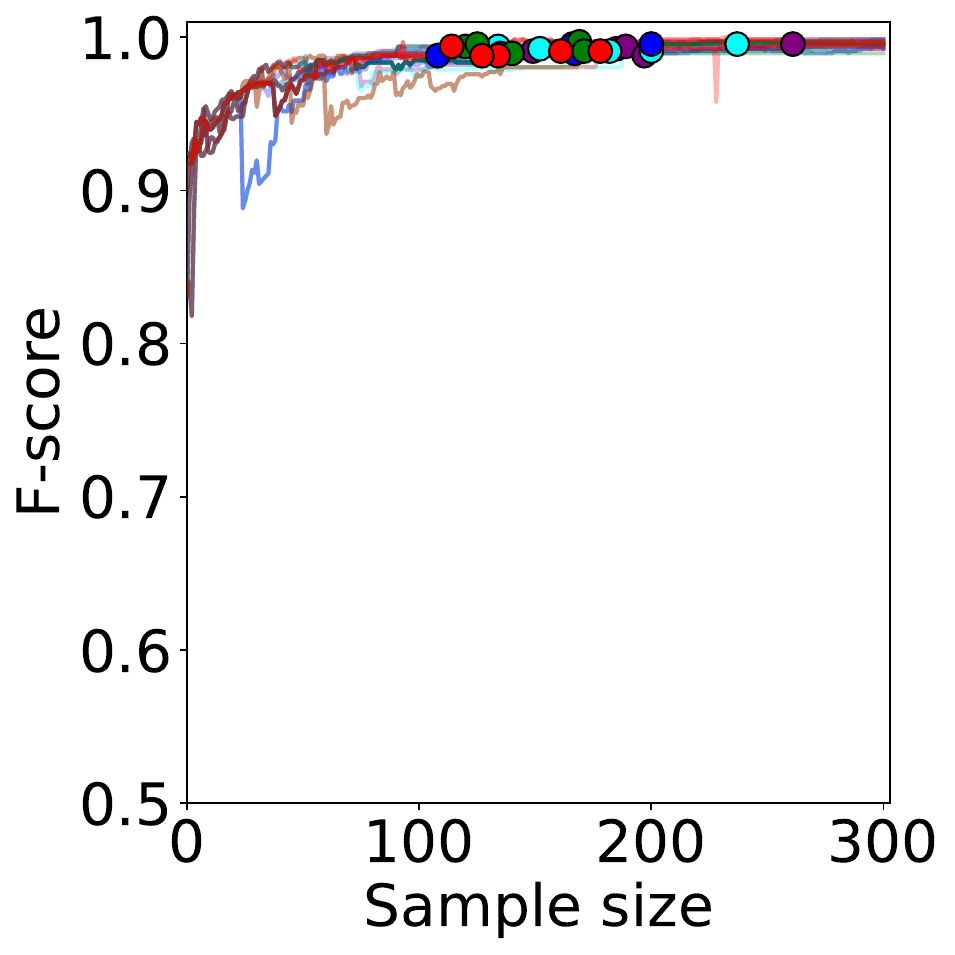}&
    \includegraphics[height=4cm,width=4cm]{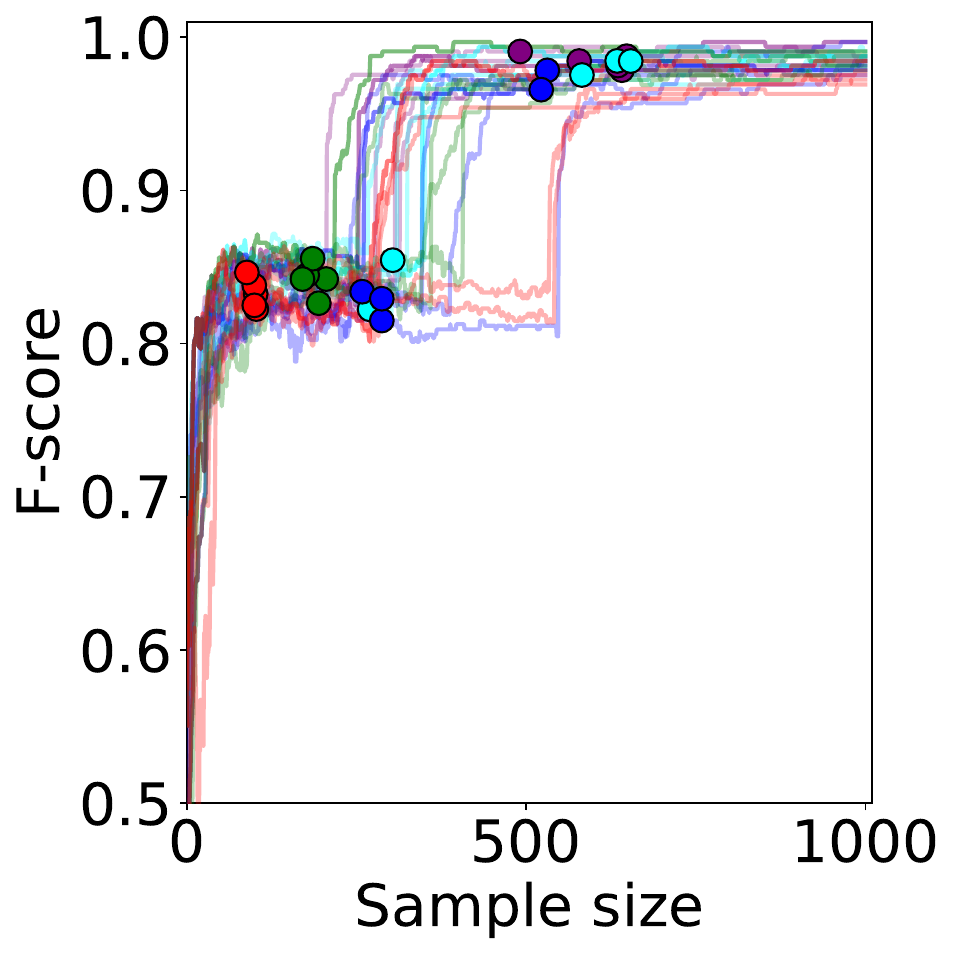}\\
    (g) {\tt{Sphere}}($\sigma_{\rm noise}=0.5$)& (h) {\tt{Rosenbrock}}($\sigma_{\rm noise}=30$)& (i) {\tt{Cross in tray}}($\sigma_{\rm noise}=0.01$)  \\
    \includegraphics[height=4cm,width=4cm]{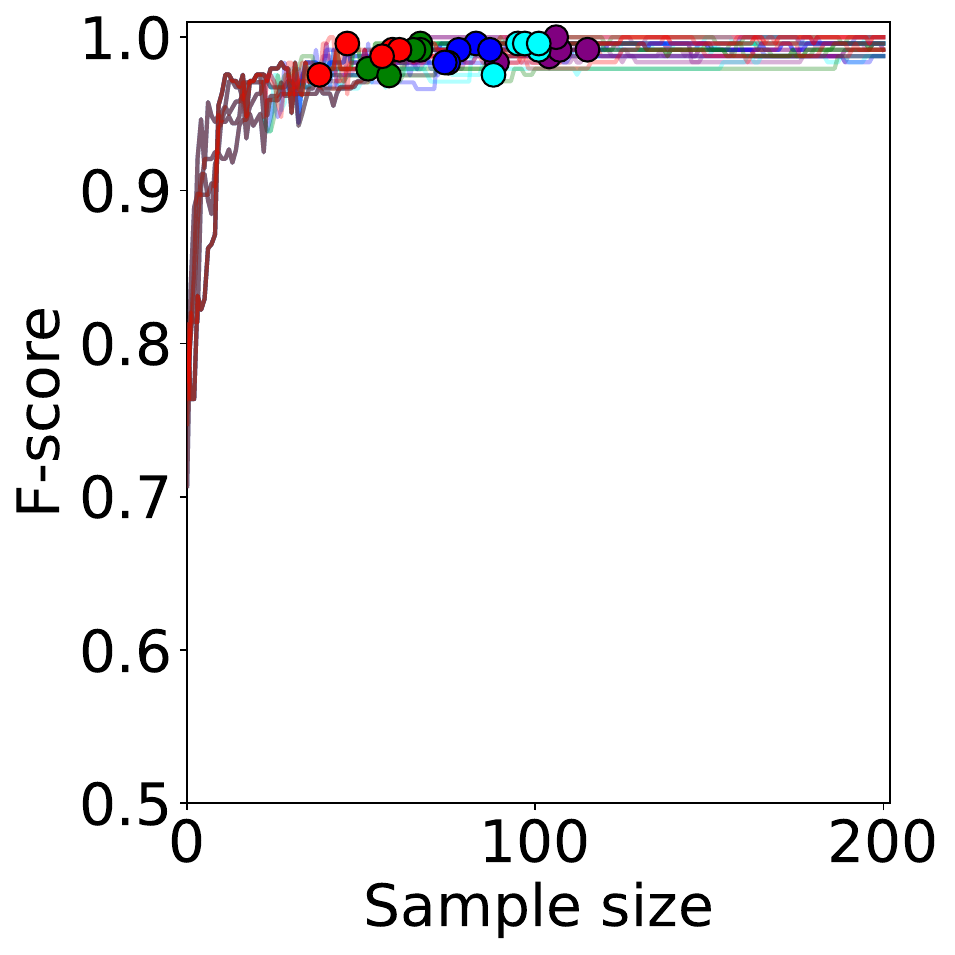}&    
    \includegraphics[height=4cm,width=4cm]{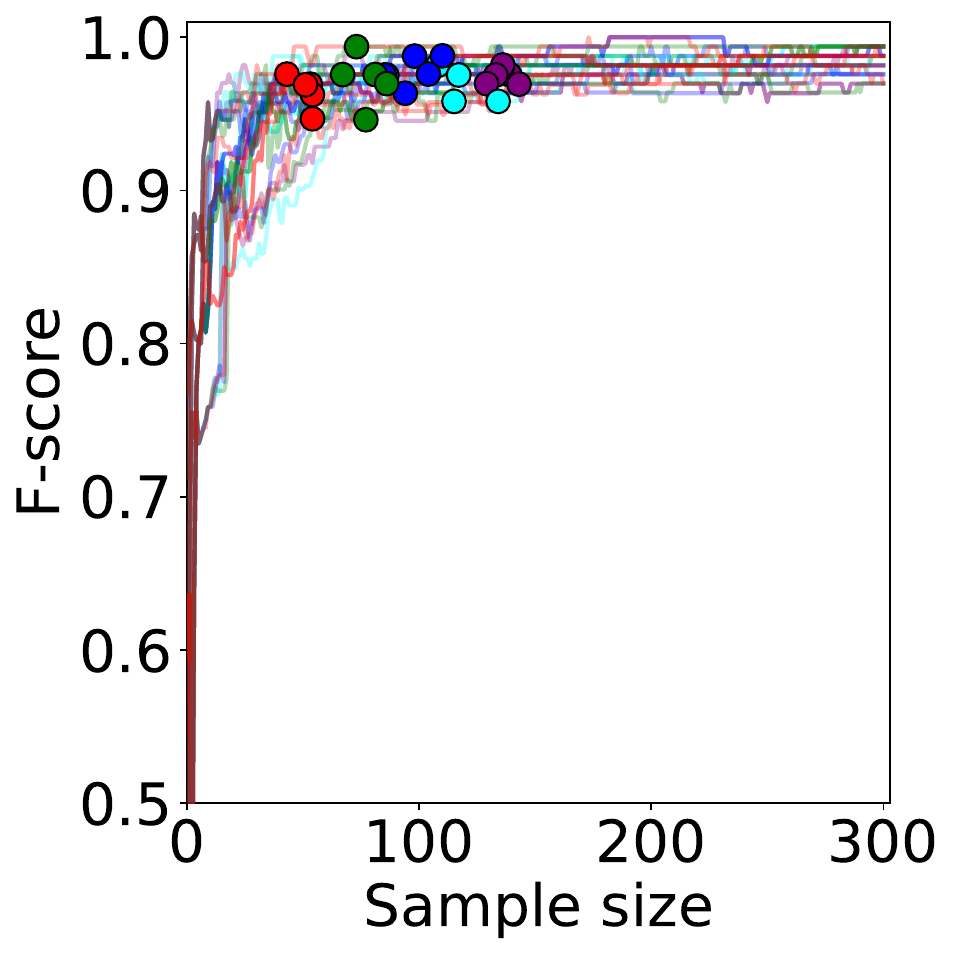}&
    \includegraphics[height=4cm,width=4cm]{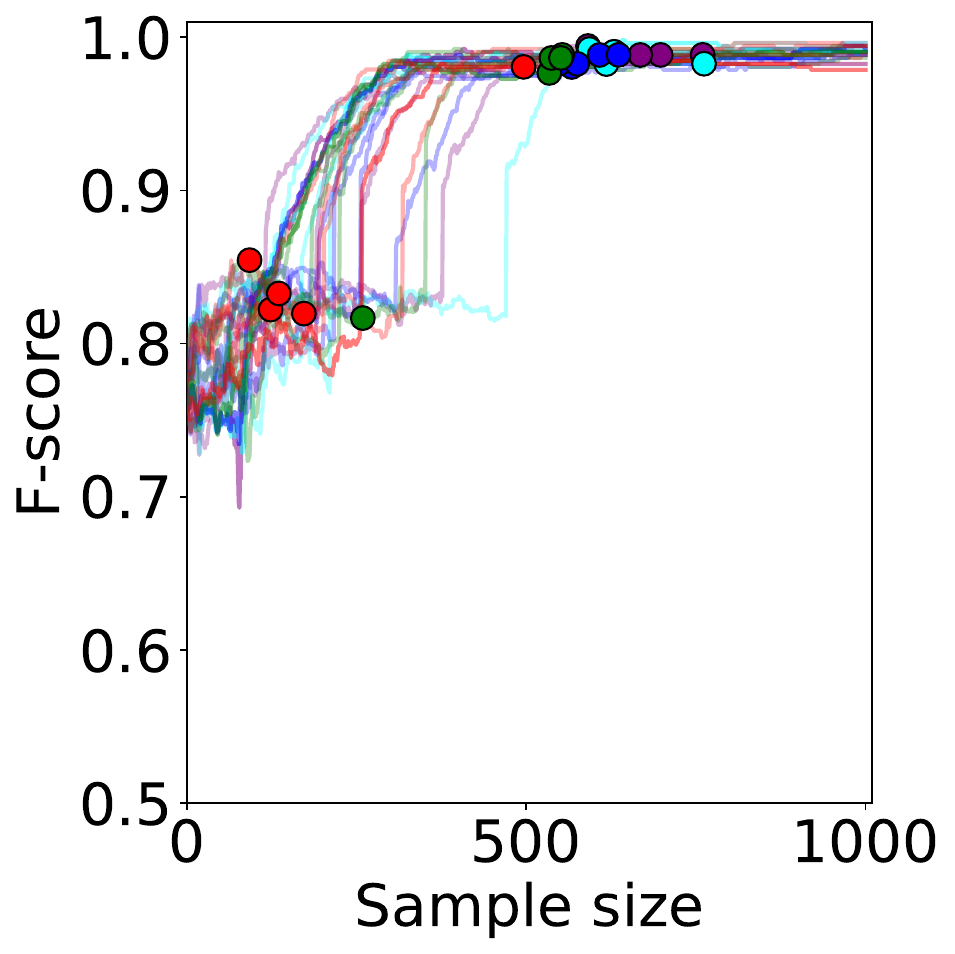}\\
    (j) {\tt{Booth}}($\sigma_{\rm noise}=30$)& (k) {\tt{Branin}}($\sigma_{\rm noise}=10$)& (l) {\tt{Holder table}}($\sigma_{\rm noise}=0.3$)
    \end{tabular}
    \caption{
    F-scores and Stopped time using each acquisition function. (a) -- (f) noise-free case. (g) -- (l) noise addition case. 
    }
    \label{fig_result_test_func_change_margin}
\end{figure}

\subsection{Effects of parameter of the proposed method}
\label{app:L}

Next, we evaluate the impact of changing the parameter $L$ of the proposed acquisition function and stopping criterion using test functions. In this experiment, experimental settings are the same as the Appendix~\ref{app:testfunc}. Figure~\ref{fig_result_test_func_change_margin} shows the F-scores and the stopping timing when $L$ is varied as $L=1, 2, 3, 4, 5$. Although there are rare cases where the increase of F-score is lower with some parameters, there are no significant differences in the F-scores when $L$ is varied, regardless of the test function used or whether observation noise is added. Thus, it can be said that the parameter $L$ does not significantly affect the acquisition function.

On the other hand, $L$ has a large impact on the stopping timing. Specifically, for test functions other than the {\tt{Cross in tray}} function and {\tt{Holder table}} function, the impact of $L$ is small without noise, but with noise, smaller $L$ values lead to earlier stopping. Therefore, in most cases, setting $L=1$ is sufficient if early stopping is desired even at the expense of accuracy, while typically setting $L$ around 3 is adequate. However, for the {\tt{Cross in tray}} and {\tt{Holder table}} functions, LSE stops before the F-score converges adequately when $L$ is small. This occurs because these functions have complex shapes, requiring appropriate hyperparameter estimation by the GP, which in turn requires a sufficient amount of data. With limited data, there is a tendency to estimate a larger kernel width. As a result, in situations with insufficient data, the GP may converge to a function different from the true function, meeting the stopping criterion and stopping LSE prematurely. In such cases, it is necessary to increase $L$ to prevent LSE from stopping until appropriate hyperparameters are estimated.

\subsection{Experiments with different threshold values}
\label{app:theta}

In this subsection, we evaluate the effect of changing the threshold on the stopping timing for four test functions: {\tt{Sphere}}, {\tt{Branin}}, {\tt{Rosenbrock}}, and {\tt{Booth}}. We examine three different threshold settings for each test function. For the {\tt{Sphere}} function, the thresholds are set to \(\theta = 10, 20, 30\); for the {\tt{Branin}} function, \(\theta = 50, 75, 100\); for the {\tt{Rosenbrock}} function, \(\theta = 100, 300, 500\); and for the {\tt{Booth}} function, \(\theta = 100, 300, 500\). All other experimental settings are the same as those described in Appendix~\ref{app:testfunc}, except for the threshold.

The results are shown in Fig.\ref{fig_change_threshold}. Consistent with the main text, it is observed that the FC criterion fails to stop, the FS criterion tends to stop LSE aggressively, and the proposed criterion stops LSE conservatively. However, as seen in Figs.\ref{fig_change_threshold}(c), (e), (d), and (f), there are cases where the FS criterion fails to stop LSE, and as shown in Figs.~\ref{fig_change_threshold}(h) and (i), there are cases where it stops before the F-score exceeds the threshold. In contrast, the proposed criterion consistently stops LSE even in such cases.

\begin{figure}[th!]
    \centering
    \begin{tabular}{ccc}
    \includegraphics[height=4cm,width=4.2cm]{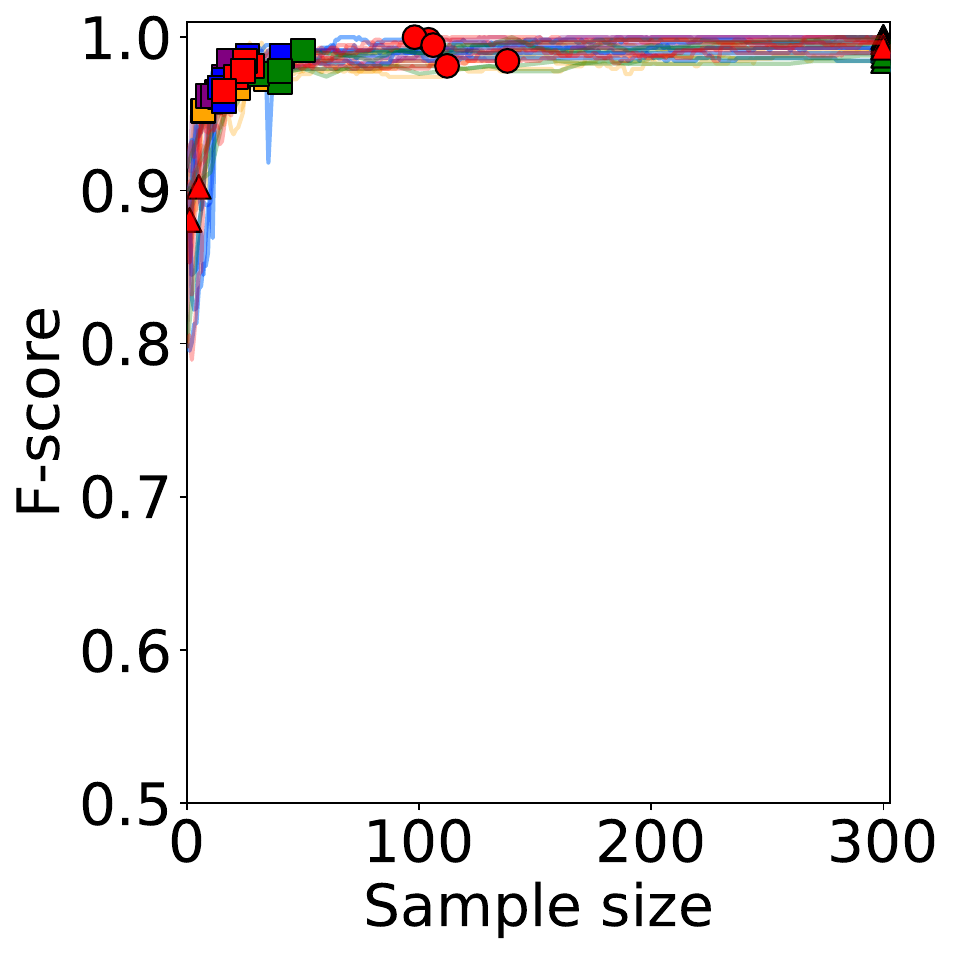}&
    \includegraphics[height=4cm,width=4.2cm]{images/sphere_20_2_f1_and_st.pdf}&
    \includegraphics[height=4cm,width=4.2cm]{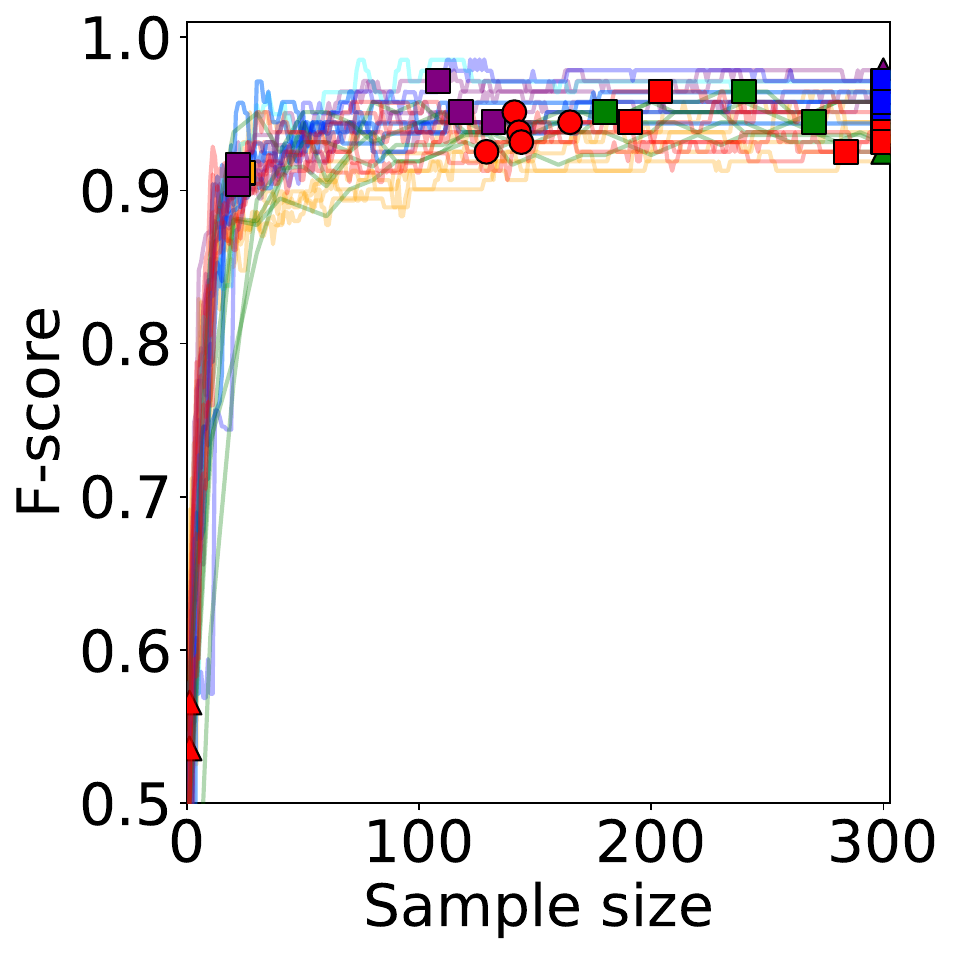}\\
    (a) {\tt{Sphere}}($\theta=10$)&    
    (b) {\tt{Sphere}}($\theta=20$)&    
    (c) {\tt{Sphere}}($\theta=30$) \\
    \includegraphics[height=4cm,width=4.2cm]{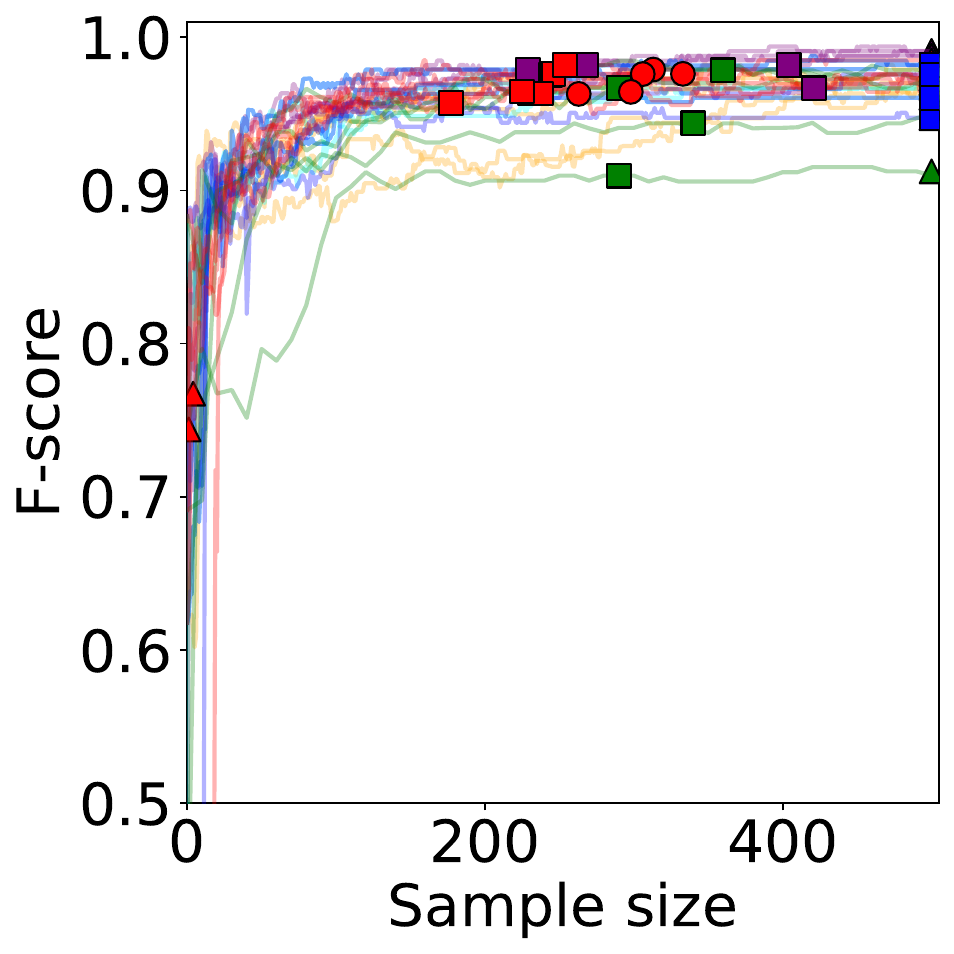}&
    \includegraphics[height=4cm,width=4.2cm]{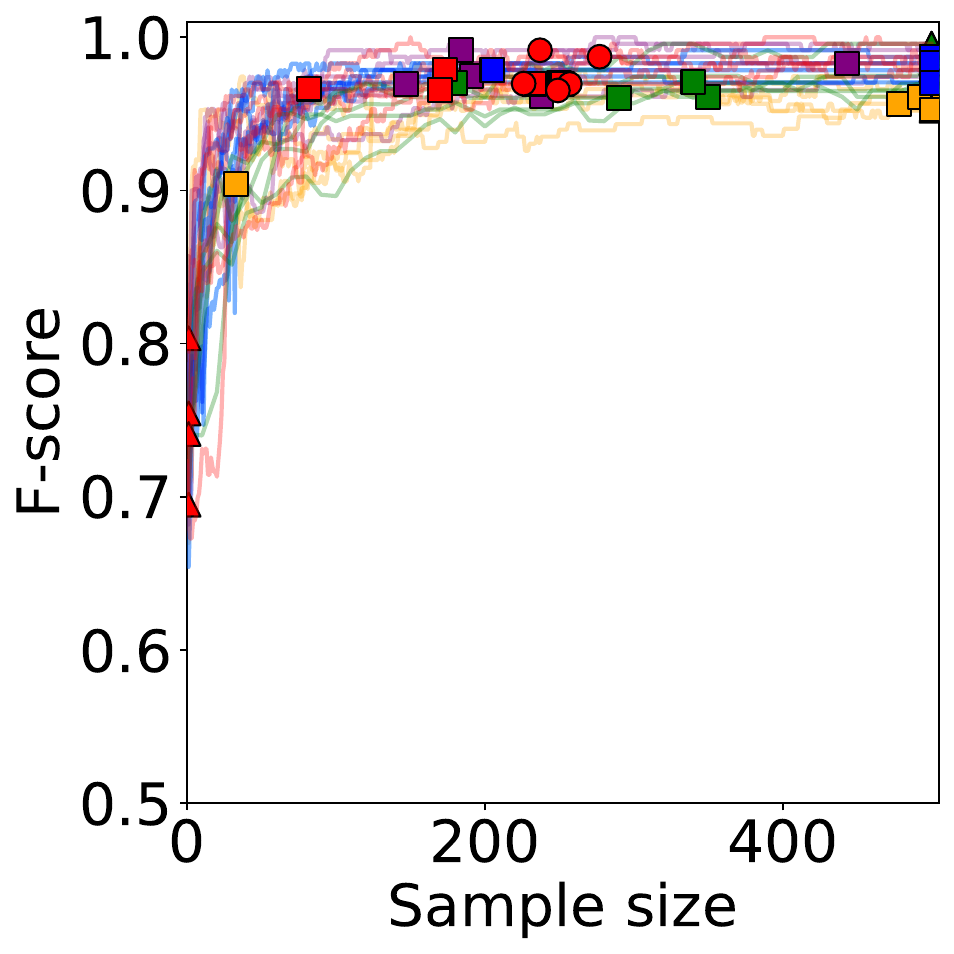}&
    \includegraphics[height=4cm,width=4.2cm]{images/branin_100_20_f1_and_st.pdf}\\
    (d) {\tt{Branin}}($\theta=50$)&    
    (e) {\tt{Branin}}($\theta=75$)&    
    (f) {\tt{Branin}}($\theta=100$) \\
    \includegraphics[height=4cm,width=4.2cm]{images/rosenbrock_100_30_f1_and_st.pdf}&
    \includegraphics[height=4cm,width=4.2cm]{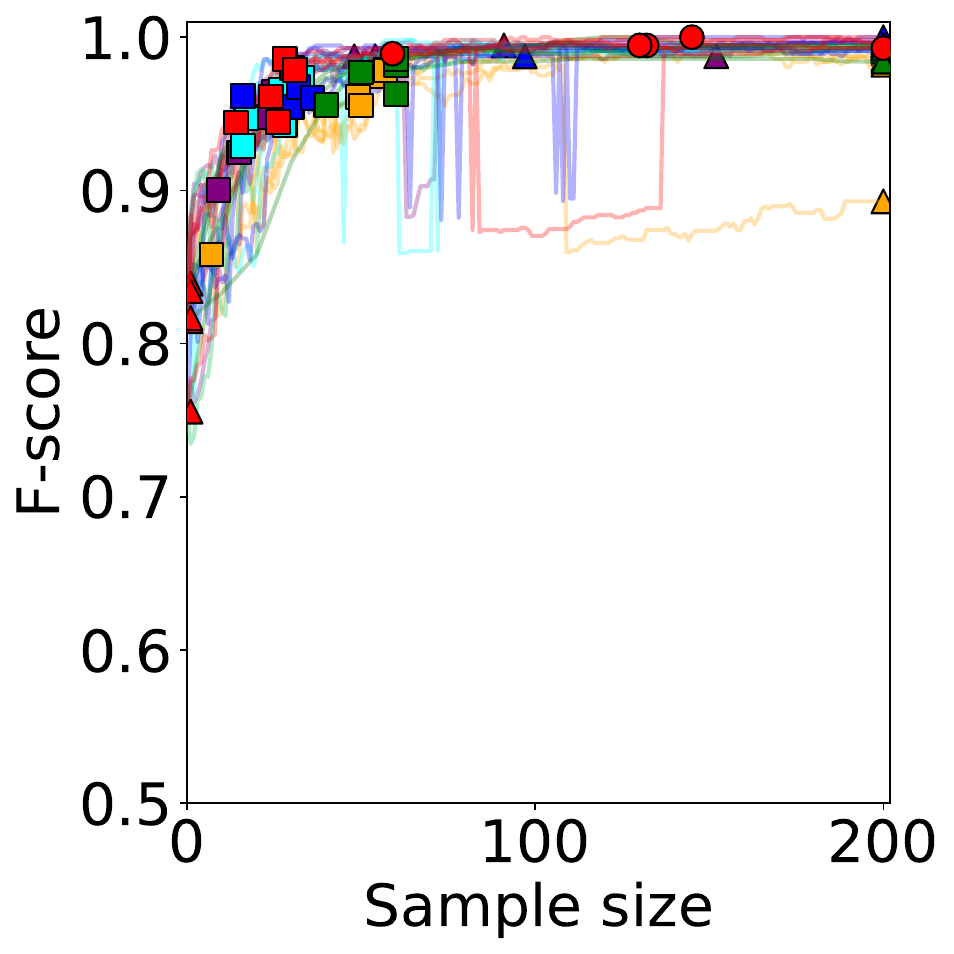}&
    \includegraphics[height=4cm,width=4.2cm]{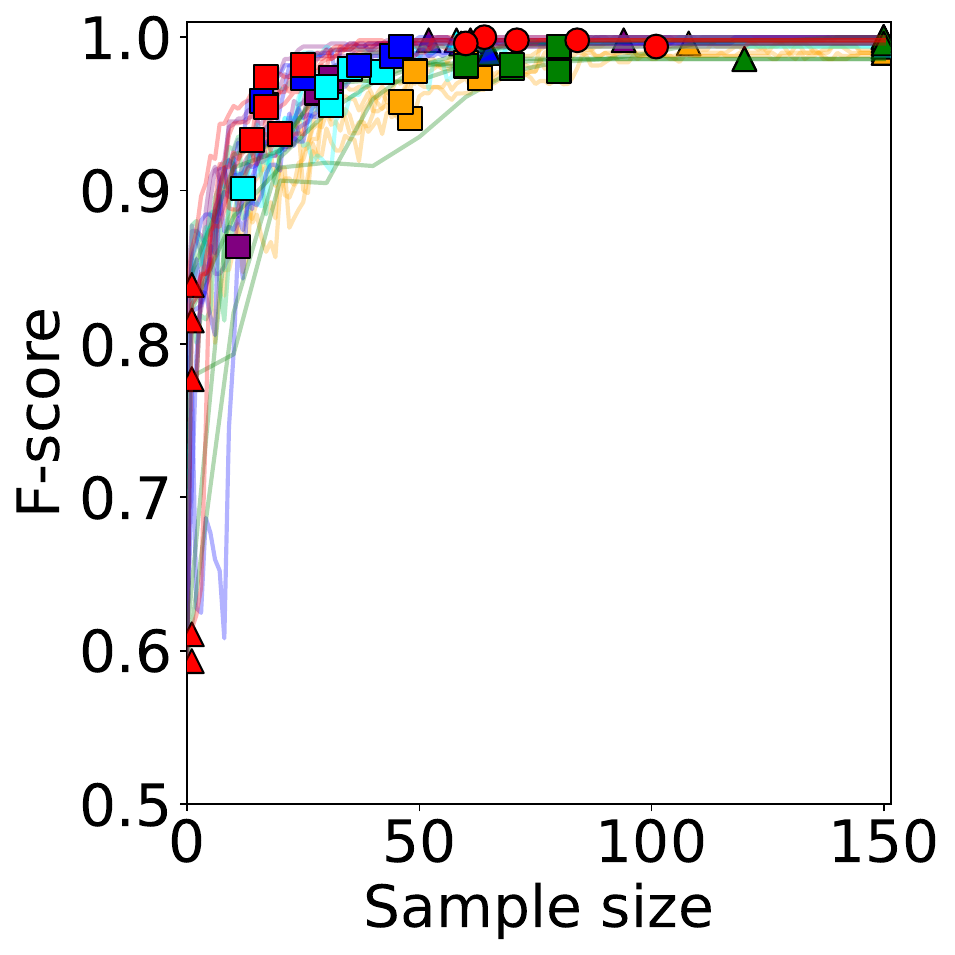}\\
    (g) {\tt{Rosenbrock}}($\theta=100$)&    
    (h) {\tt{Rosenbrock}}($\theta=300$)&    
    (i) {\tt{Rosenbrock}}($\theta=500$) \\
    \includegraphics[height=4cm,width=4.2cm]{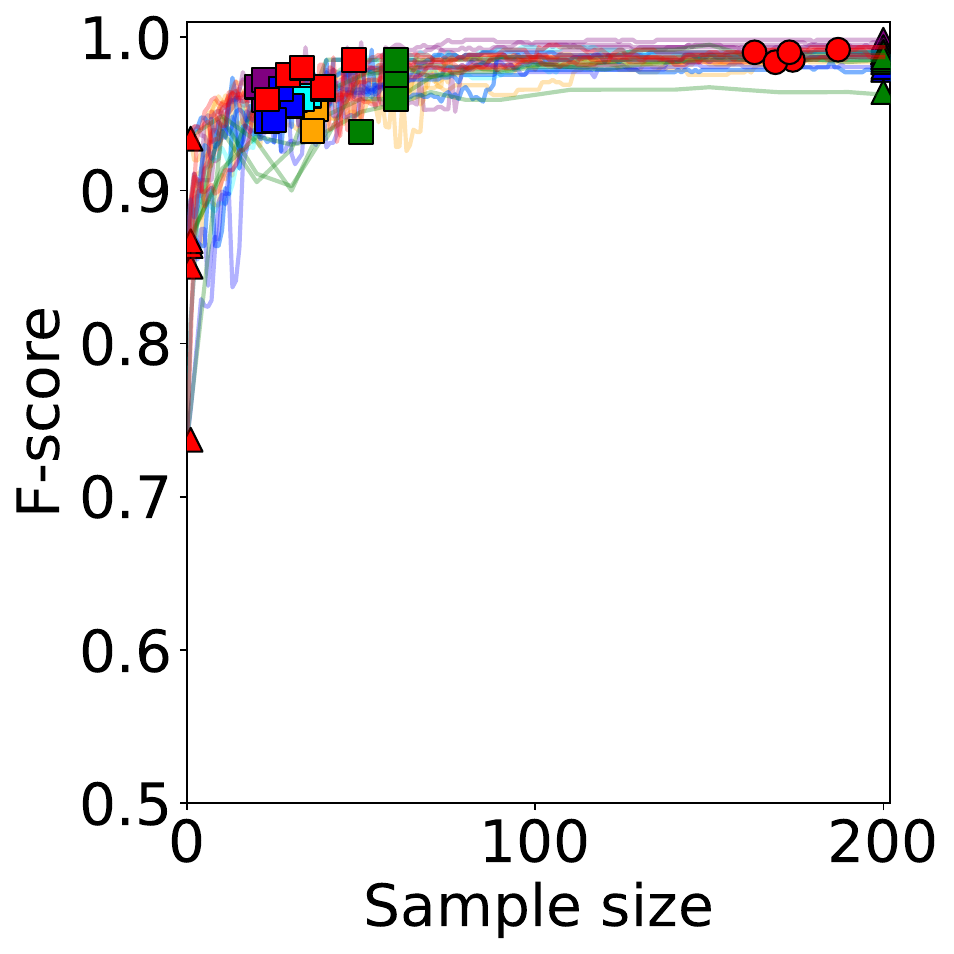}&
    \includegraphics[height=4cm,width=4.2cm]{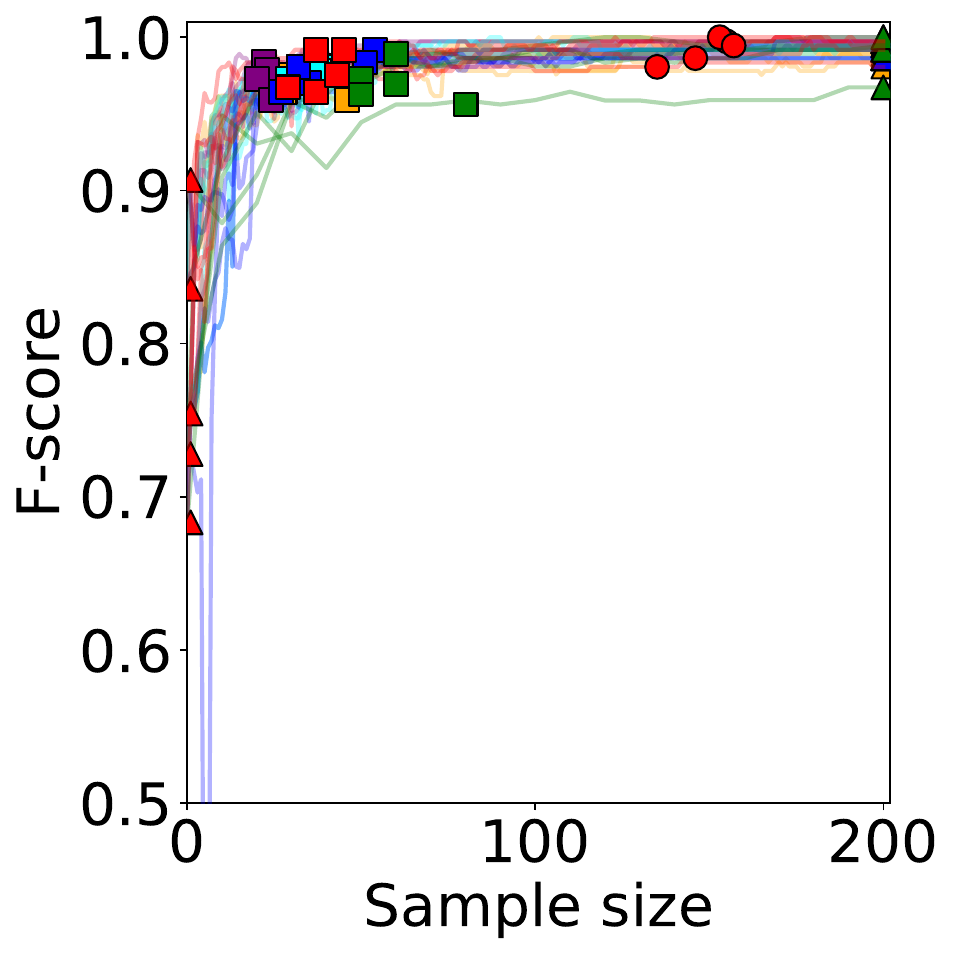}&
    \includegraphics[height=4cm,width=4.2cm]{images/booth_500_30_f1_and_st.pdf}\\
    (j) {\tt{Booth}}($\theta=100$)&    
    (k) {\tt{Booth}}($\theta=300$)&    
    (l) {\tt{Booth}}($\theta=500$) \\
    \end{tabular}
    \caption{
    F-scores using each acquisition function and stopped timings with the proposed (Ours) and fully classified (FC) criteria for datasets with different thresholds.
    }
    \label{fig_change_threshold}
\end{figure}
\end{document}